\documentclass[lettersize,journal]{IEEEtran}
\usepackage{amsmath,amssymb,amsthm}
\usepackage{array}
\usepackage{textcomp}
\usepackage{stfloats}
\usepackage{url}
\usepackage{verbatim}
\usepackage{graphicx}
\usepackage[numbers]{natbib}
\usepackage{multirow}
\usepackage{multicol}
\newtheorem{property}{Property}
\newtheorem{theorem}{Theorem}
\usepackage{enumitem}
\usepackage{mathtools}

\usepackage[utf8]{inputenc} 
\usepackage[T1]{fontenc}    
\usepackage{hyperref}       
\usepackage{url}            
\usepackage{booktabs}       
\usepackage{amsfonts}       
\usepackage{nicefrac}       
\usepackage{microtype}      
\usepackage{xcolor}         

\usepackage{amssymb}

\newtheorem{lemma}{Lemma}
\newtheorem{proposition}{Proposition}
\newtheorem{definition}{Definition}
\newtheorem{remark}{Remark}

\usepackage{tabularx}
\usepackage{booktabs}
\usepackage[table]{xcolor}

\usepackage{algorithm}
\usepackage{algpseudocode}

\usepackage[caption=false,font=footnotesize]{subfig}

\definecolor{PSAW}{RGB}{138,122,197} 

\definecolor{ETF}{RGB}{143,74,109}

\newcommand{\bx}{\mathbf{x}}
\newcommand{\by}{\mathbf{y}}
\newcommand{\bq}{\mathbf{q}}
\newcommand{\bk}{\mathbf{k}}
\newcommand{\bv}{\mathbf{v}}
\newcommand{\bW}{\mathbf{W}}
\newcommand{\bK}{\mathbf{K}}
\newcommand{\bV}{\mathbf{V}}
\newcommand{\bX}{\mathbf{X}}

\begin{document}

\title{Near-Oracle KV Selection via Pre-hoc Sparsity for Long-Context Inference} 

\author{Yifei Gao$\dagger$, Lei Wang$\dagger$,~\IEEEmembership{Senior Member,~IEEE,} Rong-Cheng Tu$\dagger$, Qixin Zhang, \\ Jun Cheng$^*$,~\IEEEmembership{Senior Member,~IEEE,} and Dacheng Tao,~\IEEEmembership{Fellow,~IEEE}
\thanks{This work was supported in part by the Guangdong Major Project of Basic and Applied Basic Research (2023B0303000016), Guangdong S\&T Program (No. 2024B0101050002), Guangdong-Hong Kong-Macao Joint Laboratory of Human-Machine Intelligence-Synergy Systems (2019B121205007), CAS Key Technology Talent Program, Colleges and Universities Key Laboratory of Intelligent Integrated Automation (No.201502), CAS Key Laboratory of Human-Machine Intelligence-Synergy Systems, Shenzhen Institutes of Advanced Technology, Chinese Academy of Sciences (2014DP173025). }

\thanks{Manuscript received July 1, 2025; revised XX, 2025. (Corresponding author: jun.cheng@siat.ac.cn)}
\thanks{Yifei Gao is with University College London, London, UK. (E-mail: ucabaoj@ucl.ac.uk).}
\thanks{Lei Wang and Jun Cheng are with the Shenzhen Institutes of Advanced Technology, Chinese Academy of Sciences (CAS), Shenzhen 518055, China. They are also with the Chinese University of Hong Kong.  (e-mail: lei.wang1@siat.ac.cn, jun.cheng@siat.ac.cn).}
\thanks{Rong-Cheng Tu, Qixin Zhang and Dacheng Tao are with the College of Computing \& Data Science, Nanyang Technological University, Singapore. (E-mail: rongcheng.tu@ntu.edu.sg, qixin.zhang@ntu.edu.sg, dacheng.tao@ntu.edu.sg)}
}

\markboth{IEEE Transactions on Pattern Analysis and Machine Intelligence,~Vol.~X, No.~X, August~2025}%
{Shell \MakeLowercase{\textit{et al.}}: A Sample Article Using IEEEtran.cls for IEEE Journals}


\maketitle

\begin{abstract}
A core bottleneck in large language model (LLM) inference is the cost of attending over the ever-growing key–value (KV) cache. Although near-oracle top-$k$ KV selection can preserve the quality of dense attention while sharply reducing computation and bandwidth, existing sparse methods generally rely on posterior heuristics—selectors conditioned on observed attention or proxy scores. Such conditioning introduces posterior bias: it tends to distort the true token importance and miss salient tokens, thereby impairing long-range reasoning.
To tackle this problem, we propose Pre-hoc Sparsity (PrHS), which selects KV entries before attention scoring and provides explicit accuracy control. Let the average attention mass of discarded entries be $\bar{\delta}$. Through an attention-routing information analysis, we derive
an upper bound on the absolute routing-MI perturbation that depends only on $\bar{\delta}$. This provides an a-priori
information certificate: by controlling dropped attention mass, PrHS limits the perturbation induced by sparse KV selection. Within PrHS, we instantiate three orthogonal pre-hoc selectors along the axes of time, depth, and layer: (i) temporal query sharing, justified by a Lipschitz bound on attention-centroid drift so that a modest dilation of the previously shared set covers the next step; (ii) recency-window retention, leveraging distance-decay priors to guarantee retained attention mass; and (iii) cross-layer KV sharing, exploiting high inter-layer similarity to eliminate redundant retrieval and computation. Extensive experiments on LLaMA, Qwen, and GLM families across GSM8K, CoQA, and LongBench v2 validate PrHS. It avoids ${>}90\%$ of head-layer retrievals while preserving near-dense accuracy, reduces retrieval complexity by up to $3\times$ over HShare, and remains effective on contexts exceeding $128\text{k}$ tokens. CPE further lowers attention FLOPs by ${\sim}15\%$ versus prior sparse baselines and achieves up to a $9.9\times$ attention-operator speedup and $3.3\times$ higher end-to-end throughput over dense baselines on an NVIDIA A100 GPU.
\end{abstract}


\vspace{-0.05in}
\section{Introduction}
\label{sec:intro}
\vspace{-0.05in}

\IEEEPARstart{L}{arge} Language Models (LLMs)~\cite{hurst2024gpt,liu2024deepseek,TPAMI2025LLM_Memristor} empower advanced applications such as multi-turn dialogue~\cite{yi2402survey,teng2024fine,ma2024multi}, long-form document summarization~\cite{zhang2023enhancing,liu2023learning}, and multimodal tasks~\cite{tu2023unsupervised,liu2023visual,lin2023video,tu2025global,TPAMI2025LLM_MultiModal,tu2025automatic} involving text and visual understanding \cite{TPAMI2025LLM_VQA,TPAMI2025LLM_Creativity,tu2025multimodal,TPAMI2025LLM_Disease,lan2025survey,tu2025mllm}. 
However, in autoregressive decoding, the per-layer key–value (KV) cache grows linearly with the generated length $L$, requiring each new query to access an increasingly large set of past KVs.
With $H$ heads and per-head dimension $d$, the per-query attention cost (score computation and value aggregation) scales as $\mathcal{O}(HLd)$~\cite{vaswani2017attention}. As $L$ increases, runtime becomes dominated by attention computation rather than model forward passes; 
for example, at a 40k-token context, LLaMA-3.1-8B~\cite{touvron2023llama} incurs attention compute on the order of tens of TFLOPs per decoding step.

KV compression alleviates these costs by exploiting attention sparsity~\cite{wang2020linformer,sun2025vorta,zhang2023h2o,xiao2023streamingllm}: most KV entries receive negligible attention, so retaining a small, critical subset can approximate full attention while reducing memory traffic and latency. With a per-head budget $k\!\ll\!L$, the top-$k$ oracle retains the $k$ highest-score entries per query, minimizing approximation error and achieving the best accuracy. While this reduces value aggregation from $\mathcal{O}(HLd)$ to $\mathcal{O}(Hkd)$, identifying the top-$k$ still requires computing all scores at $\mathcal{O}(HLd)$ plus a retrieval overhead. Thus, end-to-end cost remains dominated by full scoring, making the oracle impractical for deployment.

To simplify full scoring, existing approaches approximate the top-$k$ oracle by summarizing observed (posterior) attention statistics into priors that constrain future retrieval, yielding posterior-conditioned policies. For example, \emph{Token Dropping Oracles} (TDOs) apply static or dynamic eviction policies~\cite{zhang2023h2o,xiao2023streamingllm,he2024zipcache,wan2024d2o} to keep only entries deemed critical at the current step, restricting the candidate set and reducing per-step scoring to $\mathcal{O}(Hkd)$. \emph{Query-Aware Approximations} (QAAs)~\cite{tang2406quest,sun2024shadowkv,wu2024tokenselect,liu2024hashevict} construct low-dimensional, hand-designed surrogate features to approximate attention scores and use them to guide retrieval, incurring $\mathcal{O}(HLd')$ with $d'\!\ll\!d$.

Despite the efficiency gains, these posterior-driven simplifications incur systematic bias that degrade generative performance.
Specifically, in TDO, token salience evolves over time. Thus, selectors that
condition on observed attention along the trajectory face
non-stationary evidence and systematically skew future decisions
\cite{wan2024d2o,sun2024shadowkv}. In addition, QAAs that
replace full logits with heuristic proxies induce a score-level bias.
We group these archetypes under \textbf{Post-hoc Sparsity (PoHS)}—
\emph{posterior-conditioned} rules that bias token-importance estimates and thereby impair long-range reasoning. Empirically, PoHS yields substantially larger perturbations at both the attention and the output (latent) levels
than the top-$k$ oracle (Fig.~\ref{fig:teaser_attn_diff}, \ref{fig:teaser_latent_diff}; lower is better).

\begin{figure*}[!t]
  \centering
  \captionsetup[subfloat]{font=scriptsize}
  \subfloat[Introduced attention disturbances (absolute difference) after filtering pre-softmax attention logits $\!<\!0$.]{
  \footnotesize
  \label{fig:teaser_attn_diff}
    \includegraphics[width=0.305\textwidth]{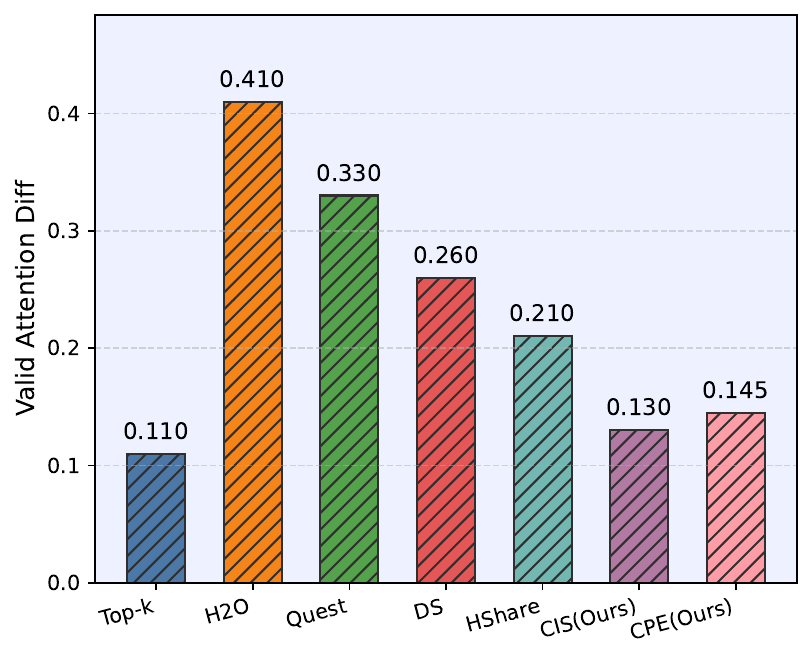}
  }\hfill
  \subfloat[Introduced absolute differences between original attention outputs on COQA.]{
  \footnotesize
  \label{fig:teaser_latent_diff}
    \includegraphics[width=0.305\textwidth]{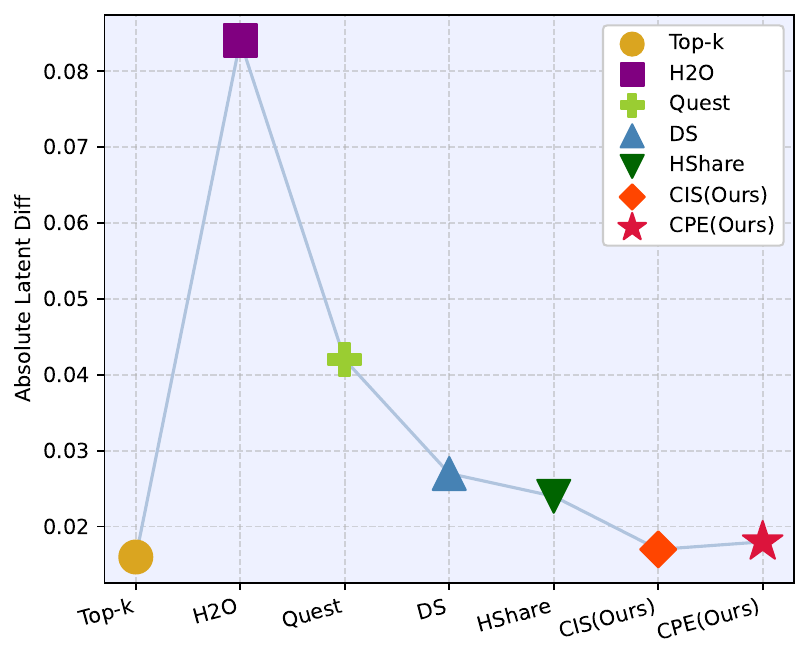}
  }\hfill
  \subfloat[Overall comparisons on accuracy-consumption trade-offs among state-of-the-art baselines.]{
  \footnotesize
  \label{fig:teaser_overall_comp}
    \includegraphics[width=0.305\textwidth]
    {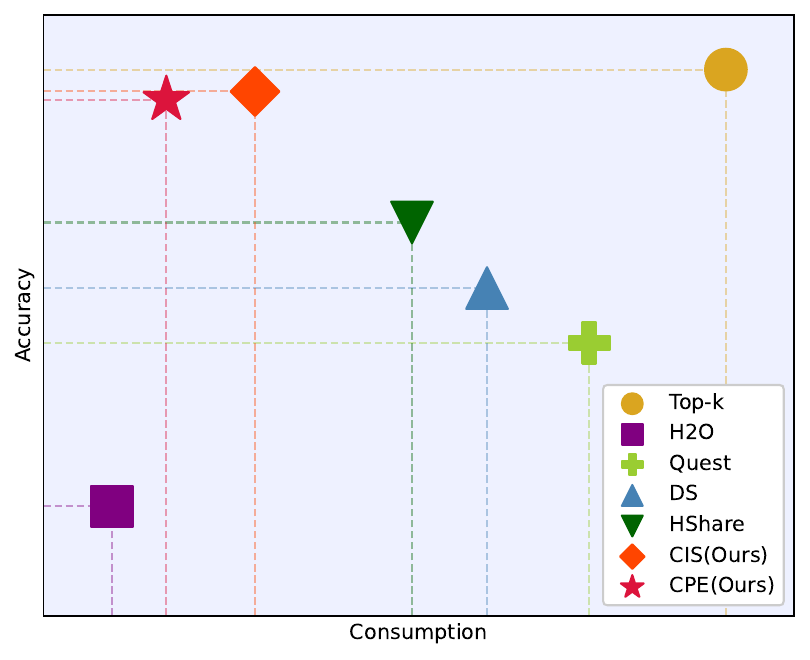}
  }
  \vspace{-0.05in}
  \caption{\textbf{Overall analysis of performance and efficiency}. Across induced attention-approximation error and accuracy-efficiency trade-offs, our method consistently surpasses prior SOTA approaches and closely tracks the top-$k$ oracle for optimal KV compression accuracy.}
  \label{fig:teaser}
  \vspace{-0.1in}
\end{figure*}

Therefore, in this paper, we analyze selection through a \emph{marginal-to–mutual-information} lens and devise a new framework, \textbf{Pre-hoc Sparsity} (PrHS).
Let $\bar{\delta}$ denote the average dropped attention mass.
We prove that KV selection induces a bounded perturbation on
the finite attention-routing channel:
the absolute routing-MI deviation is upper-bounded by a function
of $\bar{\delta}$ (Eq.~\eqref{eq:mi_bound}).
This provides an a-priori information certificate: selectors
that preserve more attention mass admit tighter guarantees on
routing-information perturbation. Consequently, PrHS markedly reduces attention- and output-level perturbations and approaches the oracle frontier in the accuracy–efficiency trade-off (Fig.~\ref{fig:teaser_overall_comp}).
We further quantify deviation from the oracle by
$\bar{\beta}_{\mathrm{th}}$, the average retained-mass gap between PrHS and the budget-matched top-$k$ oracle.
As $\bar{\beta}_{\mathrm{th}}{\rightarrow}0$, the PrHS certificateconverges to the oracle certificate; proof in
Sec.~\ref{sec:app_mib_tsa}).

We instantiate three PrHS techniques, Clustered Index Sharing (CIS), Progressive Sliding Attention Window (PSAW), and Early Token Freezing (ETF), each translating a distinct sparsity prior into a PrHS selector. These selectors are parameterized to control their retained-mass error and thereby tighten the routing-information certificate (Eq.~\eqref{eq:prhs_bound}). CIS explores temporal-locality~\cite{wuhshare}: consecutive queries are highly similar. It shares the retrieved set from the top-$k$ oracle across temporally adjacent queries whose embeddings exceed a cosine-similarity threshold, and dilates it by adding neighbors around highest-score indices. PSAW leverages the recency prior~\cite{xiao2023streamingllm,zhang2023h2o}—attention mass concentrates on a local window—and applies a depth-adaptive sliding window that shrinks monotonically with depth, incrementally pruning the candidate KV entries. ETF exploits cross-layer redundancy~\cite{liu2405minicache}: adjacent layers exhibit high KV similarity. ETF freezes stable early tokens to reduce updates. The three selectors act along different axes—time (CIS), depth (PSAW), and layer (ETF)—and compose naturally: CIS seeds the candidate pool with a small budget, then PSAW and ETF intersect their selections with the CIS seed using larger budgets to further prune. Finally, we implement custom CUDA kernels that parallelize index manipulation and sparse-attention computation to minimize runtime overhead. We refer to the combined system as \textbf{CPE}. 



Collectively, CPE substantially reduces KV-retrieval and attention-computation overhead, yielding a more favorable accuracy–efficiency trade-off than state-of-the-art (SOTA) baselines (Fig.~\ref{fig:teaser_overall_comp}). Empirically, on LLaMA~\cite{touvron2023llama}, Qwen~\cite{team2025qwen3}, and GLM~\cite{zeng2025glm} model families across GSM8K and CoQA, CIS matches or exceeds the accuracy of the recent SOTA method HShare~\cite{wuhshare} while requiring only about \emph{one-third} of its retrieval complexity. On LongBench v2, while prior sparse methods suffer substantial degradation on longer contexts, CPE exhibits markedly stronger robustness to context length, remaining effective even beyond 128k-token contexts. Overall, CPE maintains comparable model quality while eliminating $> !90\%$ of head-layer KV retrievals and reducing attention FLOPs by $\sim !15\%$ relative to prior sparse baselines. On NVIDIA A100-80GB GPUs, CPE attains a $\mathbf{9.9\!\times}$ speedup in attention-operator latency and a $\mathbf{3.3\!\times}$ throughput improvement over original models.  Our contributions are summarized as follows:
\begin{enumerate}[label=\arabic*)]
    \item Rigorous analyses of the mutual information bound from the marginal entropy in KV compression and introduction of Pre-hoc Sparsity towards the accuracy of top-$k$ oracle.
    \item Introduction of CIS, which sustains stable performance at high sharing ratios, together with PSAW and ETF for lossless sparsification of attention computation, all implemented with customized CUDA kernels for parallelization.
    \item Extensive experiments and ablations demonstrating consistent improvements in accuracy, efficiency, and throughput over state-of-the-art baselines.
\end{enumerate}


\section{Problem Formulation}

This section introduces the variables and notation used throughout our
analysis. For convenience, Table~\ref{tab:notation} summarizes the key
symbols and their meanings.

\subsection{LLM Inference}
\label{sec:inference_background}

To demonstrate the inference bottlenecks addressed by KV compression,
we briefly review LLM inference, consisting of prefill and decoding,
using the standard transformer architecture with multi-head
attention~\cite{vaswani2017attention}.

\paragraph{Multi-head Attention}
Let the input token embeddings for head $h$ be
$\bx^h=[\bx_1^h;\dots;\bx_L^h]\in\mathbb{R}^{L\times d}$.
With $H$ heads and per-head dimension $d$, for each token $i$ and
head $h\in\{1,\dots,H\}$,
\begin{equation}
\label{eq:qkv-proj}
\bq_i^h
=
\bx_i^h\bW_Q^h,\qquad
\bk_i^h
=
\bx_i^h\bW_K^h,\qquad
\bv_i^h
=
\bx_i^h\bW_V^h .
\end{equation}
Here
$\bq_i^h,\bk_i^h,\bv_i^h\in\mathbb{R}^{d}$, and $\bK^h
=
[(\bk_1^h)^\top;\dots;(\bk_L^h)^\top]$, $\bV^h=[(\bv_1^h)^\top;\dots;(\bv_L^h)^\top] \in\mathbb{R}^{L\times d}$ form the KV cache. The attention weights for token $t$ in head $h$
over indices $i\in\{1,\dots,t\}$ and the corresponding attention
output are
\begin{equation}
\label{eq:attn_eq}
\begin{aligned}
A_i^h(\bq_t^h){=}
\frac{
\exp\!\big((\bq_t^h)^\top\bk_i^h/\sqrt d\big)
}{
\sum_{j=1}^{t}
\exp\!\big((\bq_t^h)^\top\bk_j^h/\sqrt d\big)
},
\mathbf{y}_t^h
{=}
\sum_{i=1}^{t}
A_i^h(\bq_t^h)\,\bv_i^h .
\end{aligned}
\end{equation}
Since all heads follow the same computation, we omit the head index
for brevity unless explicitly stated.

\paragraph{Prefill}
The prompt is processed once to construct and cache $(\bK,\bV)$,
which is reused during subsequent decoding~\cite{ott2019fairseq}.

\paragraph{Decode}
Only the newly generated token is processed while reusing the cached
$\bK$ and $\bV$ from previous steps for each head.

\paragraph{Bottlenecks}
From Eq.~\eqref{eq:attn_eq}, the per-step attention FLOPs
(score computation and value aggregation) scale as
$\mathcal{O}(Htd)$, while the KV cache scales as
$\mathcal{O}(Htd)$. This linear growth in $t$ constitutes the primary
compute and memory bottleneck for long-context inference.

\begin{table}[t]
\caption{Summary of key variables and notation.}
\label{tab:notation}
\centering
\renewcommand{\arraystretch}{1.15}
\begin{tabularx}{\columnwidth}{lX|lX}
\toprule
$\bx$ & Input token embedding 
& $\bq,\,\bk,\,\bv$ & Query, key, value \\
$L$ & Input length
& $d$ & Per-head dimension \\
$H$ & Number of heads
& $A(\cdot)$ & Attention weight \\
$\bK, \bV$ & KV cache matrices
& $t$ & Step count\\
$\mathcal S$ & Selector returning a subset of indices
& $S_t$ & Selected KV index set at step $t$ \\
$N$ & KV budget (sparsity: $N/t$)
& $\tau(\cdot)$ & Retained attention mass \\
$\delta(\cdot)$ & Dropped attention mass
& $I_{\mathrm{full}}, I_{\mathcal S}$ & Routing MI under full and sparse attention \\
$g(\cdot)$ & Routing-MI certificate function
& $\delta^*(\cdot)$ & Oracle (top-$k$) dropped mass \\
\bottomrule
\end{tabularx}
\end{table}

\subsection{Token-sparse Attention and Key-value Sharing}
\label{sec:prob_def}

KV compression mitigates the above bottlenecks by exploiting attention
sparsity: a small subset of critical KV entries is selected to
approximate full attention. This improves efficiency at the cost of
selection-induced approximation error. To characterize
accuracy--efficiency trade-offs across KV compression methods, and
following prior work~\cite{xiao2024duoattention,wu2024tokenselect,
wuhshare,liu2025freekv}, we formalize decoding with a limited KV set
under a unified \emph{token-sparse attention} (TSA) concept, and define
\emph{KV sharing} as a special case that streamlines TSA selection.

\textbf{Definition 3.1} (Token-sparse attention, informal).
\textit{
At decoding step $t$, with KV budget $N<t$, a selector
$\mathcal S$ outputs an index set
$S_t\subseteq[t]\coloneqq\{1,\dots,t\}$ with $|S_t|=N$.
The corresponding key/value matrices are restricted to these indices,
$\bK_{S_t}$ and $\bV_{S_t}$. Attention is computed only between
$\bq_t$ and $\bK_{S_t}$ (cf. Eq.~\eqref{eq:attn_eq}), and the output
uses $\bV_{S_t}$. The sparsity ratio at step $t$ is $N/t$.
}

\textbf{Definition 3.2} (Key-value sharing).
\textit{
Given indices $S_i$ retrieved by $\mathcal S$ at an anchor step $i$
in TSA, a later token at step $j$ may directly reuse them, denoted
$S_j\leftarrow S_i$. In this case, $S_j$ is used to construct
$\bK_{S_j}$ and $\bV_{S_j}$.
}

\textbf{Discussion}.
Different TSA instantiations correspond to different selectors
$\mathcal S$ and exhibit distinct accuracy--efficiency trade-offs
under a fixed budget. For example, imprecise selection that admits
many low-contribution pairs can crowd out salient entries, reducing
retained attention mass and degrading model performance. This failure
mode is analogous to erroneous KV sharing across steps whose truly
critical sets have little overlap. Moreover, selectors rely on
different criteria~\cite{liu2024hashevict,zhang2023h2o,
sun2024shadowkv}, such as similarity heuristics, learned scores, or
recency rules, while the mechanism used to enforce these criteria
determines selection complexity and retrieval efficiency.
Consequently, an effective TSA selector must balance retained
attention mass and selection overhead.

\vspace{-0.05in}
\subsection{Routing-Information Perturbation}
\label{sec:info_bound}

To quantify selector fidelity, we characterize the perturbation of
information carried by the finite attention-routing channel and pair
it with retrieval-efficiency analysis. Our central quantity is the
dropped attention mass. The top-$k$ oracle minimizes this mass under a
fixed budget and therefore provides the tightest dropped-mass
certificate, but requires explicit attention evaluation. Practical
selectors approximate this oracle while reducing selection overhead.
This motivates designing selectors that minimize their retained-mass
gap from the oracle subject to a fixed KV budget.

\paragraph{Routing-MI Certificate}
Fix a layer--head pair and an eligible context length $L$.
Let $\bX$ denote a random query--context state, with realization
$\bx=(\bq,\bk,\bv)$, and define the routing index
$T\in[L]$ by
\begin{equation}
    \mathbb P(T=i\mid\bX=\bx)=A_i(\bq).
\end{equation}
For a selector $\mathcal S$, define the retained and dropped attention
masses
\begin{equation}
\label{eq:mass_def}
\tau_{\mathcal S}(\bq)
=
\sum_{i\in\mathcal S}A_i(\bq),
\qquad
\delta_{\mathcal S}(\bq)
=
1-\tau_{\mathcal S}(\bq).
\end{equation}
The sparse routing index $T_{\mathcal S}$ follows the
truncated-renormalized distribution
\begin{equation}
\mathbb P(T_{\mathcal S}{=}i\mid\bX{=}\bx)
=
\widetilde A_i(\bq)
=
\begin{cases}
A_i(\bq)/\tau_{\mathcal S}(\bq),
    & i\in\mathcal S,\\[2pt]
0,  & \text{otherwise}.
\end{cases}
\end{equation}

Define
\begin{equation}
I_{\mathrm{full}}
:=
I(\bX;T),
\qquad
I_{\mathcal S}
:=
I(\bX;T_{\mathcal S}),
\end{equation}
and the routing-information perturbation
\begin{equation}
\Delta I_{\mathcal S}
:=
\left|
I_{\mathrm{full}}-I_{\mathcal S}
\right|.
\end{equation}
Since truncation and renormalization need not monotonically decrease
MI, we use the absolute difference rather than a one-sided information-loss assumption. Let
\begin{equation}
\bar\delta_{\mathcal S}
:=
\mathbb E_{\bX}
\left[
\delta_{\mathcal S}(\bq)
\right].
\end{equation}
Our core result (Sec.~\ref{sec:app_uib}) gives
\begin{equation}
\label{eq:mi_bound}
\begin{aligned}
\Delta I_{\mathcal S}
\le
g\!\left(\bar\delta_{\mathcal S}\right), \quad
g(\delta)
\coloneqq
2\left[
h_{\mathrm b}(\delta)+\delta\log L
\right].
\end{aligned}
\end{equation}
Here $h_{\mathrm b}$ is the binary entropy function and $L$ is
the number of eligible routing positions. The certificate therefore
controls perturbation of the finite attention-routing information,
rather than assuming a finite alphabet for the continuous routed
values. The empirical relation between routing-MI perturbation and
downstream quality is evaluated in
Sec.~\ref{sec:theory_practice_grounding}. For the sparse-attention regimes considered here, $g(\delta)$ is increasing over the relevant domain $\delta<L/(1+L)$.

\paragraph{Top-$k$ Oracle and Optimization Goal}
For a fixed budget $N$, define
\begin{equation}
\mathcal S^{*}(\bq)
=
\operatorname*{Top}_{N}\!\big(A(\bq)\big),
\end{equation}
with retained and dropped masses
\begin{equation}
\tau^{*}(\bq)
=
\tau_{\mathcal S^{*}}(\bq),
\qquad
\delta^{*}(\bq)
=
1-\tau^{*}(\bq).
\end{equation}
Pointwise in $\bq$, the top-$k$ oracle maximizes retained mass:
\begin{equation}
\begin{aligned}
\mathcal S^{*}(\bq)
&\in
\operatorname*{arg\,min}_{S\subseteq[L],\,|S|=N}
\left(
1-\sum_{i\in S}A_i(\bq)
\right)\\
&=
\operatorname*{arg\,max}_{S\subseteq[L],\,|S|=N}
\sum_{i\in S}A_i(\bq).
\end{aligned}
\end{equation}
Hence, under the matched budget,
$\mathcal S^{*}$ minimizes the dropped mass and therefore yields the
tightest certificate in Eq.~\eqref{eq:mi_bound}. This statement does
not require the stronger assumption that top-$k$ directly minimizes
the realized routing-MI perturbation or task error.

Because evaluating $A(\bq)$ exactly incurs substantial selection
cost, practical selectors approximate the oracle. We therefore
formulate selector design in terms of the retained-mass gap:
\begin{equation}
\label{eq:selector_opt}
\begin{aligned}
\min_{\mathcal S}\;
\mathbb E_{\bX}
\left[
\tau^{*}(\bq){-}\tau_{\mathcal S}(\bq)
\right] \quad
\text{s.t.}\;
|\mathcal S^{*}(\bq)|
{=}
|\mathcal S(\bq)|
{=}
N.
\end{aligned}
\end{equation}
Equivalently, this minimizes
$\bar\delta_{\mathcal S}-\bar\delta^{*}$ and therefore tightens the
routing-information certificate under the same budget.

\paragraph{PoHS Selector $\mathcal S_D$}
A PoHS selector
\begin{equation}
\mathcal S_D(\bq)
=
\operatorname*{Top}_{N}
\!\left(\widehat A_D(\bq)\right)
\end{equation}
replaces $A(\bq)$ with a surrogate
$\widehat A_D(\bq)$ derived from posterior side information $D$
(e.g., token ages or sketches). Its surrogate error is measured by
\begin{equation}
\label{eq:pohs_bias}
\varepsilon_D(\bq)
\coloneqq
\frac12
\left\|
A(\bq)-\widehat A_D(\bq)
\right\|_1 .
\end{equation}
The corresponding pointwise dropped-mass gap satisfies
\begin{equation}
\delta_{\mathcal S_D}(\bq)
\le
\delta^{*}(\bq)+2\varepsilon_D(\bq).
\end{equation}
Defining
\begin{equation}
\bar\varepsilon_D
:=
\mathbb E_{\bX}
\left[
\varepsilon_D(\bq)
\right],
\end{equation}
we obtain
\begin{equation}
\bar\delta_{\mathrm{post}}
\le
\bar\delta^{*}+2\bar\varepsilon_D.
\end{equation}
Therefore, by Eq.~\eqref{eq:mi_bound},
\begin{equation}
\label{eq:posh_bound}
\Delta I_{\mathrm{post}}
\le
g\!\left(\bar\delta_{\mathrm{post}}\right)
\le
g\!\left(
\bar\delta^{*}+2\bar\varepsilon_D
\right).
\end{equation}

Thus, surrogate mis-scoring enlarges the dropped-mass certificate
through the posterior-error term $2\bar\varepsilon_D$. We formalize
both TDO and QAA from Sec.~\ref{sec:intro} as PoHS instances.
Their detailed bounds are provided in
Sec.~\ref{sec:app_mib_tsa}.

\paragraph{PrHS Selector $\mathcal S_{\mathrm{pre}}$}
A PrHS selector $\mathcal S_{\mathrm{pre}}$ selects the critical set
without explicitly computing the full attention distribution
$A(\bq)$ for the query being selected and admits a pre-scoring mass error
$\beta_{\mathrm{th}}(\bq)\ge0$ satisfying
\begin{equation}
\label{eq:prhs_bound}
\begin{aligned}
\tau_{\mathrm{pre}}(\bq)
&\ge
\tau^{*}(\bq)-\beta_{\mathrm{th}}(\bq),\\
\delta_{\mathrm{pre}}(\bq)
&\le
\delta^{*}(\bq)+\beta_{\mathrm{th}}(\bq).
\end{aligned}
\end{equation}
Define
\begin{equation}
\bar\beta_{\mathrm{th}}
:=
\mathbb E_{\bX}
\left[
\beta_{\mathrm{th}}(\bq)
\right].
\end{equation}
Averaging Eq.~\eqref{eq:prhs_bound} gives
\begin{equation}
\bar\delta_{\mathrm{pre}}
\le
\bar\delta^{*}+\bar\beta_{\mathrm{th}},
\end{equation}
and hence
\begin{equation}
\label{eq:prhs_mi_bound}
\Delta I_{\mathrm{pre}}
\le
g\!\left(\bar\delta_{\mathrm{pre}}\right)
\le
g\!\left(
\bar\delta^{*}+\bar\beta_{\mathrm{th}}
\right).
\end{equation}

As $\bar\beta_{\mathrm{th}}\rightarrow0$, the PrHS certificate
converges to the oracle certificate:
\begin{equation}
\label{eq:limit_to_oracle}
\Delta I_{\mathrm{pre}}
\le
g(\bar\delta^{*})
\le
g\!\left(
\bar\delta^{*}+2\bar\varepsilon_D
\right).
\end{equation}
More generally, whenever $\bar\beta_{\mathrm{th}}<2\bar\varepsilon_D$, PrHS admits a tighter dropped-mass routing-information certificate than the corresponding PoHS bound under the same oracle term
$\bar\delta^{*}$.

This comparison is a certificate-level statement rather than a
direct guarantee of downstream task accuracy. Detailed formulations
are given in Sec.~\ref{sec:app_mib_tsa}, while
Sec.~\ref{sec:theory_practice_grounding} empirically connects
routing-information perturbation with downstream degradation.
Moreover, PrHS selectors can partially (CIS) or fully (PSAW and ETF)
avoid explicit full-attention retrieval, improving the
accuracy--efficiency trade-off in practice.

\newcommand{\bs}{\mathbf{s}}
\begin{figure*}[!t]
\includegraphics[width=1\textwidth]{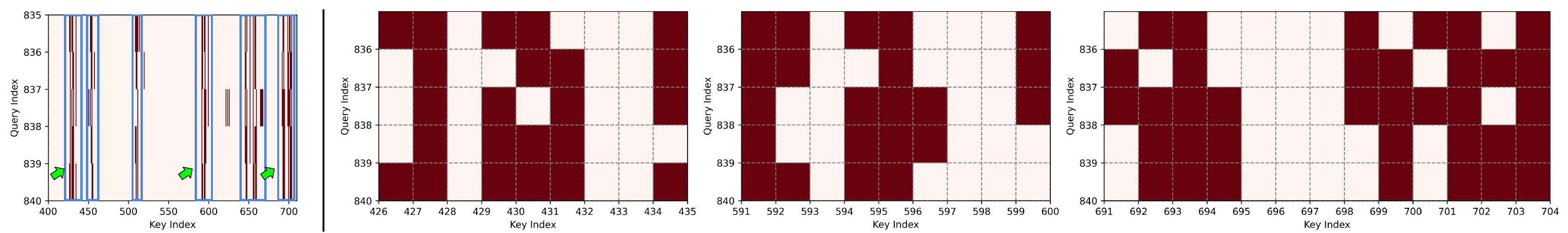}
\centering
\vspace{-1em}
    \caption{\textbf{Distribution of critical indices in LLaMA2-7B-Chat (Layer 10, Head 2).} For each query, we retrieve 64 critical indices using the top-$k$ oracle on WikiText2. Columns show five temporally adjacent queries with cosine similarity $>\!0.8$. \emph{Left}: overall distribution across keys 400–710; clusters are outlined in \textcolor{blue}{blue boxes}. \emph{Right}: zoom-ins of three clusters (annotated by \textcolor{green}{green arrows} in the left) at keys 426–436, 591–700, and 691–704. Critical indices are in \textcolor{red}{red}. Similar patterns across other models are shown in \textbf{Appendix G}.}
    \label{fig:attn-cluster}
    \vspace{-0.5em}
\end{figure*}


\section{Observations and Properties}
\label{sec:observations}

While empirical priors (temporal locality, recency, cross-layer redundancy) are suggestive, PrHS requires rules that can be applied pre-hoc before full attention scoring of the query being selected. Accordingly, from the transformer’s mechanics we derive two properties: (i) clustered locality and among-query continuity of critical indices, and (ii) progressive cross-layer propagation of information. These properties motivate our selector parameterization and, in
Sec.~\ref{sec:method}, yield pointwise retained-mass guarantees of the
form $\tau_{\mathrm{pre}}(\bq)\ge
\tau^*(\bq)-\beta_{\mathrm{th}}(\bq)$. Averaging these guarantees over query-context states controls the routing-information perturbation through Eq.~\eqref{eq:mi_bound}.

\begin{figure}[!t]
\includegraphics[width=0.45\textwidth]{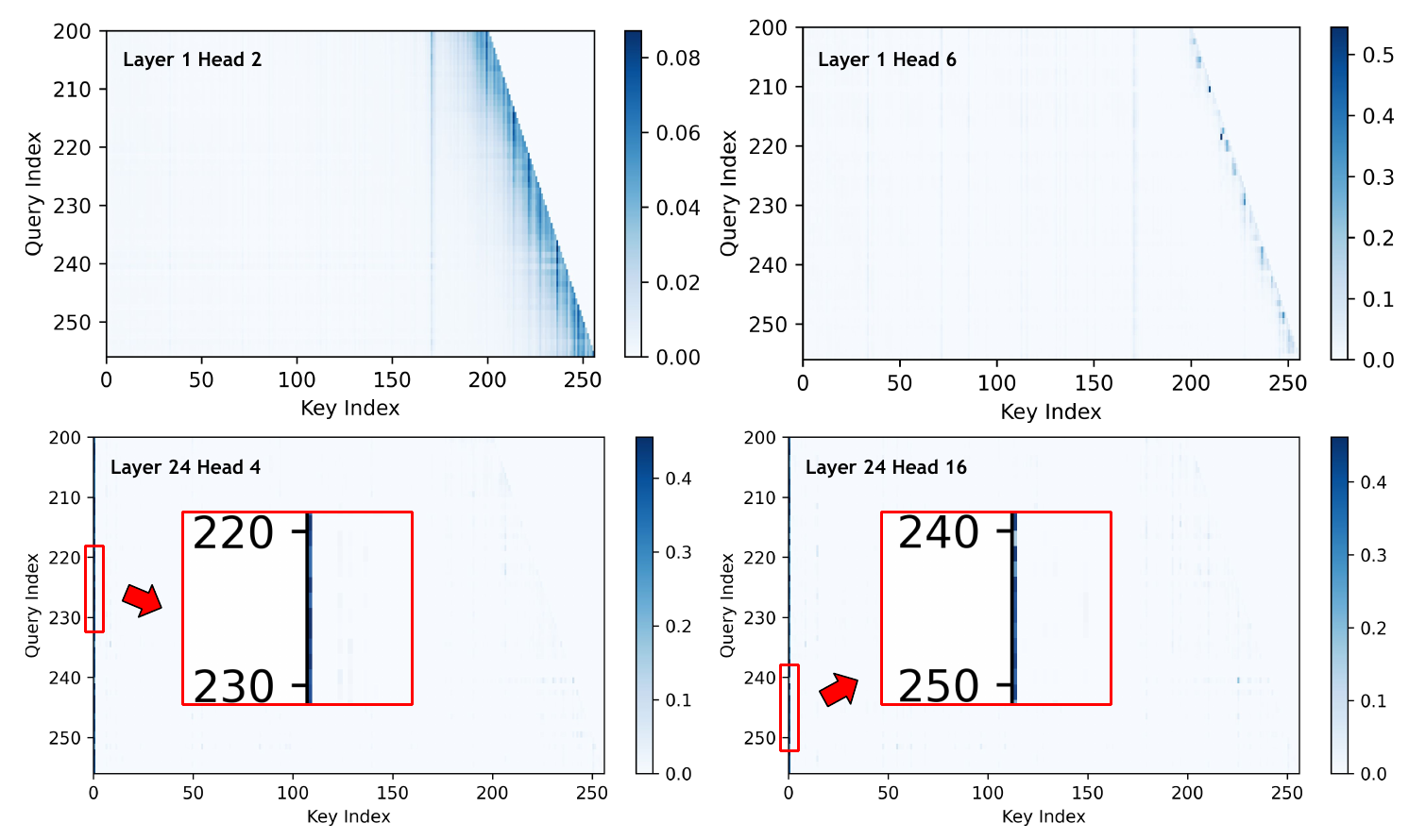}
\centering
\vspace{-0.5em}
    \caption{\textbf{Attention distribution on LLaMA2-7B-Chat.} Heatmaps across representative layer–head pairs, with darker color indicating stronger attention.}
    \label{fig:attn_dis}
    \vspace{-1em}
\end{figure}

\subsection{Observation 1: Critical Tokens Exist in Clustered Manners}

Prior work~\cite{sun2024shadowkv,li2024snapkv,wuhshare} shows that consecutive decoding queries are close in embedding space and often share critical indices, which motivates KV sharing strategies, detailed in \textbf{Appendix G}. However, \emph{per-head} selections can still vary markedly, and the overlap between selected indices remains small even for high cosine similarity between queries~\cite{wu2024tokenselect}, leading to
suboptimal sharing. For example, HShare~\cite{wuhshare} shares indices directly, missing many truly critical indices at high sharing ratios.
To bridge this gap, we observe that critical indices organize into a small number of clusters that persist across adjacent queries, providing a more stable sharing
unit than individual indices. 

\begin{property}[Among-query Property]
\label{prop:among-query}
For two semantically similar consecutive queries, the centroids of their critical-token clusters undergo a smooth mean shift along the sequence axis.
\end{property}

As shown in Fig.~\ref{fig:attn-cluster}, when the cosine similarity between $\bq_t$ and $\bq_{t+1}$ exceeds $0.8$, critical indices concentrate into a few clusters; although cluster centers and spreads shift slightly, the cluster-level overlap remains large. Our theoretical analysis further shows that these clusters satisfy a mean-shift (bounded drift) property, formalized in \textbf{Theorem}~\ref{the:mean-shift}.

\begin{theorem}[Centroid Drift]
\label{the:mean-shift}
Consider the scaled dot-product self-attention (cf. Eq.~\eqref{eq:attn_eq}) with fixed keys $\{\bk_j\}_{j=1}^t$, attention weights $A_j(\bq_t)$ for key $\bk_j$, and let $\{p_j\}_{j=1}^t$ be scalar token positions with
$\mathrm{diam}\,\mathcal{P}:=\max_j p_j-\min_j p_j$ for $\mathcal{P}=\{p_j\}$. Define the sequence-axis centroid $c(\bq_t)=\sum_{j} A_j(\bq_t)\,p_j$. For $\bq' = \bq_t + \Delta$ with $\|\Delta\|\ll 1$,
\begin{equation}
\label{eq:centroid_shift}
    c(\bq') = c(\bq_t) + O(\|\Delta\|),
\end{equation}
\textit{i.e., the attention centroid moves in a Lipschitz-continuous fashion with respect to the query. In particular, $|c(\bq')-c(\bq_t)| \le 2(\mathrm{diam}\,\mathcal{P})\,\frac{\max_j \|\bk_j\|}{\sqrt d}\,\|\Delta\|$.}
\end{theorem}
The drift is $\mathcal{O}(\|\Delta\|)$ with Lipschitz modulus $L_{\mathrm{drift}}
\le 2(\mathrm{diam}\,\mathcal P)
\frac{\max_j\|\bk_j\|}{\sqrt d}$.
This smooth drift is consistent with empirical attention patterns and with
StreamingLLM’s “keep a few initial tokens” observation~\cite{xiao2023streamingllm}.

In summary, \textbf{across heads and layers, the critical tokens of successive similar queries form dense, localized clusters}. We call this the \textbf{sink effect}: \textit{tokens near a high-attention location inherit elevated attention due to semantic proximity}. Exploring the effect enables explicit control under PrHS, our CIS operationalizes this by dilating shared sets to cover the local cluster and compensate for centroid drift, avoiding full per-head re-retrieval. Detailed proofs are in \textbf{Appendix A}.

\subsection{Observation 2: Information Propagates Progressively Among Layers}

Locality-bias analyses~\cite{zhang2023h2o,ge2023fastgen} show that attention is
sharply concentrated locally within a narrow window around the current query
and the sink token, across diverse prompts (Fig.~\ref{fig:attn_dis}).
This suggests that the current token draws most information from recent tokens,
whereas the sink token acts as an attractor rather than conveying semantic content.
Viewed in isolation, such constrained receptive fields would only support short-term memory for LLMs, seemingly at odds with their strong long-context performance.
Moreover, recent studies~\cite{cai2024pyramidkv,wan2024d2o,gao2023empower} show that:
across layers, intermediate tokens incrementally absorb the semantics of their relevant predecessors; by higher layers, a mid-sequence token encodes both its own content and a compressed synopsis of earlier, distant tokens.
It means that stacking transformer layers realizes long-range dependency paths through chains of neighboring tokens even when no individual layer attends broadly.
Consequently, salient long-distance information is distilled into nearby carrier
tokens, reducing the marginal importance of many earlier ones; hence,
later layers need not re-scan the full prefix.
Building on this insight, PSAW and ETF impose structured pruning
rules whose design-time analysis bounds the additional dropped mass
relative to the oracle, $\delta_{\mathrm{pre}}(\bq)\le \delta^*(\bq)+\beta_{\mathrm{th}}(\bq)$. After averaging over query--context states, this yields the routing-information certificate in Eq.~\eqref{eq:mi_bound}. We formalize these guarantees in \textbf{Appendix C--D}.

\begin{figure}[!t]
\includegraphics[width=0.48\textwidth]{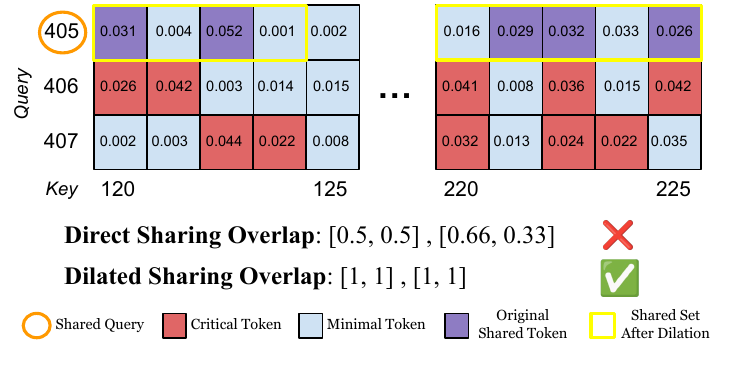}
\centering
\vspace{-0.5em}
    \caption{\textbf{Critical-index set (CIS) dilation.} For three adjacent queries (405–407), we examine two key-index clusters $[120,125]$ and $[220,225]$. Critical tokens from query 405 are shared to queries 406 and 407. Overlap is measured as $\frac{\text{Number of True Positive}}{\text{Number of Critical Tokens}}$. Tokens in \textcolor[HTML]{A64DFF}{violet} are also critical for query 405.}
    \label{fig:index-frag}
    \vspace{-1em}
\end{figure}

\vspace{-0.05in}
\section{Proposed Method and Implementation}
\vspace{-0.05in}
\label{sec:method}

In this section we present CPE. First, \textbf{CIS} (from Observation~1) enables efficient head-level KV sharing across adjacent queries. Next, \textbf{PSAW} and \textbf{ETF} (from Observation~2) prune redundant attention by shrinking visible attention window and freezing early token updates, respectively, with theoretical error guarantees. We also describe CUDA kernels for parallel index manipulation and sparse attention.

\subsection{Clustered Indices Sharing (CIS)}
\label{sec:cis}

CIS performs head-level KV sharing with dilated shared set among semantically similar, and temporally adjacent queries. By \textbf{Observation}~1 and \textbf{Theorem}~\ref{the:mean-shift}, the critical clusters for such queries undergo only small mean shifts relative to a reference query. Therefore, exhaustive per-query head-level retrieval is unnecessary. We first partition the sequence into fixed-length blocks and restrict sharing to within a block to enforce temporal adjacency, where semantically similar but non-adjacent (cross-block) queries are not shared.

\paragraph{Selection Criteria}
Within each block, the first query serves as an initial reference and computes its critical indices across layers and heads. Following~\cite{zhang2023h2o,wuhshare} and the locality bias in \textbf{Observation~2}, selection budgets are allocated to three token groups: initial (sink) tokens, local tokens, and salient middle tokens. Let $C_{\text{sink}}$ and $C_{\text{local}}$ denote the numbers of sink and local tokens, respectively. At decoding step $t$, we obtain the middle set with a budget $k$ by applying the top-$k$ oracle over positions $[C_{\text{sink}}, t\!-\!C_{\text{local}}]$, excluding the sink and local regions. This yields $S_t^*=\{p_{t,i}\}_{i=1}^{k}$, where $p_{t,i}$ is sorted and selected in descending order of attention weight. The per-head critical index set is $\mathcal{C}_t=\{1,\ldots,C_{\text{sink}}\} \cup S_t^* \cup \{t-C_{\text{local}}+1,\ldots,t\}$, with a total size $C=C_{\text{sink}}+k+C_{\text{local}}$.

\paragraph{Similarity Criteria}
For a later query, if there exists an earlier query $j$ whose cosine similarity exceeds a threshold $\tau$, KV sharing is performed. We measure similarity as
\begin{equation}
\begin{aligned}
\label{eq:cos_sim}
\mathsf{sim}(i,j)=\frac{\bq_{i}^{\!\top} \; \bq_{j}}{\|\bq_{i}\|\,\|\bq_{j}\|}.
\end{aligned}
\end{equation}
If multiple previous queries satisfy
$\mathrm{sim}(i,j)>\tau$, we select the candidate with the highest cosine similarity; ties are broken in favor of the most recent query to preserve temporal locality. 


\begin{figure*}[!t]
\includegraphics[width=0.95\linewidth]{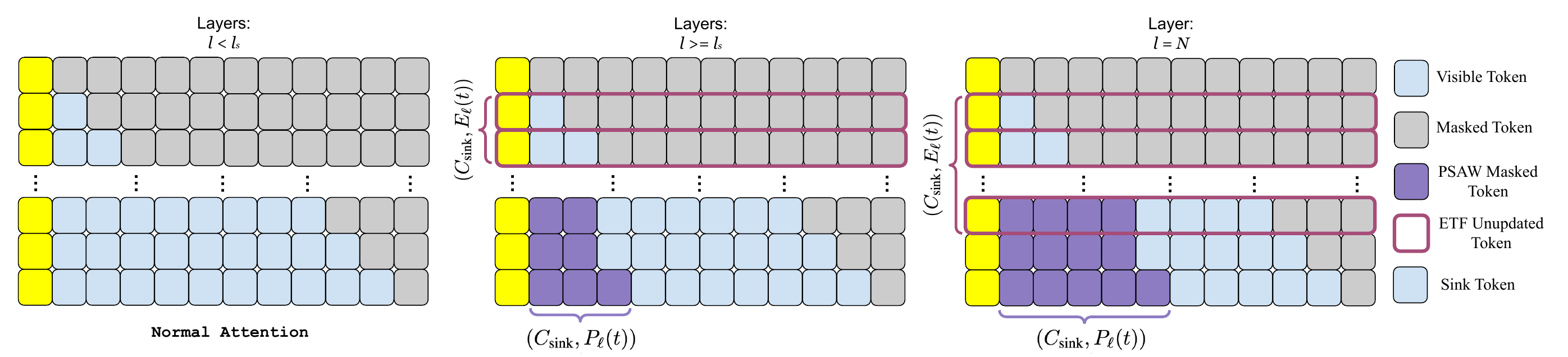}
\centering
    \caption{\textbf{Visual illustration of PSAW and ETF in prefill stage.} When $\ell<\ell_s$, attention is unchanged. For $\ell\ge\ell_s$, both methods prune redundant computation: \textbf{PSAW} computes a per-step sliding window, so the set of \textcolor{PSAW}{masked tokens} can vary across steps (columns); \textbf{ETF} applies a fixed prune range and freezes \textcolor{ETF}{earlier tokens} so they no longer update.}
    \label{fig:psaw&etf}
    \vspace{-0.5em}
\end{figure*}

\paragraph{Neighboring Dilation}
Naive sharing ignores centroid drift (Eq.~\eqref{eq:centroid_shift}) and yields fragmented critical sets (Figs.~\ref{fig:attn-cluster}, \ref{fig:index-frag}), missing true positives for shared queries. We therefore dilate the $m$ highest-weight indices in $S_t^*$ (with $m\!\le\!k$) by their $\pm r$ neighbors:
\begin{equation}
\label{eq:neighboring}
\begin{aligned}
\hat S_t 
\;=\; 
S_t^* 
\;\cup\;
\bigcup_{i=1}^{m}\big\{p_{t,i}+j \;|\;-r\le j\le r\,\big\}.
\end{aligned}
\end{equation}
As shown in Tab.~\ref{tab:matched_attribution_transfer}, dilation largely increases true-positive overlap and achieves complete coverage compared to direct sharing. Moreover, restricting dilation to the top-$m$ indices mitigates attention pollution from adding excessive low-weight tokens, as detailed in \textbf{Appendix G}.

\paragraph{Analysis}
\label{sec:cis_analysis}
Theoretically, CIS turns the query continuity underlying Property~\ref{prop:among-query} into a pre-hoc retained-mass guarantee for each reused query, as proved in Theorem~\ref{thm:cis-main} (detailed in \textbf{Appendix A, A3}).

\begin{theorem}[CIS Retained-Mass and Routing-MI Certificate]
\label{thm:cis-main}
For a reference query $\bq$, let $S_t^{\mathrm{CIS}}(\bq)$ denote the
actual CIS set constructed from the dilated set $\hat S_t$ in Eq.~\eqref{eq:neighboring} together with the retained sink/local indices, and
let $b_t=|S_t^{\mathrm{CIS}}(\bq)|$. Define its budget-matched oracle mass
$\tau_{b_t}^*(\bq)$ and reference gap
\[
\beta_{\mathrm{ref}}(\bq)
:=
\tau_{b_t}^*(\bq)
-
\tau_{S_t^{\mathrm{CIS}}(\bq)}(\bq)
\ge 0.
\]
For any later query $\bq'$ sharing this set with
$\mathsf{sim}(\bq,\bq')\ge\tau$, assume additionally, for constants
$\varepsilon_{\mathrm{norm}}\ge0$ and $Q_{\max}>0$, that
$|\|\bq'\|-\|\bq\||\le\varepsilon_{\mathrm{norm}}$ and
$\|\bq\|,\|\bq'\|\le Q_{\max}$. Then the CIS mass error satisfies
\begin{equation}
\begin{aligned}
&\beta_{\mathrm{th}}^{\mathrm{CIS}}(\bq')
:=
\tau_{b_t}^*(\bq')-\tau_{\mathrm{pre}}(\bq')\le
\beta_{\mathrm{ref}}(\bq)
+2\,\Delta_{\mathrm{att}},
\\
&\Delta_{\mathrm{att}}(\tau,\varepsilon_{\mathrm{norm}},Q_{\max})
:=
\frac{2K_{\max}}{\sqrt d}
\left(
\varepsilon_{\mathrm{norm}}
+Q_{\max}\sqrt{2-2\tau}
\right),
\end{aligned}
\end{equation}
where
$\|A(\bq')-A(\bq)\|_1
\le\Delta_{\mathrm{att}}(\tau,\varepsilon_{\mathrm{norm}},Q_{\max})$.
Here $\beta_{\mathrm{ref}}(\bq)$ is evaluated from the already-scored anchor
query $\bq$ and is therefore known before full scoring of the reused query
$\bq'$. Thus CIS satisfies an anchor-conditioned pre-hoc mass guarantee for
$\bq'$ in Eq.~\eqref{eq:prhs_bound}, with the oracle matched pointwise to the
actual budget $b_t$. Averaging over query--context states gives
$\bar\beta_{\mathrm{th}}^{\mathrm{CIS}}
\le\bar\beta_{\mathrm{ref}}
+2\Delta_{\mathrm{att}}(\tau,\varepsilon_{\mathrm{norm}},Q_{\max})$;
applying Eq.~\eqref{eq:prhs_mi_bound} yields the corresponding routing-MI
certificate.
\end{theorem}

In short, CIS combines similarity-induced attention continuity with the
structured dilation controlled by $(m,r)$: $(\tau,\varepsilon_{\mathrm{norm}},Q_{\max})$ bound attention variation, while $(m,r)$ determine the reused set and the
already-scored anchor determines its budget-matched reference gap
$\beta_{\mathrm{ref}}$. Empirically, the mean-shift rate saturates once the
prefix length exceeds a moderate scale; i.e., within a reasonably large block,
cluster distributions of aligned queries become nearly identical. During
experiments, $s\geq8$ and $\mathsf{sim}\geq80\%$ still achieve competitive
performance, corresponding to aggressive sharing ratios.

\subsection{Progressive Sliding Attention Window}
\label{sec:psaw}

PSAW progressively narrows each layer’s attention window \textbf{during both prefill and decoding}. Consistent with \textbf{Observation~2}, attention is largely localized~\cite{zhang2023h2o,wan2024d2o}: apart from sink tokens, very early tokens receive negligible mass from later queries, and their content has already been propagated forward in lower layers. Thus, attending to them again in deeper layers is redundant.

Let $P_{\ell}(t)$ denote the earliest visible (non-sink) position at step $t$ for layer $\ell$, and let $N$ be the total number of layer. We mask tokens in the index range $(C_{\text{sink}},P_{\ell}(t))$, making the visible set $\{1,\ldots, C_{\text{sink}}\}\cup \{P_{\ell}(t),\ldots,t\}$. Let $\ell_s$ be the layer depth at which pruning starts, we define
\begin{equation}
\label{eq:psaw}
P_\ell(t)=
\begin{cases}
0, &\ell < \ell_s
\\
\big\lfloor (1 \!-\!\phi\;^{\alpha\cdot\frac{\ell-\ell_s}{N-\ell_s}}) \, t \big\rfloor, &\ell \ge \ell_s
\end{cases}
\end{equation}
where $\phi\!\in\!(0,1)$ and $\alpha\!\ge\!0$. For $\ell\!<\!\ell_s$, no pruning is applied to preserve early-information propagation. For $\ell \!\ge\! \ell_s$, the boundary $P_\ell(t)$ moves monotonically forward with depth because $\phi^{\,\alpha\cdot\frac{\ell-\ell_s}{N-\ell_s}}$ decreases as $\ell$ increases, yielding a progressively shrinking window. We adopt an exponentially decaying schedule rather than a linear one to better match attention locality biases, e.g., positional long-term decay~\cite{su2024roformer} and softmax concentration, yielding tighter alignment with the empirical attention drop in long-range influence across layers (see \textbf{Appendix B}). The parameter $\alpha$ controls the decay schedule, while $\phi^\alpha$ sets the top-layer truncation strength. A visualization of the PSAW mask is shown in Fig.~\ref{fig:psaw&etf}, where \textcolor{PSAW}{tokens in soft violet} indicate masked tokens. 

Theoretically, under the exponential recency assumption beyond sink tokens, PSAW admits the dropped-mass certificate
$\delta_{\ell}^{\mathrm{PSAW}}
\le \kappa_\ell e^{-\lambda_\ell D_\ell}$,
where $D_\ell=t-P_\ell(t)+1$.
Thus, $(\phi,\alpha)$ directly control the truncation error; averaging the dropped mass over query--context states yields the corresponding routing-MI certificate through Eq.~\eqref{eq:mi_bound}, as detailed in \textbf{Appendix C}.

\subsection{Early Token Freezing}
\label{sec:etf}
During \textbf{prefill}, ETF reduces computation by freezing an expanding prefix of early tokens in deeper layers. The intuition is that information from early tokens has already been propagated forward in shallower layers; further updating them yields diminishing returns. This implementation is also supported by the cross-layer redundancy (Sec.~\ref{sec:intro})~\cite{liu2405minicache}. Let $E_\ell(t)$ denote the last frozen (non-sink) index at layer $\ell$ and step $t$. Tokens with positions in $(C_{\text{sink}},E_\ell(t)]$ are excluded from updates: they reuse their previous-layer hidden states, and their attention computations are skipped. Analogous to PSAW, we define
\begin{equation}
\label{eq:etf}
E_\ell(t)=
\begin{cases}
0, &\ell < \ell_s
\\
\big\lfloor (1 \!-\!\psi\;^{\gamma \cdot\frac{\ell-\ell_s}{N-\ell_s}}) \, t \big\rfloor, &\ell \ge \ell_s
\end{cases}
\end{equation}
where $\psi\!\in\!(0,1)$ controls the final unfrozen fraction, and $\gamma\!>\!0$ modulates the nonlinearity of the schedule. A visualization is provided in Fig.~\ref{fig:psaw&etf}, where \textcolor{ETF}{tokens with reddish-purple edges} are frozen. Under the cross-layer stability and exponential-recency assumptions, ETF admits a layerwise routing-perturbation certificate $\beta_{\ell}^{\mathrm{ETF}}$, which accounts for both the
frozen-prefix support truncation and the cross-layer perturbation
under the same incoming state. The bound can be tightened over active layers by increasing $\ell_s$ or $\mu$, while $(\psi,\gamma)$ control both the frozen-prefix truncation error and the frozen workload through $E_\ell(t)$, as detailed in \textbf{Appendix D}.

\begin{algorithm}[!t]
\caption{CPE Inference at Layer $\ell$}
\label{alg:cpe_inference}
\scriptsize
\algrenewcommand\algorithmicindent{0.75em}

\begin{algorithmic}[1]
\Require $Q,K,V$, CIS cache $\mathcal{C}_{\ell}$
\Statex \hspace{\algorithmicindent}
$\Theta=\{s,\tau,k,C_s,C_l,m,r,\ell_s,\phi,\alpha,\psi,\gamma\}$

\If{$|Q|>1$} \Comment{prefill}
    \State $\mathcal{M}_{\ell}\gets
    \Call{JointMask}{\ell,|K|,C_s,\ell_s,\phi,\alpha,\psi,\gamma}$
    \State \Return
    $\Call{FA2-CPE}{Q,K,V,\mathcal{M}_{\ell}}$
    \Comment{PSAW + ETF}
\EndIf

\State $t\gets|K|$, $j\gets t\bmod s$,
$p\gets\Call{PSAWStart}{\ell,t,\ell_s,\phi,\alpha}$
\State $\Omega_t\gets[1,C_s]\cup[t-C_l+1,t]$
\Comment{sink/local tokens}

\ForAll{heads $(b,h)$ \textbf{ in parallel}}
    \State $a\gets
    \Call{CISMatch}{q_{bh},\mathcal{C}_{\ell}[0{:}j],\tau}$

    \If{$a\neq\bot$} \Comment{shared head}
        \State $\mathcal{B}\gets
        \mathcal{C}_{\ell}[a].\mathrm{idx}$
        \State $\mathcal{D}\gets
        \Call{Dilate}{\operatorname{TopM}(\mathcal{B},m),r}$
    \Else \Comment{retrieval head}
        \State $\mathcal{B}\gets
        \Call{TopK}{
        q_{bh}K_{[\max(C_s+1,p),\,t-C_l]}^\top,k}$
        \State $\mathcal{C}_{\ell}[j]\gets(q_{bh},\mathcal{B})$,
        $\mathcal{D}\gets\emptyset$
    \EndIf

    \State $\mathcal{S}\gets
    \Call{SortUnique}{
    (\Omega_t\cup\mathcal{B}\cup\mathcal{D})
    \cap([1,C_s]\cup[p,t])}$
    \State $O_{bh}\gets
    \Call{SparseAttn}{
    q_{bh},K_{\mathcal{S}},V_{\mathcal{S}}}$
\EndFor

\State \Return $O$
\end{algorithmic}
\end{algorithm}

\subsection{CUDA Acceleration}
\label{sec:cuda_acceleration}

We implement CPE with customized CUDA kernels for both prefill and decoding. During prefill, PSAW and ETF are integrated into FlashAttention-2~\cite{dao2024flashattention}, allowing inactive query and KV tiles to be skipped according to their visibility mask. During decoding, CIS-shared heads directly reuse cached critical-index sets, while only retrieval heads perform full scoring and top-$k$ selection. The resulting sink, local, selected, and dilated indices are merged and processed by a fused sparse-attention kernel, with batch--head parallelism and warp-level reductions to reduce redundant computation and memory traffic. Algorithm~\ref{alg:cpe_inference} summarizes the execution pipeline of attention kernel. Note that Transformer layer implementations differ across model families. The CUDA pipeline described here focuses on the model-agnostic attention-kernel component of ETF. Reusing frozen hidden states, including bypassing the corresponding model-specific layer operations, requires architecture-dependent integration and is implemented separately for each model family; these layer-level details are therefore omitted from the generic pipeline.

\section{Experiments}

\begin{table}[t]
\centering
\scriptsize
\caption{\textbf{Evaluation of different methods on GSM8K and COQA.} The best result (excluding origin) in each column is highlighted in bold. Comp* refers to the theoretical time complexity for each method to select critical KV cache tokens, where $\mathcal{O}(1)$ denotes constant time complexity, and $T$ represents the theoretical computation time for a dense attention mechanism.}
\label{tab:gsm_coqa}
\begin{tabular}{l|l|cccccc}
\toprule
Model & Method & $\hat{\rho}$ & GSM8K $\uparrow$ & COQA $\uparrow$ & Comp* $\downarrow$ \\
\midrule
\multirow{11}{*}{\shortstack{LLaMA2\\-7b-chat}}
& Original  & -  & 0.2297/0.2297 & 0.5997 & - \\
& H2O \cite{zhang2023h2o}   & -        & 0.0986/0.0144 & 0.4952 & $\mathcal{O}(1)$ \\
& Quest \cite{tang2406quest}  & -       & 0.0478/0.0462 & 0.5713 & 0.1250$T$ \\
\cmidrule{2-6}
& Hshare-1 \cite{wuhshare} & 28.1  & 0.1803/0.1703 & 0.5898 & 0.0180$T$ \\
& Hshare-2 \cite{wuhshare} & 12.5   & 0.1524/0.1145 & 0.5672 & 0.0080$T$ \\
\cmidrule{2-6}
& \multirow{3}{*}{CIS (ours)}  & 16.4  & \textbf{0.1842}/\textbf{0.1827} & 0.6085 & 0.0104$T$ \\
&  & 8.7   & 0.1804/0.1782 & 0.6072 & 0.0055$T$ \\
&  & \textbf{6.2}   & 0.1759/0.1751 & 0.6068 & \textbf{0.0039}$T$ \\
\cmidrule{2-6}
& \multirow{3}{*}{CIS$^\ast$ (ours)}  & 16.1  & 0.1812/0.1804 & \textbf{0.6063} & 0.0103$T$ \\
&  & 8.4   & 0.1789/0.1782 & 0.6063 & 0.0047$T$ \\
&  & 6.3   & 0.1736/0.1729 & 0.6083 & 0.0040$T$ \\
\midrule
\multirow{10}{*}{\shortstack{Qwen3\\-4B}} 
& Original  & -   & 0.8438/0.8567 & 0.6730 & - \\
& Quest \cite{tang2406quest}  & -  & 0.4624/0.4516 & 0.5149 & 0.1250$T$ \\
\cmidrule{2-6}
& Hshare-1 \cite{wuhshare} & 28.1  & 0.8178/0.7756 & 0.6543 & 0.0180$T$ \\
& Hshare-2 \cite{wuhshare} & 12.5   & 0.7961/0.5481 & 0.6103 & 0.0080$T$ \\
\cmidrule{2-6}
& \multirow{3}{*}{CIS (ours)}  & 17.3  & 0.8317/0.8491 & 0.6763 & 0.0111$T$ \\
&  & 8.4   & 0.8211/0.8446 & 0.6763 & 0.0054$T$ \\
&  & \textbf{6.2}   & 0.8302/0.8476 & 0.6763 & \textbf{0.0040}$T$ \\
\cmidrule{2-6}
& \multirow{3}{*}{CIS$^\ast$ (ours)}  & 17.6  & \textbf{0.8385/0.8537} & \textbf{0.6768} & 0.0113$T$ \\
&  & 7.8   & 0.8226/0.8484 & 0.6763 & 0.0050$T$ \\
&  & 6.5   & 0.8256/0.8514 & 0.6763 & 0.0041$T$ \\
\midrule
\multirow{10}{*}{\shortstack{GLM4.5\\-Air}} 
& Original & -   & 0.8999/0.9151 & 0.7017 & - \\
& Quest \cite{tang2406quest}  & -       & 0.5212/0.5341 & 0.5467 & 0.1250$T$ \\
\cmidrule{2-6}
& Hshare-1 \cite{wuhshare} & 28.1   & 0.8652/0.8731 & 0.6824 & 0.0180$T$ \\
& Hshare-2 \cite{wuhshare} & 12.5   & 0.7887/0.7803 & 0.6429 & 0.0080$T$ \\
\cmidrule{2-6}
& \multirow{3}{*}{CIS (ours)} & 17.4  & \textbf{0.9022/0.9105} & \textbf{0.7037} & 0.0111$T$ \\
&  & 8.8   & 0.8946/0.9060 & 0.7022 & 0.0056$T$ \\
&  & \textbf{6.2}   & 0.8999/0.9067 & 0.7022 & \textbf{0.0040}$T$ \\
\cmidrule{2-6}
& \multirow{3}{*}{CIS$^\ast$ (ours)} & 16.8  & 0.9022/0.9098 & 0.6983 & 0.0107$T$ \\
&  & 8.9   & 0.8923/0.9014 & 0.6997 & 0.0056$T$ \\
&  & 6.5   & 0.8992/0.9075 & 0.6997 & 0.0041$T$ \\
\bottomrule
\end{tabular}
\end{table}

\subsection{Experimental Setup}
\label{sec:exp_setup}
We evaluate our CPE on three widely-used model families: LLaMA~\cite{touvron2023llama}, Qwen~\cite{team2025qwen3} and GLM~\cite{zeng2025glm}. Specifically, we select LLaMA2-7B-Chat, Qwen3-4B-Instruct-2507 (Qwen3-4B), Qwen2.5-7B-Instruct, and GLM-4.5-Air (108B) for our experiments. As baselines, we select one TDO method H2O~\cite{zhang2023h2o}, one QAA algorithm Quest~\cite{tang2406quest}, as well as a SOTA retrieval-based PoHS instantiation: HShare~\cite{wuhshare}. We report results on GSM8K~\cite{cobbe2021gsm8k}, COQA~\cite{reddy2019coqa}, and LongBench v2~\cite{bai2025longbench}.

In our CPE, we fix the following hyperparameters throughout all experiments. For CIS, we set the similarity threshold $\tau\!=\!0.8$, and apply dilation on the $m \!=\! \lfloor k/3\rfloor$ highest-score indices with radius $r=1$. Pruning and freezing of PSAW and ETF start from $\ell_s \!=\! \lfloor N/2 \rfloor$ for a model with $N$ layers. PSAW uses $\phi \!=\! 0.7$ and $\alpha \!=\! 1$, while ETF uses $\psi \!=\! 0.5$ and $\gamma \!=\! 1$. 

Because CIS performs head-level KV sharing to avoid full retrieval, we define the \emph{per-step retrieval ratio} $\rho_t\!=\!\frac{R_t}{H\cdot N}$ to measure the retrieval frequency, where $R_t$ is the number of performed head–layer retrieval at step $t$, and $H\cdot N$ is the amount required by a naïve top-$k$ oracle. The ratio $\rho_t$ is controlled by the CIS block size $s$ and the similarity threshold $\tau$. In reports, we present the averaged ratio $\hat{\rho}=\tfrac{1}{T}\sum^{T}\rho_t*100\%$ over the $T$ generation steps.

\subsection{Results on GSM8K and COQA}
\label{sec:gsm8k&coqa}

\subsubsection{Datasets and Metrics}
We perform zero-shot evaluation with LM-Eval on the GSM8K~\cite{cobbe2021gsm8k} and COQA~\cite{reddy2019coqa} benchmarks. GSM8K comprises roughly 8 000 grade-school mathematics problems, whereas COQA targets multi-turn conversational question answering. The mean input lengths are about 500 tokens (GSM8K) and 2 000 tokens (COQA). Following prior work, we report both flexible and strict exact-match (EM) accuracy on GSM8K, and EM on COQA.

\subsubsection{Inference Details}
We apply TSA in the decoding stage and solely evaluate CIS as presented in Table~\ref{tab:gsm_coqa}. We select a critical KV budget $C_t \!=\! 128$ ~\cite{wuhshare}. For CIS and Hshare, we set $C_{\text{sink}}{=}8$, $C_{\text{local}} \!=\! 32$, and $k \!=\! 88$ in the middle. Since CIS dilation incurs additional KV budget during sharing, we therefore additionally include a variant CIS$^\ast$, which sets $k=72$, yielding its average budget close to 128. We choose $s\in\{8,16,20\}$ to vary $\hat{\rho}$.

\newcommand{\cotcell}[2]{%
  \begingroup
  \setlength{\tabcolsep}{0pt}%
  \renewcommand{\arraystretch}{1}%
  \raisebox{0pt}[0pt][0pt]{%
    \begin{tabular}{@{}c@{\,$\vert$\,}c@{}}
      #1 & \cellcolor{gray!20}\strut #2
    \end{tabular}%
  }%
  \endgroup
}

\begin{table*}[t]

\footnotesize
\centering
\caption{\textbf{Evaluation of different methods on LongBench v2.}
Results under CoT prompting are highlighted with a gray background.
The best result in each row (excluding the original) is highlighted in bold.}
\label{tab:longbench}

\begin{tabular}{l|l|cc|cc|cc|cc|cc}
\toprule

\multicolumn{1}{c|}{\multirow{2}{*}{\textbf{Model}}}
&
\multicolumn{1}{|c|}{\multirow{2}{*}{\textbf{Eval}}}
&
\multicolumn{10}{|c}{\textbf{Method}}
\\

\cmidrule(lr){3-12}

&
&
Original
&
Quest\cite{tang2406quest}
&
\multicolumn{2}{c|}{HShare\cite{wuhshare}}
&
\multicolumn{2}{c|}{CIS (ours)}
&
\multicolumn{2}{c|}{CIS$^\ast$ (ours)}
&
\multicolumn{2}{c}{CPE (ours)}
\\

\midrule

\multirow{8}{*}{\shortstack{Qwen2.5\\-7B\\-Instruct}}
& $\hat{\rho}$
& &
& 28.1
& 12.5
& \cotcell{19.4}{16.9}
& \cotcell{13.1}{9.3}
& \cotcell{19.2}{17.2}
& \cotcell{12.8}{9.1}
& \cotcell{18.8}{16.2}
& \cotcell{13.4}{9.5}
\\
\cmidrule{2-12}

& Easy
& \cotcell{30.2}{37.5}
& \cotcell{15.8}{21.3}
& \cotcell{19.8}{28.6}
& \cotcell{16.4}{25.2}
& \cotcell{29.7}{\textbf{35.4}}
& \cotcell{29.2}{31.8}
& \cotcell{29.2}{33.2}
& \cotcell{29.7}{34.9}
& \cotcell{26.6}{31.8}
& \cotcell{26.0}{30.7}
\\

& Hard
& \cotcell{32.5}{28.3}
& \cotcell{22.9}{26.5}
& \cotcell{23.2}{28.0}
& \cotcell{20.7}{24.5}
& \cotcell{30.9}{29.6}
& \cotcell{30.5}{31.5}
& \cotcell{31.2}{27.0}
& \cotcell{30.5}{26.4}
& \cotcell{\textbf{33.4}}{32.8}
& \cotcell{\textbf{33.4}}{33.1}
\\

\cmidrule{2-12}

& Short
& \cotcell{41.1}{35.5}
& \cotcell{29.8}{30.2}
& \cotcell{35.5}{34.4}
& \cotcell{29.2}{31.5}
& \cotcell{41.6}{37.8}
& \cotcell{41.1}{37.2}
& \cotcell{\textbf{42.2}}{34.3}
& \cotcell{\textbf{42.2}}{35.0}
& \cotcell{40.5}{38.9}
& \cotcell{40.5}{36.6}
\\

& Medium
& \cotcell{27.9}{32.5}
& \cotcell{13.1}{24.5}
& \cotcell{15.8}{25.1}
& \cotcell{11.1}{23.6}
& \cotcell{24.6}{\textbf{30.7}}
& \cotcell{24.6}{29.8}
& \cotcell{25.6}{27.0}
& \cotcell{24.2}{26.5}
& \cotcell{25.1}{28.4}
& \cotcell{24.6}{30.7}
\\

& Long
& \cotcell{23.1}{24.1}
& \cotcell{10.4}{23.9}
& \cotcell{11.1}{24.1}
& \cotcell{9.3}{22.7}
& \cotcell{23.1}{24.1}
& \cotcell{22.2}{25.9}
& \cotcell{20.4}{25.3}
& \cotcell{20.4}{26.8}
& \cotcell{25.9}{\textbf{29.6}}
& \cotcell{25.9}{27.8}
\\

\cmidrule{2-12}

& Overall
& \cotcell{31.6}{31.8}
& \cotcell{20.2}{24.5}
& \cotcell{21.9}{28.2}
& \cotcell{19.1}{24.8}
& \cotcell{30.4}{31.8}
& \cotcell{30.0}{31.6}
& \cotcell{30.4}{29.2}
& \cotcell{30.2}{29.6}
& \cotcell{30.8}{\textbf{32.4}}
& \cotcell{30.6}{32.2}
\\

\midrule

\multirow{8}{*}{\shortstack{GLM-4.5\\-Air}}
& $\hat{\rho}$
& &
&
28.1
& 12.5
& \cotcell{20.5}{16.4}
& \cotcell{11.6}{9.2}
& \cotcell{19.7}{16.3}
& \cotcell{11.8}{9.5}
& \cotcell{20.7}{16.6}
& \cotcell{11.4}{10.1}
\\
\cmidrule{2-12}

& Easy
& \cotcell{44.3}{50.0}
& \cotcell{27.2}{29.3}
& \cotcell{33.8}{38.5}
& \cotcell{29.9}{30.4}
& \cotcell{45.3}{49.0}
& \cotcell{45.3}{48.8}
& \cotcell{45.3}{49.0}
& \cotcell{45.3}{49.5}
& \cotcell{42.7}{50.5}
& \cotcell{42.7}{\textbf{51.3}}
\\

& Hard
& \cotcell{38.3}{44.4}
& \cotcell{24.5}{26.7}
& \cotcell{31.6}{34.4}
& \cotcell{25.8}{27.6}
& \cotcell{38.9}{\textbf{47.9}}
& \cotcell{38.9}{\textbf{47.9}}
& \cotcell{38.9}{44.1}
& \cotcell{38.9}{46.3}
& \cotcell{37.3}{45.3}
& \cotcell{37.0}{44.4}
\\

\cmidrule{2-12}

& Short
& \cotcell{47.2}{55.0}
& \cotcell{29.4}{31.2}
& \cotcell{35.2}{40.3}
& \cotcell{32.2}{33.6}
& \cotcell{48.3}{\textbf{56.6}}
& \cotcell{48.3}{56.2}
& \cotcell{48.3}{48.9}
& \cotcell{48.3}{51.2}
& \cotcell{46.6}{52.7}
& \cotcell{46.6}{52.5}
\\

& Medium
& \cotcell{35.3}{40.4}
& \cotcell{23.7}{24.9}
& \cotcell{29.6}{33.7}
& \cotcell{24.4}{27.1}
& \cotcell{36.3}{\textbf{44.2}}
& \cotcell{36.3}{43.8}
& \cotcell{35.8}{42.8}
& \cotcell{35.8}{42.6}
& \cotcell{32.5}{41.4}
& \cotcell{32.1}{39.8}
\\

& Long
& \cotcell{39.8}{44.4}
& \cotcell{23.4}{25.3}
& \cotcell{30.6}{33.8}
& \cotcell{23.8}{26.9}
& \cotcell{39.8}{42.6}
& \cotcell{39.8}{42.3}
& \cotcell{40.7}{47.2}
& \cotcell{40.7}{48.4}
& \cotcell{40.7}{50.0}
& \cotcell{40.7}{\textbf{50.2}}
\\

\cmidrule{2-12}

& Overall
& \cotcell{40.6}{46.5}
& \cotcell{25.5}{27.7}
& \cotcell{32.4}{35.9}
& \cotcell{27.4}{28.7}
& \cotcell{41.4}{\textbf{48.3}}
& \cotcell{41.4}{48.2}
& \cotcell{41.4}{45.9}
& \cotcell{41.4}{47.5}
& \cotcell{39.4}{47.3}
& \cotcell{39.2}{47.3}
\\

\bottomrule
\end{tabular}

\end{table*}

\subsubsection{Result Analysis}

As shown in Tab.~\ref{tab:gsm_coqa}, CIS and CIS$^{*}$ consistently maintain strong accuracy under aggressive KV sharing and outperform other sparse baselines. Compared with HShare-1 at $\hat{\rho}=28.1\%$, CIS requires only $16.4\%$-$17.4\%$ retrieval, reducing retrieval by about $38$-$42\%$ while achieving higher accuracy on both GSM8K and COQA. The advantage remains under more aggressive sharing: CIS at $\hat{\rho}=8.4\%$-$8.8\%$ consistently outperforms HShare-2 at $12.5\%$ retrieval, with particularly large gains on Qwen3-4B and GLM-4.5-Air. Even when $\hat{\rho}$ is further reduced to around $6\%$, CIS retains strong performance, e.g., $0.8302/0.8476$ on Qwen3-4B and $0.8999/0.9067$ on GLM-4.5-Air for GSM8K. CIS$^{*}$ exhibits a similar trend under the budget-matched setting. Overall, CIS avoids more than $90\%$ of head-layer retrievals while preserving substantially higher accuracy than approximate sparse baselines, demonstrating stable KV selection across different model architectures.

\subsection{Longbench v2}
\label{sec:longbench}

\subsubsection{Datasets and Metrics}
LongBench v2 contains 503 multiple-choice questions spanning single-document QA, multi-document QA, long in-context learning, long-dialogue history understanding, code-repository understanding, and long structured-data understanding. We adhere to the official evaluation protocol and report classification accuracy overall, as well as accuracy across difficulty levels (easy and hard) and context-length groups (short$<$32k, medium 32k-128k, and long$>$128k).

\begin{table*}[t]
\centering
\caption{Attention operator latency (ms $\downarrow$) across batch sizes and sequence lengths. Flex-BS denotes FlexAttention with BlockMask using the same logical sparse pattern as CPE; BlockMask construction is excluded from operator timing. HShare configuration follows Tab.~\ref{tab:gsm_coqa}. The suffix of CIS/CPE denotes the CIS block size $s$.}
\label{tab:attention_latency}
\begin{tabular}{cc|cccccccccc}
\toprule
BS & Seqlen
& Flash
& Flex-BS
& H2O \cite{zhang2023h2o}
& Quest \cite{tang2406quest}
& DS \cite{yang2024post}
& HShare-0 \cite{wuhshare}
& CIS-8
& CIS-16
& CPE-8
& CPE-16 \\
\midrule

\multirow{3}{*}{8}
& 2k
& 0.830
& 0.119
& \textbf{0.093}
& 0.460
& 0.241
& 0.160
& 0.147
& 0.106
& 0.144
& 0.095 \\

& 4k
& 1.630
& 0.560
& 0.470
& 0.850
& 0.733
& 0.570
& 0.534
& 0.501
& 0.506
& \textbf{0.469} \\

& 16k
& 7.140
& 2.500
& \textbf{2.080}
& 3.420
& 3.160
& 2.430
& 2.280
& 2.250
& 2.220
& 2.110 \\

\midrule

\multirow{3}{*}{16}
& 2k
& 1.630
& 0.200
& \textbf{0.110}
& 0.770
& 0.422
& 0.230
& 0.214
& 0.178
& 0.196
& 0.165 \\

& 4k
& 3.230
& 1.000
& 0.850
& 2.210
& 1.350
& 1.041
& 0.953
& 0.869
& 0.904
& \textbf{0.821} \\

& 16k
& 14.220
& 4.550
& 4.040
& 10.030
& 5.940
& 4.850
& 4.430
& 4.130
& 4.270
& \textbf{3.910} \\

\bottomrule
\end{tabular}
\end{table*}

\subsubsection{Inference Details}

All baselines apply TSA only during decoding and retain a $4416$ KV budget, with a sparsity ratio of ${\approx}$1/16. For HShare and CIS we set $C_{\text{sink}}{=}64$, $C_{\text{local}}{=}256$, and $k{=}4096$. With CIS’s neighbor dilation, the effective shared-head set averages $5032$ tokens. To enable a budget-matched comparison, CIS$^\ast$ reduces the middle-budget to $k{=}3456$, equalizing the average KV budget processed during decoding to $4416$. For the combined system \textbf{CPE}, we use $k{=}4096$ and also activate PSAW and ETF during \emph{prefill}. Reductions achieved in prefill are not counted toward the decoding-budget metric, so CPE further lowers prefill cost in addition to the standard decoding-only comparison. We use block size $s\in\{8,16\}$.

\textit{3) Results Analysis:}
As shown in Tab.~\ref{tab:longbench}, Quest suffers severe performance collapse as context complexity and length increase, particularly on the Medium and Long subsets. HShare is relatively more robust at the less aggressive retrieval point ($\hat{\rho}{=}28.1\%$), especially for Short contexts, but degrades substantially on longer inputs and deteriorates further at $\hat{\rho}{=}12.5\%$. In contrast, CIS and CIS$^{*}$ preserve near-dense accuracy with only approximately $9\%$-$20\%$ retrieval, substantially outperforming both
baselines and even exceeding HShare at its more expensive 28.1\% retrieval point. CPE also remains clearly superior despite additionally applying PSAW and ETF during prefill, indicating that the proposed PrHS selectors preserve the critical retained attention mass across both prefill and decoding.

On GLM-4.5-Air, CIS and CIS$^{*}$ show almost identical non-CoT results across different retrieval ratios. Inspection of the generated responses shows that GLM-4.5-Air typically emits its prediction within the first few decoding tokens, usually within eight tokens; since CIS preserves nearly complete semantic information during these early steps, aggressive sharing rarely changes the final answer. With CoT prompting and longer reasoning trajectories, the advantage becomes more pronounced: CIS reaches $48.3/48.2$ overall, exceeding the dense baseline of $46.5$, while CIS$^{*}$ reaches $47.5$ under the more aggressive setting. Remarkably, CPE still obtains $47.3$ overall and reaches $50.0/50.2$ on the Long subset, compared with $44.4$ for dense attention. These results suggest that PrHS not only preserves the decisive attention mass under aggressive sparsification, but can also suppress low-relevance context while retaining information required for long-context reasoning.

\begin{table*}[t]
\centering
\caption{End-to-end decoding throughput ($\uparrow$) across batch sizes and sequence lengths. Flex-BS replaces our fused sparse-attention kernel with FlexAttention BlockMask while retaining the same CPE selection pipeline. The suffix of CIS/CPE denotes the CIS block size $s$.}
\label{tab:throughput}
\begin{tabular}{cc|cccccccccc}
\toprule
BS & Seqlen
& GPT-Fast
& Flex-BS
& H2O \cite{zhang2023h2o}
& Quest \cite{tang2406quest}
& DS \cite{yang2024post}
& HShare-0 \cite{wuhshare}
& CIS-8
& CIS-16
& CPE-8
& CPE-16 \\
\midrule

\multirow{3}{*}{8}
& 2k
& 188
& 228
& \textbf{234}
& 206
& 213
& 222
& 216
& 224
& 225
& 233 \\

& 4k
& 118
& 210
& \textbf{228}
& 152
& 201
& 214
& 216
& 220
& 219
& 227 \\

& 16k
& 33
& 80
& \textbf{92}
& 59
& 68
& 83
& 86
& 87
& 88
& 91 \\

\midrule

\multirow{3}{*}{16}
& 2k
& 233
& 408
& 416
& 287
& 360
& 398
& 402
& 415
& 412
& \textbf{421} \\

& 4k
& 136
& 338
& \textbf{396}
& 175
& 286
& 350
& 352
& 364
& 362
& 384 \\

& 16k
& 34
& 99
& 112
& 47
& 79
& 95
& 102
& 108
& 105
& \textbf{114} \\

\bottomrule
\end{tabular}
\end{table*}

\subsection{Efficiency Evaluation} 

\subsubsection{Inference Detail} 

Following prior work~\cite{yang2024post,wuhshare}, we evaluate (i) self-attention operator latency and (ii) end-to-end decoding throughput. Experiments use \mbox{LLaMA-2-7B-Chat} on an Intel Xeon Platinum 8336C CPU and a single NVIDIA A100 GPU. FlashAttention-2 (Flash)~\cite{dao2022flashattention} is the optimized dense operator baseline. We additionally evaluate FlexAttention~\cite{dong2025flexattention} with BlockMask (Flex-BS) as an optimized block-sparse backend. For a controlled implementation comparison, Flex-BS uses the same logical sparse pattern produced by CPE, while its BlockMask skips fully inactive KV blocks. For operator latency, mask construction is excluded from timing; for end-to-end evaluation, Flex-BS replaces our fused sparse-attention kernel in the same decoding pipeline. GPT-Fast~\cite{pytorch} serves as the dense end-to-end baseline. We test batch sizes $\{8,16\}$ and sequence lengths $1\text{k}$-$16\text{k}$ at a token sparsity ratio of $1/8$. CIS and CPE (CIS+PSAW; ETF is prefill-only) use block sizes $s\in\{8,16\}$.

\subsubsection{Results Analysis}

\paragraph{Operator latency}
Table~\ref{tab:attention_latency} shows that CPE-16 remains best or near-best across context lengths. At $2\text{k}$ and BS$=16$, it achieves a $9.9\times$ speedup over dense Flash. At $16\text{k}$, CPE-16 is $3.38\times$ and $3.64\times$ faster than Flash at BS$=8$ and $16$, while reducing latency over Flex-BS by $15.6\%$ and $14.1\%$, respectively. It also reduces latency over HShare-0 by $13.2\%$--$19.4\%$; at BS$=16$, it is $34.2\%$ faster than DS and slightly faster than H2O. This comparison highlights the additional benefit of our token-level CUDA backend under the same logical sparsity pattern.

\paragraph{End-to-end throughput}
As shown in Table~\ref{tab:throughput}, these operator gains translate to decoding throughput. At $16\text{k}$, CPE-16 improves throughput over Flex-BS by $13.8\%$ and $15.2\%$ at BS$=8$ and $16$, and exceeds HShare-0 by $9.6\%$--$20.0\%$. At BS$=16$, it reaches $3.35\times$ the throughput of GPT-Fast and slightly surpasses H2O. The smaller end-to-end gains reflect unchanged non-attention costs such as MLPs and sampling.

\paragraph{Effect of block size}
Across settings, $s{=}16$ provides better efficiency than $s{=}8$. At BS$=16$ and $16\text{k}$, CPE-16 reduces operator latency by $8.4\%$ and improves throughput by $8.6\%$ over CPE-8, confirming the benefit of amortizing sharing and index-management overhead over larger blocks.

\definecolor{orange}{HTML}{FF5C33}
\definecolor{darkgreen}{HTML}{2E7D32}

\begin{table*}[t]
    \centering
    \caption{
        \textbf{Controlled attribution and hyperparameter-transfer analysis.} Panel A compares CIS with budget-, retrieval-, and implementation-matched controls. Panel B compares PSAW with static and linear windows under matched query-key FLOPs. $\overline{\beta}_{\mathrm{emp}}$ denotes the mean of the empirical retained-mass gap. Accuracy retention is averaged over Qwen2.5-7B-Instruct, Qwen3-4B.
    }
    \label{tab:matched_attribution_transfer}
    \setlength{\tabcolsep}{3.5pt}
    \renewcommand{\arraystretch}{1.08}

    \textbf{Panel A: CIS matched-control attribution}\\

    \footnotesize{
    \begin{tabular}{lccccccc}
    \hline
    Variant
    & Parameter setting
    & $\hat{\rho}\downarrow$
    & Avg.\ KV
    & $\overline{\beta}_{\mathrm{emp}}\downarrow$
    & GSM8K $\uparrow$
    & COQA $\uparrow$
    & LongBench v2 $\uparrow$
    \\
    \hline
    
    HShare-matched
    & Matched
    & 10.0
    & 1060
    & 0.061
    & 0.7002
    & 0.5550
    & \cotcell{24.25}{27.35}
    \\
    
    DirectReuse
    & Global-frozen
    & 10.3
    & 1060
    & 0.022
    & 0.7907
    & 0.6068
    & \cotcell{31.30}{34.40}
    \\
    
    RandomExpansion
    & Global-frozen
    & 10.1
    & 1060
    & 0.041
    & 0.7187
    & 0.5454
    & \cotcell{29.15}{31.40}
    \\
    
    CIS
    & Global-frozen
    & 10.1
    & 1064
    & 0.019
    & 0.8007
    & 0.6050
    & \cotcell{31.85}{35.40}
    \\
    
    CIS
    & Per-model tuned
    & 10.2
    & 1065
    & 0.018
    & 0.8051
    & 0.6082
    & \cotcell{31.95}{35.85}
    \\
    
    \hline
    \end{tabular}
    }

    \vspace{1mm}
    \textbf{Panel B: Window-schedule attribution under matched FLOPs}\\

    \footnotesize{
    \begin{tabular}{lcccccc}
        \hline
        Variant
        & Parameter setting
        & $\overline{\beta}_{\mathrm{emp}}\downarrow$
        & WikiText PPL $\downarrow$
        & LongBench v2 \ $\uparrow$
        & Avg.\ latency / throughput
        \\
        \hline

        Static Window
        & FLOPs-matched
        & 0.0169 & 7.81 & \cotcell{33.2}{34.6} & 0.037 / 43
        \\

        Linear Window
        & FLOPs-matched
        & 0.0128 & 7.75 & \cotcell{33.2}{34.8} & 0.039 / 41
        \\

        PSAW
        & Global-frozen
        & 0.0091 & 7.74 & \cotcell{33.2}{35.0} & 0.040 / 42
        \\

        PSAW
        & Per-model tuned
        & 0.0078 & 7.74 & \cotcell{33.2}{35.0} & 0.039 / 41
        \\

        \hline
    \end{tabular}
    }
\end{table*}

\section{Further Analysis}
\label{sec:further_analysis}

We further conduct controlled experiments to identify the source of the improvements achieved by PrHS. Specifically, we examine whether the gains can be explained by additional KV capacity, retrieval frequency, window size, or model-specific hyperparameter tuning rather than improved selection quality. We first analyze CIS under retrieval- and budget-matched controls, then compare PSAW with alternative window schedules under matched attention FLOPs, evaluate cross-model hyperparameter transfer, and finally connect the retained-mass certificate to empirical routing-information perturbation and downstream degradation. \textbf{Appendix G} analyses examine hyperparameter tuning and similarity thresholds, characterize query-similarity distributions and critical-index clustering on modern models, and verify CIS robustness across a wide KV-budget sweep.

\subsection{Matched-Control Attribution for CIS}
\label{sec:cis_matched_control}

\paragraph{Controlled variants.}
We construct four variants using the same sparse-attention backend,
sink/local-token budgets, and effective processed-KV budget. Their averaged
retrieval ratios $\hat{\rho}$ are matched within $0.5\%$ at each controlled
operating point.
\textbf{HShare-matched} reparameterizes HShare to match the retrieval ratio
and processed-KV budget of CIS.
\textbf{DirectReuse} directly reuses the anchor query's critical indices
without neighboring dilation; since dilation introduces no extra indices,
its middle-token retrieval budget $k$ is increased by approximately $1/6$
to match the processed-KV budget of CIS.
\textbf{RandomExpansion} replaces neighboring dilation with the same number
of uniformly sampled indices from the eligible non-sink and non-local
region; results are averaged over 42 random seeds.
\textbf{CIS} uses the complete similarity-gated sharing and top-$m$
neighboring dilation.

These controls isolate the corresponding effects. DirectReuse tests whether
index sharing alone explains the improvement, RandomExpansion controls for
the additional KV capacity introduced by dilation, and HShare-matched
provides a retrieval- and implementation-matched comparison with the
existing sharing strategy.

\paragraph{Empirical retained-mass gap.}
Following Eq.~\eqref{eq:mass_def}, we measure the deviation of a selector from
the top-$k$ oracle under its actual KV budget. For query $q_t$, let $S_t$
denote the selected set and $b_t=|S_t|$ its effective budget. The
budget-matched oracle is
\begin{equation}
S_{b_t}^{*}(q_t)
=
\operatorname{Top}_{b_t}\!\left(A(q_t)\right).
\label{eq:budget_matched_oracle_set}
\end{equation}
We define the empirical retained-mass gap as
\begin{equation}
\beta_{\mathrm{emp}}(q_t)
=
\tau_{S_{b_t}^{*}}(q_t)-\tau_{S_t}(q_t)
=
\delta_{S_t}(q_t)-\delta_{S_{b_t}^{*}}(q_t)
\geq 0.
\label{eq:empirical_mass_gap}
\end{equation}
A smaller $\beta_{\mathrm{emp}}$ indicates closer retained mass to the
budget-matched oracle. Using the actual $b_t$ accounts for the
query-dependent candidate-set size introduced by dilation. We report
$\overline{\beta}_{\mathrm{emp}}$ averaged over decoding steps, layers,
and attention heads.

\begin{figure*}[!t]
  \centering
  \captionsetup[subfloat]{font=scriptsize}
  \subfloat[Downstream accuracy.]{
  \footnotesize
  \label{fig:retri_acc_comp}
    \includegraphics[width=0.22\textwidth]{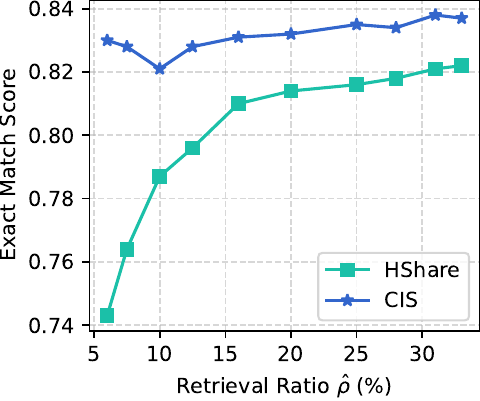}
  }\hfill
  \subfloat[Empirical retained-mass gap.]{
  \footnotesize
  \label{fig:retri_mass_comp}
    \includegraphics[width=0.22\textwidth]{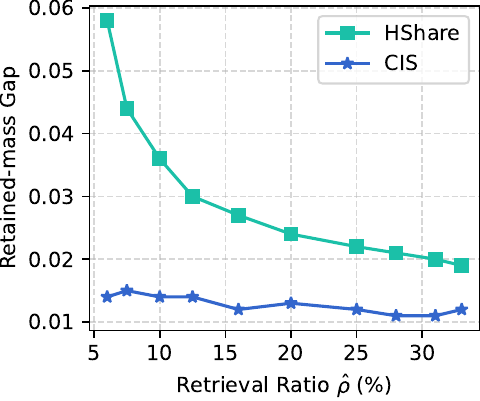}
  }\hfill
  \subfloat[Critical-index overlap with the top-$k$ oracle.]{
  \footnotesize
  \label{fig:retri_overlap_comp}
    \includegraphics[width=0.22\textwidth]{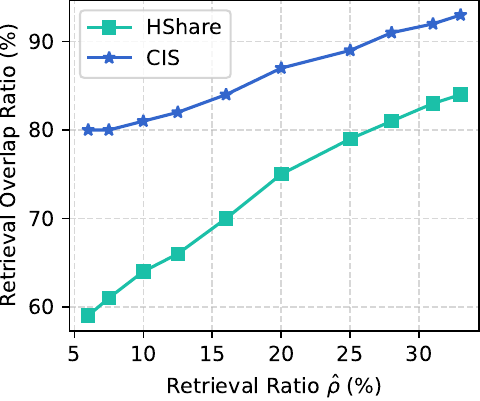}
  }\hfill
  \subfloat[Overlap with the top-50\% highest-mass oracle indices.]{
  \footnotesize
  \label{fig:retri_top_comp}
    \includegraphics[width=0.22\textwidth]{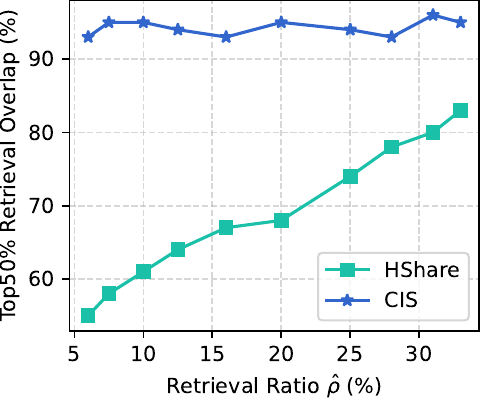}
  }
  \vspace{-0.05in}
  \caption{\textbf{Accuracy and KV-selection quality across retrieval ratios on GSM8K with Qwen3-4B.} CIS consistently maintains higher accuracy and a smaller retained-mass gap than HShare, together with higher overlap with the top-$k$ oracle critical indices.}
  \label{fig:retrival_comp}
  \vspace{-0.1in}
\end{figure*}

\paragraph{Retrieval-ratio analysis.}
Figure~\ref{fig:retrival_comp} compares CIS and HShare over a wide range of
retrieval ratios on GSM8K with Qwen3-4B. CIS maintains substantially more
stable downstream accuracy under increasingly aggressive retrieval, whereas
HShare degrades rapidly at low $\hat{\rho}$
(Fig.~\ref{fig:retri_acc_comp}). This behavior is reflected in retained
attention mass: the empirical mass gap of CIS remains low and nearly stable,
while that of HShare increases sharply as retrieval becomes sparser
(Fig.~\ref{fig:retri_mass_comp}).

The critical-index statistics provide a consistent explanation. CIS maintains
higher overlap with the budget-matched top-$k$ oracle throughout the sweep
(Fig.~\ref{fig:retri_overlap_comp}). More importantly, among the top 50\%
oracle indices ranked by attention weight, CIS preserves above 90\% overlap
even at aggressive retrieval ratios, whereas HShare loses a substantial
fraction of these high-importance indices
(Fig.~\ref{fig:retri_top_comp}). Thus, the reduced mass error of CIS is
primarily associated with preserving influential KV entries rather than
uniformly increasing index overlap.

\paragraph{Matched-control results.}
Panel~A of Table~\ref{tab:matched_attribution_transfer} further isolates
this effect at approximately $\hat{\rho}=10\%$, where all variants process
about $1.06$K KV entries per query. HShare produces an average retained-mass
gap of $0.061$, compared with $0.019$ for CIS, a $68.9\%$ reduction. At the
same resource point, CIS improves GSM8K from $0.7002$ to $0.8007$ and COQA
from $0.5550$ to $0.6050$.

DirectReuse substantially reduces the gap relative to HShare, confirming
that query-level sharing itself is beneficial. However, its
$\overline{\beta}_{\mathrm{emp}}=0.022$ remains above CIS despite using a
larger base retrieval budget to compensate for the absence of dilation.
RandomExpansion yields a larger gap of $0.041$ and lower downstream scores
while processing essentially the same number of KV entries as CIS. CIS
therefore reduces the empirical mass gap by $53.7\%$ relative to
RandomExpansion, showing that the gain from dilation is attributable to
placing expanded indices around high-importance clusters rather than simply
increasing KV capacity.

\subsubsection{Matched-FLOPs Attribution for PSAW}
\label{sec:psaw_matched_flops}

To determine whether PSAW benefits merely from using a sliding window, we
compare it with two alternative schedules under the same attention-computation
budget.
\textbf{Static Window} applies a constant visible window to all layers after
$\ell_s$.
\textbf{Linear Window} reduces the visible window linearly with depth.
\textbf{PSAW} uses the proposed exponentially depth-progressive schedule in
Eq.~\eqref{eq:psaw}.
All variants preserve the same sink tokens and pruning start layer $\ell_s$
and process approximately the same total number of query--key pairs.

Panel~B of Table~\ref{tab:matched_attribution_transfer} shows that the window
schedule has a measurable effect even under matched FLOPs. Static Window and
Linear Window yield
$\overline{\beta}_{\mathrm{emp}}=0.0169$ and $0.0128$, respectively,
whereas globally configured PSAW reduces the gap to $0.0091$, corresponding
to reductions of $46.2\%$ and $28.9\%$. Their runtime and WikiText perplexity
remain comparable, indicating that the lower mass error is not obtained by
increasing attention computation.

The same trend appears on LongBench v2. Although downstream differences are
modest, PSAW achieves the strongest CoT result under the matched computation
budget. This suggests that allocating a larger receptive field to shallower
layers and pruning more aggressively in deeper layers better matches the
depth-wise distribution of useful attention mass than generic static or
linear schedules.

\subsubsection{Cross-Model Hyperparameter Transfer}
\label{sec:hyperparameter_transfer}

The Global-Frozen configuration used throughout the main experiments
(Sec.~\ref{sec:exp_setup}) is calibrated once on LLaMA2-7B-Chat using an
unlabeled WikiText2 subset. Rather than downstream accuracy, calibration uses
retained attention mass. At a sparsity ratio of $1/8$, a configuration is
accepted when
\begin{equation}
\overline{\tau}_{S}
\geq
\overline{\tau}^{*}-\epsilon_{\mathrm{cal}},
\label{eq:retained_mass_calibration}
\end{equation}
where $\overline{\tau}_{S}$ and $\overline{\tau}^{*}$ are the average
retained masses of the evaluated selector and its budget-matched oracle.
We set $\epsilon_{\mathrm{cal}}=0.005$ for CIS and
$\epsilon_{\mathrm{cal}}=0.01$ for CPE. Once selected, the hyperparameters
are frozen and transferred unchanged across models and downstream benchmarks.

We additionally construct a \textbf{Per-Model Tuned} control to quantify the
benefit of architecture-specific calibration. For each model, we search the
same predefined hyperparameter space on a small unlabeled subset while
holding sparsity, resource budget, and $\epsilon_{\mathrm{cal}}$ fixed.
Thus, only structural hyperparameters change, without using downstream task
labels.

As shown in Panel~A of Table~\ref{tab:matched_attribution_transfer},
model-specific tuning provides only limited improvement over Global-Frozen
CIS. The empirical mass gap decreases from $0.019$ to $0.018$, while GSM8K
and COQA improve from $0.8007$ to $0.8051$ and from $0.6050$ to $0.6082$,
respectively; the LongBench v2 improvement is similarly small. Thus, most
of the selection-quality gain is preserved when parameters calibrated on
LLaMA2/WikiText2 are transferred directly to newer model families.

A similar pattern is observed for PSAW in Panel~B. Per-model tuning further
reduces $\overline{\beta}_{\mathrm{emp}}$ from $0.0091$ to $0.0078$ while
leaving WikiText perplexity and LongBench v2 accuracy essentially unchanged.
Hence, although model-specific calibration can further tighten the
retained-mass approximation, the downstream accuracy--efficiency trade-off
is largely preserved by globally frozen parameters.

\begin{figure}[t]
    \centering
    \includegraphics[width=\columnwidth]{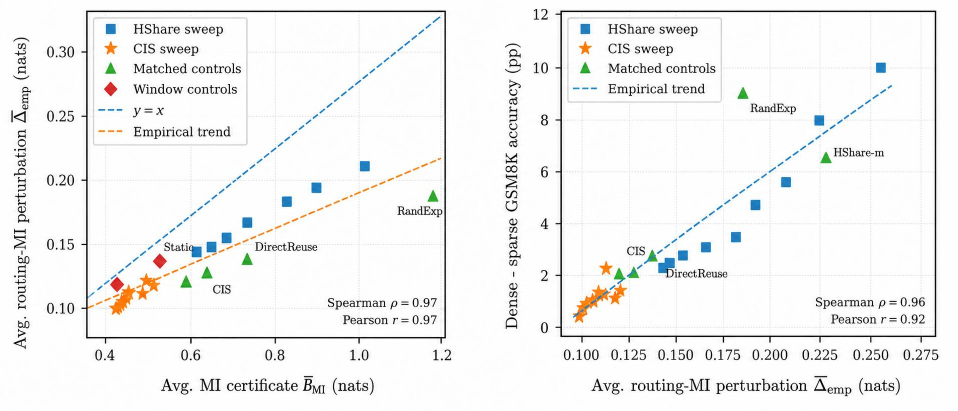}
    \vspace{-1.5em}
    \caption{
    \textbf{Empirical grounding of the information certificate.}
    Left: averaged routing-MI perturbation
    $\overline{\Delta I}_{\mathrm{emp}}$ versus the averaged theoretical
    certificate $\overline{B}_{\mathrm{MI}}$; the dashed line denotes
    $y=x$.
    Right: routing-MI perturbation versus downstream degradation
    (dense-minus-sparse GSM8K accuracy on Qwen3-4B).
    The certificate is conservative but order-preserving, while larger
    routing-information perturbations are associated with larger task-level
    degradation.
    }
    \label{fig:mi_grounding}
    \vspace{-2mm}
\end{figure}

\subsection{Theory-Practice Grounding}
\label{sec:theory_practice_grounding}

To directly align the empirical analysis with the theoretical certificate,
we evaluate mutual information on the finite
\emph{attention-index channel} in Sec.~\ref{sec:app_uib}.
For a fixed layer--head pair $(\ell,h)$, consider $M$ query--context states
with the same eligible context length $L$. Let $X\in[M]$ denote the state
with $P(X=m)=1/M$, and let $T,T_{\mathcal S}\in[L]$ denote the full and
sparse routing indices:
\begin{equation}
\begin{aligned}
P(T=i\mid X=m) &= A_{m,i},\\
P(T_{\mathcal S}=i\mid X=m)
&=\widetilde A_{m,i}
=\frac{A_{m,i}\mathbf{1}\{i\in\mathcal S_m\}}
{\tau_{\mathcal S_m}(\bq_m)} .
\end{aligned}
\label{eq:emp_mi_dist}
\end{equation}
Defining
$\bar A=M^{-1}\sum_m A_m$ and
$\overline{\widetilde A}_{\mathcal S}
=M^{-1}\sum_m\widetilde A_m$, the two MIs are computed exactly as
\begin{align}
\widehat I_{\mathrm{full}}
&=H(\bar A)-\frac{1}{M}\sum_{m=1}^{M}H(A_m), \nonumber\\
\widehat I_{\mathcal S}
&=H(\overline{\widetilde A}_{\mathcal S})
-\frac{1}{M}\sum_{m=1}^{M}H(\widetilde A_m).
\label{eq:emp_mi}
\end{align}
Since truncation and renormalization need not monotonically decrease MI, we
measure the absolute routing-information perturbation
\begin{equation}
\Delta I_{\mathrm{emp}}
=
\left|
\widehat I_{\mathrm{full}}
-
\widehat I_{\mathcal S}
\right|.
\label{eq:emp_mi_perturbation}
\end{equation}

Because both channels share the same input distribution, their joint
total-variation distance is
\begin{equation}
\left\|P_{X,T}-P_{X,T_{\mathcal S}}\right\|_{\mathrm{TV}}
=
\bar\delta_{\mathcal S}
=
\frac{1}{M}\sum_{m=1}^{M}\delta_{\mathcal S_m}(\bq_m)
=
\bar\delta^{*}+\bar\beta_{\mathrm{emp}},
\label{eq:emp_delta_beta}
\end{equation}
where
\begin{equation}
\bar\delta^{*}
=
\frac{1}{M}\sum_{m=1}^{M}
\delta^{*}_{b_m}(\bq_m),
\qquad
b_m=|\mathcal S_m|,
\end{equation}
and $\bar\beta_{\mathrm{emp}}$ is the corresponding average
budget-matched retained-mass gap in
Eq.~\eqref{eq:empirical_mass_gap}.
Applying the MI continuity bound on the common finite output alphabet
$[L]$ gives
\begin{equation}
\Delta I_{\mathrm{emp}}
\le
B_{\mathrm{MI}}
:=
2\!\left[
h_{\mathrm b}(\bar\delta_{\mathcal S})
+\bar\delta_{\mathcal S}\log L
\right].
\label{eq:emp_mi_bound}
\end{equation}
Thus, the certificate controls perturbation of the
\emph{routing information} induced by KV selection, without requiring a
finite-alphabet assumption on the continuous routed values. Its relation to
end-task quality is evaluated empirically below.

For variable-length inputs, we partition the samples into fixed-length
strata $c\in\mathcal C$. Within each stratum, Eq.~\eqref{eq:emp_mi_bound}
is evaluated using its corresponding $L_c$. We then report
\begin{equation}
\begin{aligned}
\left(
\overline{\Delta I}_{\mathrm{emp}},
\overline B_{\mathrm{MI}}
\right)
=
\sum_{c\in\mathcal C}
\frac{w_c}{|\mathcal L|H}
\sum_{\ell\in\mathcal L}\sum_{h=1}^{H}
\left(
\Delta I_{\mathrm{emp}}^{(c,\ell,h)},
B_{\mathrm{MI}}^{(c,\ell,h)}
\right),
\end{aligned}
\label{eq:emp_mi_agg}
\end{equation}
where $w_c=M_c/\sum_{c'}M_{c'}$ is the sample fraction of stratum $c$.

\paragraph{Empirical validation of the information certificate.}
Figure~\ref{fig:mi_grounding} (left) compares
$\overline{\Delta I}_{\mathrm{emp}}$ with
$\overline B_{\mathrm{MI}}$ across retrieval sweeps, matched selector
controls, and window-schedule controls. All operating points lie below the
identity line, consistent with Eq.~\eqref{eq:emp_mi_bound}. Although the
certificate is conservative, it is strongly rank-correlated with the
measured perturbation: larger retained-mass errors generally correspond to
larger routing-information deviations, whereas CIS and PSAW remain in the
low-certificate and low-perturbation region. Hence, the dropped-mass
certificate provides worst-case control while remaining informative about
the relative magnitude of routing-information perturbation.

\paragraph{From information perturbation to downstream degradation.}
Figure~\ref{fig:mi_grounding} (right) relates
$\overline{\Delta I}_{\mathrm{emp}}$ to dense-to-sparse GSM8K accuracy
degradation on Qwen3-4B. Larger routing-information perturbations are
generally associated with larger task-level degradation, yielding a strong
positive rank correlation. A strictly linear relation is not expected,
since downstream behavior also depends on token semantics, layer-wise error
propagation, and autoregressive dynamics. Accordingly, routing MI is treated
as an information-preservation surrogate rather than a direct accuracy
guarantee. The empirical results therefore support
\[
\bar{\delta}_{\mathcal S}
=
\bar{\delta}^{*}+\bar{\beta}_{\mathrm{emp}}
\;\rightarrow\;
\Delta I_{\mathrm{emp}}
\le
B_{\mathrm{MI}}
\;\overset{\mathrm{empirical}}{\leftrightarrow}\;
\text{degradation},
\]
linking the pre-hoc certificate to both information-level and task-level
behavior.

\section{Conclusion}

In this paper, we propose Pre-hoc Sparsity (PrHS), a framework for KV selection with explicit control of attention-routing information perturbation. We derive an upper bound on the absolute routing-MI perturbation in terms of the average dropped attention mass; under a fixed KV budget, the top-$k$ oracle minimizes this mass, while PrHS approaches its certificate by controlling the retained-mass gap. We instantiate PrHS as \textbf{CPE} with three complementary modules: \emph{Clustered Indices Sharing} (CIS), \emph{Progressive Sliding Attention Window} (PSAW), and \emph{Early Token Freezing} (ETF), which exploit temporal locality, depth-wise recency, and cross-layer stability to reduce redundant retrieval and attention computation. Empirically, the proposed certificate tracks measured routing-information perturbation and downstream degradation. Across LLaMA, Qwen, and GLM models on GSM8K, COQA, and LongBench v2, CPE avoids ${>}90\%$ of head-layer retrievals while preserving near-dense accuracy, reduces retrieval complexity by up to $3\times$ over HShare~\cite{wuhshare}, and remains effective on contexts exceeding $128\text{k}$ tokens. With our CUDA implementation, CPE further achieves up to a $9.9\times$ attention-operator speedup and $3.3\times$ end-to-end throughput gain over dense baselines.

\section{Unified Information Bounds}
\label{sec:app_uib}

\subsection{Setup}

We consider the minimal unit of the information bound and fix both
the layer and the attention head. Following the softmax attention
distribution defined in Sec.~\ref{sec:inference_background} and
Eq.~\eqref{eq:attn_eq}, let $\bX=(\bq,\bK,\bV)$ denote a random query-context state, where
$\bK=[\bk_1,\dots,\bk_L]$ and
$\bV=[\bv_1,\dots,\bv_L]$ are the keys and values of a context with length $L$. A realization is denoted by $\bx=(\bq,\bk,\bv)$. For notational consistency with the main text, we write
$A_i(\bq)$ for the attention probability associated with the
current realization $\bx$, with the corresponding key context
$\bK$ understood implicitly.

We introduce the finite \emph{routing-index} random variable
\begin{equation}
    T\in[L],\qquad
    \mathbb P(T=i\mid \bX)=A_i(\bq).
\label{eq:routing_channel}
\end{equation}

For a size-$n$ selector
$\mathcal S(\bq)\subseteq[L]$, define the retained and dropped
attention masses as
\begin{equation}
    \tau_{\mathcal S}(\bq)
    =
    \sum_{i\in\mathcal S(\bq)}A_i(\bq),
    \qquad
    \delta_{\mathcal S}(\bq)
    =
    1-\tau_{\mathcal S}(\bq).
\label{eq:routing_mass}
\end{equation}
The corresponding truncated and renormalized routing distribution is
\begin{equation}
    \widetilde A_i(\bq)
    =
    \begin{cases}
    A_i(\bq)/\tau_{\mathcal S}(\bq),
        & i\in\mathcal S(\bq),\\[2pt]
    0,  & \text{otherwise},
    \end{cases}
    \;
    T_{\mathcal S}\mid\bX \!=\! \bx
    \!\sim\!
    \widetilde A_{\mathcal S}(\bq).
\label{eq:sparse_routing_channel}
\end{equation}

We define the full and sparse routing information as
\begin{equation}
    I_{\mathrm{full}}
    :=
    I(\bX;T),
    \qquad
    I_{\mathcal S}
    :=
    I(\bX;T_{\mathcal S}),
\end{equation}
and their absolute routing-information perturbation as
\begin{equation}
    \Delta I_{\mathcal S}
    :=
    \left|
    I_{\mathrm{full}}-I_{\mathcal S}
    \right|.
\label{eq:routing_mi_perturbation}
\end{equation}
We use the absolute difference because truncation and renormalization
need not monotonically decrease mutual information.

Let $\mathcal S^{*}(\bq)=\operatorname*{Top}_{n}\!\big(A(\bq)\big)$ be the top-$n$ oracle set, with
\begin{equation}
\label{eq:oracle_attn_mass}
    \tau^{*}(\bq)
    =
    \tau_{\mathcal S^{*}}(\bq),
    \qquad
    \delta^{*}(\bq)
    =
    1-\tau^{*}(\bq).
\end{equation}
For selectors with a query-dependent effective budget, the same
definition applies pointwise with $n$ replaced by the actual budget
$|\mathcal S(\bq)|$, as used in
Eq.~\eqref{eq:budget_matched_oracle_set}.

Finally, define the expected dropped masses over the
query--context distribution as
\begin{equation}
    \bar\delta_{\mathcal S}
    :=
    \mathbb E_{\bX}
    \left[
    \delta_{\mathcal S}(\bq)
    \right],
    \qquad
    \bar\delta^{*}
    :=
    \mathbb E_{\bX}
    \left[
    \delta^{*}(\bq)
    \right].
\label{eq:avg_dropped_mass}
\end{equation}

\begin{remark}[We bound information via the routing-index channel]
The certificate below controls mutual information carried by the
finite attention-routing channel $\bX\rightarrow T$.
We do not require a finite-alphabet assumption on the continuous
routed values $\bV_T$, nor do we claim a general MI bound for the
complete routed-value channel. The relation between preservation of
routing information and downstream model quality is evaluated
empirically in Sec.~\ref{sec:theory_practice_grounding}.
\end{remark}

\subsection{Two Universal Technical Lemmas}
\label{sec:app_tech_lemmas}

\begin{lemma}[Total variation of truncation]
\label{lem:tv}
For every realization $\bx$ and selector $\mathcal S$,
\begin{equation}
    \left\|
    A(\cdot\mid\bX=\bx)
    -
    \widetilde A_{\mathcal S}(\cdot\mid\bX=\bx)
    \right\|_{\mathrm{TV}}
    =
    \delta_{\mathcal S}(\bq).
\label{eq:tv_pointwise}
\end{equation}
Consequently,
\begin{equation}
    \left\|
    P_{\bX,T}
    -
    P_{\bX,T_{\mathcal S}}
    \right\|_{\mathrm{TV}}
    =
    \bar\delta_{\mathcal S}.
\label{eq:tv_joint}
\end{equation}

\begin{proof}
Fix $\bx$ and write
$P=A(\cdot\mid\bX=\bx)$ and
$Q=\widetilde A_{\mathcal S}(\cdot\mid\bX=\bx)$. Then
\begin{align*}
\|P-Q\|_{\mathrm{TV}}
&=
\frac12\sum_{i=1}^{L}|P_i-Q_i|\\
&=
\frac12
\left(
\sum_{i\notin\mathcal S}P_i
+
\sum_{i\in\mathcal S}
\left|
P_i-\frac{P_i}{\tau_{\mathcal S}}
\right|
\right)\\
&=
\frac12
\left(
\delta_{\mathcal S}+\delta_{\mathcal S}
\right)
=
\delta_{\mathcal S}.
\end{align*}
Because the two channels share the same input distribution
$P_{\bX}$,
\begin{align*}
\left\|
P_{\bX,T}-P_{\bX,T_{\mathcal S}}
\right\|_{\mathrm{TV}}
&=
\mathbb E_{\bX}
\left[
\left\|
A(\cdot\mid\bX)
-
\widetilde A_{\mathcal S}(\cdot\mid\bX)
\right\|_{\mathrm{TV}}
\right]\\
&=
\mathbb E_{\bX}
\left[
\delta_{\mathcal S}(\bq)
\right]
=
\bar\delta_{\mathcal S}.
\end{align*}
\end{proof}
\end{lemma}

\begin{lemma}[Continuity of routing mutual information]
\label{lem:continuity}
For the full and sparse routing channels defined above,
\begin{equation}
    \Delta I_{\mathcal S}
    \le
    2\left[
    h_{\mathrm b}(\bar\delta_{\mathcal S})
    +
    \bar\delta_{\mathcal S}\log L
    \right].
\label{eq:routing_mi_continuity}
\end{equation}

\begin{proof}
For convenience, define
\begin{equation}
    f_L(\delta)
    :=
    h_{\mathrm b}(\delta)+\delta\log L.
\end{equation}
For each realization $\bx$, Lemma~\ref{lem:tv} and the
finite-alphabet entropy-continuity inequality give
\begin{equation}
\left|
H(T\mid\bX=\bx)
-
H(T_{\mathcal S}\mid\bX=\bx)
\right|
\le
f_L\!\left(\delta_{\mathcal S}(\bq)\right).
\end{equation}
Averaging over $\bX$ and using the concavity of
$f_L$ yields
\begin{align}
\left|
H(T\mid\bX)
-
H(T_{\mathcal S}\mid\bX)
\right|
&\le
\mathbb E_{\bX}
\left[
f_L\!\left(\delta_{\mathcal S}(\bq)\right)
\right]
\nonumber\\
&\le
f_L(\bar\delta_{\mathcal S}).
\label{eq:cond_entropy_routing}
\end{align}

Moreover, marginalization cannot increase total variation.
Therefore, Eq.~\eqref{eq:tv_joint} implies
\begin{equation}
    \left\|
    P_T-P_{T_{\mathcal S}}
    \right\|_{\mathrm{TV}}
    \le
    \bar\delta_{\mathcal S}.
\end{equation}
Since both $T$ and $T_{\mathcal S}$ take values in the common
finite alphabet $[L]$,
\begin{equation}
\left|
H(T)-H(T_{\mathcal S})
\right|
\le
f_L(\bar\delta_{\mathcal S}).
\label{eq:marg_entropy_routing}
\end{equation}

Using
$I(\bX;T)=H(T)-H(T\mid\bX)$,
\begin{align*}
\Delta I_{\mathcal S}
&=
\left|
I(\bX;T)-I(\bX;T_{\mathcal S})
\right|\\
&\le
\left|
H(T)-H(T_{\mathcal S})
\right|
+
\left|
H(T\mid\bX)
-
H(T_{\mathcal S}\mid\bX)
\right|\\
&\le
2f_L(\bar\delta_{\mathcal S}),
\end{align*}
which proves Eq.~\eqref{eq:routing_mi_continuity}.
\end{proof}
\end{lemma}

Define
\begin{equation}
    g(\delta)
    :=
    2\left[
    h_{\mathrm b}(\delta)+\delta\log L
    \right].
\label{eq:g_def}
\end{equation}
The universal routing-information certificate is therefore
\begin{equation}
    \Delta I_{\mathcal S}
    \le
    g(\bar\delta_{\mathcal S}).
\label{eq:unified_routing_certificate}
\end{equation}
Moreover,
\begin{equation}
    g'(\delta)
    =
    2\log
    \frac{L(1-\delta)}{\delta},
\end{equation}
so $g$ is increasing for
$\delta<L/(L+1)$, which includes the sparse-attention regimes
considered in this work.

\section{Information Bounds of Token-Sparse Attention}
\label{sec:app_mib_tsa}

For the selector of token-sparse attention, we have:

\begin{proposition}[Universal routing-MI bound and KL identity]
\label{prop:universal}
For any selector $\mathcal S$,
\begin{equation}
    \Delta I_{\mathcal S}
    \le
    g(\bar\delta_{\mathcal S}).
\tag{U1}
\end{equation}
Moreover,
\begin{equation}
D_{\mathrm{KL}}
\left(
P_{\bX,T_{\mathcal S}}
\middle\|
P_{\bX,T}
\right)
=
\mathbb E_{\bX}
\left[
\log\frac{1}{\tau_{\mathcal S}(\bq)}
\right].
\tag{U2}
\end{equation}

\begin{proof}
(U1) follows directly from Lemma~\ref{lem:continuity}.
For (U2), since both joint distributions share the same
$P_{\bX}$,
\begin{align*}
D_{\mathrm{KL}}
\left(
P_{\bX,T_{\mathcal S}}
\middle\|
P_{\bX,T}
\right)
&=
\mathbb E_{\bX}
D_{\mathrm{KL}}
\left(
\widetilde A_{\mathcal S}(\cdot\mid\bX)
\middle\|
A(\cdot\mid\bX)
\right)\\
&=
\mathbb E_{\bX}
\left[
\sum_{i\in\mathcal S}
\frac{A_i(\bq)}{\tau_{\mathcal S}(\bq)}
\log
\frac{A_i(\bq)/\tau_{\mathcal S}(\bq)}
     {A_i(\bq)}
\right]\\
&=
\mathbb E_{\bX}
\left[
\log\frac{1}{\tau_{\mathcal S}(\bq)}
\right].
\end{align*}
This KL identity quantifies the change of the routing channel,
but does not impose an ordering between
$I_{\mathrm{full}}$ and $I_{\mathcal S}$.
\end{proof}
\end{proposition}

\subsection{\texorpdfstring{Oracle Top-$k$ Selection}
{Oracle Top-k Selection}}

\begin{theorem}[Oracle top-$k$ routing-information bound]
\label{thm:oracle}
Let
\[
\mathcal S^{*}(\bq)
=
\operatorname*{Top}_{n}\!\big(A(\bq)\big).
\]
Then
\begin{equation}
    \Delta I_{\mathcal S^{*}}
    \le
    g(\bar\delta^{*}),
\tag{O1}
\end{equation}
and
\begin{equation}
D_{\mathrm{KL}}
\left(
P_{\bX,T_{\mathcal S^{*}}}
\middle\|
P_{\bX,T}
\right)
=
\mathbb E_{\bX}
\left[
\log\frac{1}{\tau^{*}(\bq)}
\right].
\tag{O2}
\end{equation}

\begin{proof}
For every realization $\bx$, among all selectors with
$|\mathcal S|=n$, $\mathcal S^{*}$ maximizes
$\tau_{\mathcal S}(\bq)$ and therefore minimizes
$\delta_{\mathcal S}(\bq)$. Hence
\begin{equation}
    \bar\delta^{*}
    \le
    \bar\delta_{\mathcal S}
\end{equation}
for every selector with the same budget. Since $g$ is increasing
in the regime considered here, $\mathcal S^{*}$ yields the tightest
dropped-mass routing-information certificate. The two statements
follow from Proposition~\ref{prop:universal}.
\end{proof}
\end{theorem}

Importantly, the oracle minimizes the dropped mass and therefore
the associated upper certificate; we do not require the stronger
claim that top-$k$ directly minimizes the realized
routing-MI perturbation.

\subsection{Post-Hoc Sparsity}

Post-hoc methods construct a surrogate probability vector
$\widehat A_D(\bq)$ from posterior side information $D$ and select
\begin{equation}
    \mathcal S_D(\bq)
    =
    \operatorname*{Top}_{n}
    \!\big(\widehat A_D(\bq)\big).
\end{equation}
Define the posterior retained-mass gap
\begin{equation}
\beta_D(\bq)
:=
\tau^{*}(\bq)-\tau_{\mathcal S_D}(\bq)
=
\delta_{\mathcal S_D}(\bq)-\delta^{*}(\bq)
\ge 0.
\tag{PB}
\end{equation}
Measure the surrogate misfit by
\begin{equation}
\varepsilon_D(\bq)
:=
\frac12
\left\|
A(\bq)-\widehat A_D(\bq)
\right\|_1
=
\left\|
A(\bq)-\widehat A_D(\bq)
\right\|_{\mathrm{TV}}.
\end{equation}

\begin{lemma}[Mis-scoring $\Rightarrow$ mass loss]
\label{lem:beta}
Let
$\mathcal S^{*}=\operatorname*{Top}_n(A)$ and
$\mathcal S_D=\operatorname*{Top}_n(\widehat A_D)$.
Then, for every realization $\bx$,
\begin{equation}
\beta_D(\bq)
\le
2\varepsilon_D(\bq),
\qquad
\delta_{\mathcal S_D}(\bq)
\le
\delta^{*}(\bq)+2\varepsilon_D(\bq).
\tag{MassLoss}
\end{equation}

\begin{proof}
For every $S\subseteq[L]$,
\[
|A(S)-\widehat A_D(S)|
\le
\varepsilon_D.
\]
Since
$\widehat A_D(\mathcal S_D)
\ge
\widehat A_D(\mathcal S^{*})$,
\begin{align*}
A(\mathcal S_D)
&\ge
\widehat A_D(\mathcal S_D)-\varepsilon_D
\ge
\widehat A_D(\mathcal S^{*})-\varepsilon_D\\
&\ge
A(\mathcal S^{*})-2\varepsilon_D.
\end{align*}
Thus
$\tau_{\mathcal S_D}
\ge
\tau^{*}-2\varepsilon_D$,
which proves the claim.
\end{proof}
\end{lemma}

Define
\begin{equation}
    \bar\beta_D
    :=
    \mathbb E_{\bX}
    \left[
    \beta_D(\bq)
    \right],
    \qquad
    \bar\varepsilon_D
    :=
    \mathbb E_{\bX}
    \left[
    \varepsilon_D(\bq)
    \right].
\label{eq:posthoc_average_errors}
\end{equation}

\begin{theorem}[Unified post-hoc routing-information bound]
\label{thm:posthoc}
Let
\begin{equation}
    I_{\mathrm{post}}
    :=
    I(\bX;T_{\mathcal S_D}),
    \qquad
    \Delta I_{\mathrm{post}}
    :=
    \left|
    I_{\mathrm{full}}-I_{\mathrm{post}}
    \right|.
\end{equation}
Then
\begin{equation}
    \bar\delta_{\mathrm{post}}
    =
    \bar\delta^{*}+\bar\beta_D
    \le
    \bar\delta^{*}+2\bar\varepsilon_D.
\label{eq:posthoc_mass_bound}
\end{equation}
Consequently,
\begin{equation}
    \Delta I_{\mathrm{post}}
    \le
    g(\bar\delta_{\mathrm{post}})
    \le
    g\!\left(
    \bar\delta^{*}+2\bar\varepsilon_D
    \right),
\tag{P1}
\end{equation}
provided the arguments remain in the increasing regime of $g$.

Moreover, if
\[
1-\delta^{*}(\bq)-2\varepsilon_D(\bq)>0
\]
almost surely, then
\begin{align}
D_{\mathrm{KL}}
\left(
P_{\bX,T_{\mathcal S_D}}
\middle\|
P_{\bX,T}
\right)
\nonumber\le
\mathbb E_{\bX}
\left[
\log
\frac{1}{
1-\delta^{*}(\bq)-2\varepsilon_D(\bq)}
\right].
\tag{P2}
\end{align}

\begin{proof}
Lemma~\ref{lem:beta} gives pointwise
\[
\delta_{\mathcal S_D}(\bq)
=
\delta^{*}(\bq)+\beta_D(\bq)
\le
\delta^{*}(\bq)+2\varepsilon_D(\bq).
\]
Taking expectations yields
Eq.~\eqref{eq:posthoc_mass_bound}. Applying
Proposition~\ref{prop:universal} gives (P1).

For (P2), use
\[
\tau_{\mathcal S_D}(\bq)
\ge
1-\delta^{*}(\bq)-2\varepsilon_D(\bq)
\]
in the KL identity (U2).
\end{proof}
\end{theorem}

\paragraph{Token-Dropping Oracle (TDO).}
TDO generally uses posterior statistics such as token age or
cumulative attention. Since these statistics need not accurately
represent the current query-conditioned attention distribution,
$\varepsilon_D(\bq)$ can increase when attention patterns shift,
which enlarges the dropped-mass certificate through the term
$2\bar\varepsilon_D$.

\paragraph{Query-Aware Approximation (QAA).}
QAA constructs a query-conditioned compressed representation
that approximates the oracle attention scores. Let
\begin{equation}
    a_i(\bq)
    =
    \frac{\bq^\top\bk_i}{\sqrt d},
\end{equation}
and suppose
\begin{equation}
\left\|
a(\bq)-\widehat a(\bq)
\right\|_\infty
\le
\eta_C(\bq).
\tag{LogitError}
\end{equation}

\begin{lemma}[Softmax Lipschitz from logits to probabilities]
\label{lem:softmaxLip}
For any $a,\widehat a\in\mathbb R^L$,
\begin{equation}
\left\|
\operatorname{softmax}(a)
-
\operatorname{softmax}(\widehat a)
\right\|_1
\le
2
\left\|
a-\widehat a
\right\|_\infty.
\end{equation}

\begin{proof}
Let
$s=\operatorname{softmax}(z)$ and
$J(z)=\operatorname{diag}(s)-ss^\top$.
For any $v\in\mathbb R^L$,
\begin{align*}
\|J(z)v\|_1
=
\sum_i
s_i
\left|
v_i-\sum_j s_jv_j
\right|\le
2\|v\|_\infty.
\end{align*}
The mean-value representation of softmax along the segment
joining $a$ and $\widehat a$ then gives the result.
\end{proof}
\end{lemma}

Therefore,
\begin{align}
\varepsilon_{\mathrm{QAA}}(\bq)
&=
\frac12
\left\|
\widehat A_{\mathrm{QAA}}(\bq)
-
A(\bq)
\right\|_1
\nonumber\\
&\le
\left\|
a(\bq)-\widehat a(\bq)
\right\|_\infty
\le
\eta_C(\bq).
\tag{QAA-eps}
\end{align}
Define
\begin{equation}
    \bar\eta_C
    :=
    \mathbb E_{\bX}
    \left[
    \eta_C(\bq)
    \right].
\end{equation}
Then Theorem~\ref{thm:posthoc} yields
\begin{equation}
    \Delta I_{\mathrm{QAA}}
    \le
    g\!\left(
    \bar\delta^{*}+2\bar\eta_C
    \right).
\tag{QAA}
\end{equation}
Thus, a smaller query-conditioned approximation error leads to a
tighter routing-information certificate. This statement concerns
the certificate and does not by itself imply a deterministic
ordering of downstream task accuracy.

\subsection{Pre-Hoc Sparsity}

\begin{definition}[Pre-hoc selector with controlled mass error]
\label{def:prehoc-error}
Consider a size-$n$ pre-hoc selector
$\mathcal S_{\mathrm{pre}}(\bK,\bq;\zeta)$,
parameterized by control information $\zeta$ available before full attention
scoring of the query being selected, with
$|\mathcal S_{\mathrm{pre}}(\bK,\bq;\zeta)|=n$.
For brevity, we write it as $\mathcal S_{\mathrm{pre}}(\bq)$ when
$\bK$ and $\zeta$ are clear from context.
The selector is said to have pointwise mass error
$\beta_{\mathrm{th}}(\bq)\ge0$ if
\begin{align}
&\tau_{\mathrm{pre}}(\bq)
:=
\sum_{i\in\mathcal S_{\mathrm{pre}}(\bq)}
A_i(\bq)
\ge
\tau^{*}(\bq)-\beta_{\mathrm{th}}(\bq),
\nonumber\\
&\text{equivalently}\qquad
\delta_{\mathrm{pre}}(\bq)
\le
\delta^{*}(\bq)+\beta_{\mathrm{th}}(\bq).
\label{eq:prehoc-mass-error}
\end{align}
Here $\tau^{*}(\bq)$ and $\delta^{*}(\bq)$ correspond to the
budget-matched top-$n$ oracle. We consider either a pointwise control $\beta_{\mathrm{th}}(\bq){\le}\varepsilon_{\mathrm{th}}\;\text{a.s.}$, or an average control $\bar{\beta}_{\mathrm{th}}:=\mathbb E_{\bX}\left[\beta_{\mathrm{th}}(\bq)\right]\le\overline{\varepsilon}_{\mathrm{th}}$, where $\bX$ denotes the random query-context state.
\end{definition}

\begin{theorem}[Information bound for pre-hoc sparsity]
\label{thm:prehoc-error}
Let
\begin{equation}
    I_{\mathrm{pre}}
    :=
    I(\bX;T_{\mathcal S_{\mathrm{pre}}}),
    \qquad
    \Delta I_{\mathrm{pre}}
    :=
    \left|
    I_{\mathrm{full}}-I_{\mathrm{pre}}
    \right|.
\end{equation}
Then
\begin{equation}
    \bar\delta_{\mathrm{pre}}
    \le
    \bar\delta^{*}+\bar\beta_{\mathrm{th}},
\end{equation}
and therefore
\begin{equation}
    \Delta I_{\mathrm{pre}}
    \le
    g(\bar\delta_{\mathrm{pre}})
    \le
    g\!\left(
    \bar\delta^{*}+\bar\beta_{\mathrm{th}}
    \right).
\label{eq:prehoc-pointwise}
\end{equation}

Under the pointwise control
$\beta_{\mathrm{th}}(\bq)\le\varepsilon_{\mathrm{th}}$,
\begin{equation}
    \Delta I_{\mathrm{pre}}
    \le
    g\!\left(
    \bar\delta^{*}+\varepsilon_{\mathrm{th}}
    \right).
\label{eq:prehoc-uniform}
\end{equation}

Under the average control
$\bar\beta_{\mathrm{th}}
\le
\overline\varepsilon_{\mathrm{th}}$,
\begin{equation}
    \Delta I_{\mathrm{pre}}
    \le
    g\!\left(
    \bar\delta^{*}
    +
    \overline\varepsilon_{\mathrm{th}}
    \right).
\label{eq:prehoc-expected}
\end{equation}
These substitutions use the increasing regime of $g$.

Moreover, if
\[
1-\delta^{*}(\bq)-\beta_{\mathrm{th}}(\bq)>0
\]
almost surely, then
\begin{equation}
\label{eq:prehoc-KL}
D_{\mathrm{KL}}
\left(
P_{\bX,T_{\mathcal S_{\mathrm{pre}}}}
\middle\|
P_{\bX,T}
\right)
\nonumber\le
\mathbb E_{\bX}
\left[
\log
\frac{1}{
1-\delta^{*}(\bq)-\beta_{\mathrm{th}}(\bq)}
\right].
\end{equation}
\end{theorem}

\begin{proof}
By Definition~\ref{def:prehoc-error},
\begin{equation}
\delta_{\mathrm{pre}}(\bq)
\le
\delta^{*}(\bq)+\beta_{\mathrm{th}}(\bq).
\end{equation}
Taking expectations over $\bX$ gives
\begin{equation}
\bar\delta_{\mathrm{pre}}
\le
\bar\delta^{*}+\bar\beta_{\mathrm{th}}.
\end{equation}
Applying Proposition~\ref{prop:universal} yields
Eq.~\eqref{eq:prehoc-pointwise}.

If
$\beta_{\mathrm{th}}(\bq)
\le
\varepsilon_{\mathrm{th}}$
almost surely, then
$\bar\beta_{\mathrm{th}}
\le
\varepsilon_{\mathrm{th}}$,
which gives Eq.~\eqref{eq:prehoc-uniform}.
The average-control result
Eq.~\eqref{eq:prehoc-expected}
follows directly from
$\bar\beta_{\mathrm{th}}
\le
\overline\varepsilon_{\mathrm{th}}$.

Finally,
\[
\tau_{\mathrm{pre}}(\bq)
\ge
1-\delta^{*}(\bq)-\beta_{\mathrm{th}}(\bq),
\]
and substituting this inequality into (U2) gives
Eq.~\eqref{eq:prehoc-KL}.
\end{proof}

From the perspective of certificate expansion, consider
\[
g(\delta)
=
2\left[
h_{\mathrm b}(\delta)+\delta\log L
\right].
\]
For fixed $\bar\delta^{*}$ and sufficiently small
$\bar\beta_{\mathrm{th}}$,
\begin{align}
g\!\left(
\bar\delta^{*}+\bar\beta_{\mathrm{th}}
\right)
&=
g(\bar\delta^{*})
\nonumber\\
&\quad+
2
\left(
\log
\frac{1-\bar\delta^{*}}{\bar\delta^{*}}
+
\log L
\right)
\bar\beta_{\mathrm{th}}
+
o(\bar\beta_{\mathrm{th}}).
\label{eq:prehoc_taylor}
\end{align}
Hence, near the oracle operating point, the additional
routing-information certificate induced by the theoretical mass
error grows approximately linearly with
$\bar\beta_{\mathrm{th}}$.

Comparing Theorems~\ref{thm:posthoc} and
\ref{thm:prehoc-error}, if $\bar\beta_{\mathrm{th}} < 2\bar\varepsilon_D$,
then the pre-hoc selector admits a smaller dropped-mass upper bound
and therefore a tighter routing-information certificate, under the
same oracle term $\bar\delta^{*}$ and within the increasing regime
of $g$. This is a certificate-level advantage; its relation to
downstream quality is established empirically in
Sec.~\ref{sec:theory_practice_grounding}.



\bibliography{references}
\bibliographystyle{IEEEtran}
\begin{IEEEbiography}[{\includegraphics[width=1in,height=1.25in,clip,keepaspectratio]{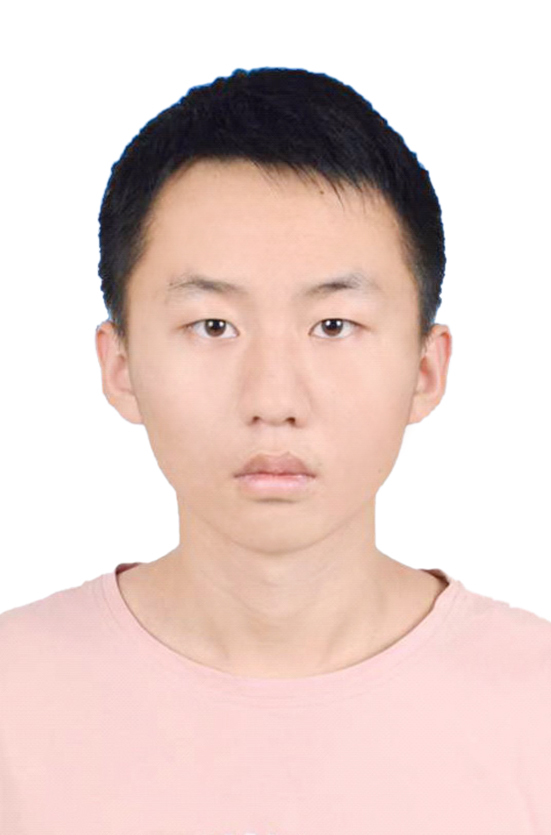}}]{Yifei Gao} is pursuing the master degree at University College London. He previously obtained his B.Eng. degree from the College of Software Engineering at the University of Electronic Science and Technology of China in 2024. His current research interests encompass large language models, 3D reconstruction, and diffusion models.
\end{IEEEbiography}

\vspace{8pt}

\vspace{-33pt}

\begin{IEEEbiography}[{\includegraphics[width=1in,height=1.25in,clip,keepaspectratio]{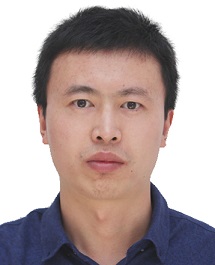}}]{Lei Wang} (Senior Member, IEEE) received the Ph.D. degree in Electrical Engineering from Xidian University, Xi'an, China, in 2010. From 2011 to 2012, he was with Huawei Technologies Company Ltd. From 2014 to 2015, he was with the Department of Embedded Systems Engineering, Incheon National University, as a Postdoctoral Fellow. He is currently an Associate Professor with the Shenzhen Institutes of Advanced Technology, Chinese Academy of Sciences, Shenzhen, China. He has authored or coauthored more than 70 papers in conferences and journals, including IEEE TRANSACTIONS ON CIRCUITS AND SYSTEMS FOR VIDEO TECHNOLOGY, IEEE TRANSACTIONS ON CYBERNETICS, IEEE TRANSACTIONS ON MULTIMEDIA, IEEE SIGNAL PROCESSING LETTERS, and NAACL. His research interests include image processing, transforms, machine learning, computer vision, 3-D reconstruction, and robotics.
\end{IEEEbiography}

\vspace{8pt}

\vspace{-33pt}

\begin{IEEEbiography}[{\includegraphics[width=1in,height=1.25in,clip,keepaspectratio]{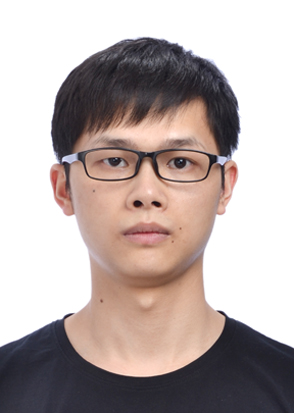}}]{Rong-Cheng Tu}
received the B.S. degree in 2018 and the Ph.D. degree in 2023 from the Beijing Institute of Technology, Beijing, China. 
He is currently a Research Fellow with Nanyang Technological University (NTU), Singapore. 
His research interests include deep learning, information retrieval, and Large Language Model efficiency.
\end{IEEEbiography}

\vspace{8pt}

\vspace{-33pt}

\begin{IEEEbiography}[{\includegraphics[width=1in,height=1.25in,clip,keepaspectratio]{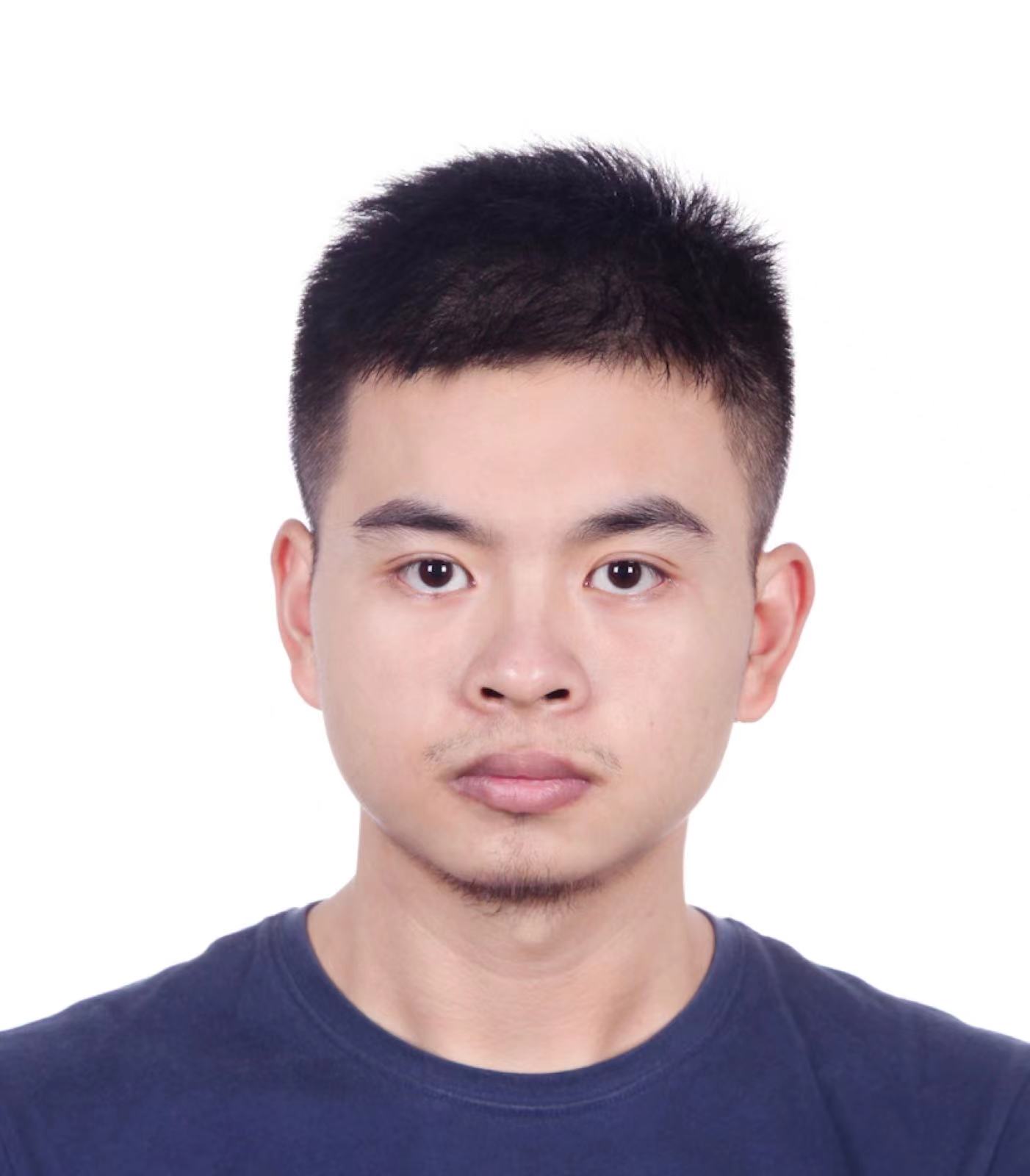}}]{Qixin Zhang}
received his B.S. degree from the University of Science and Technology of China in 2018. He subsequently earned his Ph.D. degree in the School of Data Science at City University of Hong Kong in 2024. Currently, he is a Research Fellow at Nanyang Technological University, Singapore. His research interests include submodular optimization, subset selection and network prune. He has published over 15 papers in top data science venues such as ICML, NeurIPS, ICLR, AISTATS, AAAI and TKDE, etc.
\end{IEEEbiography}

\vspace{8pt}

\vspace{-33pt}

\begin{IEEEbiography}[{\includegraphics[width=1in,height=1.25in,clip,keepaspectratio]{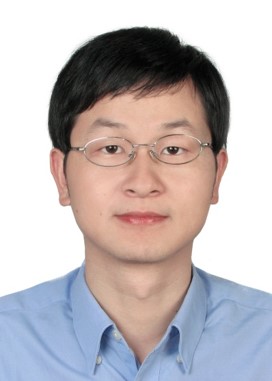}}]{Jun Cheng} (Senior Member, IEEE) received the B.Eng. and M.Eng. degrees from the University of Science and Technology of China, Hefei, China, in 1999 and 2002, respectively, and the Ph.D. degree from Chinese University of Hong Kong, Hong Kong, in 2006. He is currently a Professor with Shenzhen Institutes of Advanced Technology, Chinese Academy of Sciences, Shenzhen, China, where he is also the Director of the Laboratory for Human Machine Control. His current research interests include computer vision, robotics, machine intelligence, and control.
\end{IEEEbiography}

\vspace{8pt}

\vspace{-33pt}

\begin{IEEEbiography}[{\includegraphics[width=1in,height=1.25in,clip,keepaspectratio]{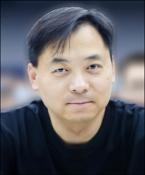}}]{Dacheng Tao} (F’15) is currently a Distinguished University Professor in the College of Computing \& Data Science at Nanyang Technological University. He mainly applies statistics and mathematics to artificial intelligence and data science, and his research is detailed in one monograph and over 200 publications in prestigious journals and proceedings at leading conferences, with best paper awards, best student paper awards, and test-of-time awards. His publications have been cited over 126K times and he has an hindex 160+ in Google Scholar. He received the 2015 and 2020 Australian Eureka Prize, the 2018 IEEE ICDM Research Contributions Award, and the 2021 IEEE Computer Society McCluskey Technical Achievement Award. He is a Fellow of the Australian Academy of Science, AAAS, ACM and IEEE.
\end{IEEEbiography}

\clearpage
\appendix

\subsection{Mean-shift Property of Token-sparse Attention}
\label{sec:mean-shift}

In this section, building upon prior research Deja Vu~\cite{liu2023deja}, we first describe the analytical foundation of our work: self-attention can be interpreted as a form of mean-shift clustering. Based on this perspective, we then explore the among-query properties of self-attention and leveraging the formal definition of TSA (as introduced in \textbf{Definition 3.1}) to construct and complete our theoretical proof.

\subsubsection{Self-Attention as Mean-Shift Clustering}

\paragraph{Implicit Clustering in Attention} The softmax weighting means that the query $\bq$ mainly focuses on a subset of keys with the largest dot-products $\bq \cdot \bk_j$. This behavior causes tokens to form clusters: tokens (keys) that are similar to a given query will receive high attention weights (“cluster members”), whereas unrelated tokens get negligible weight. In other words, each query \textbf{picks out clusters} of relevant key tokens and largely ignores the rest, which is exactly the token sparsity phenomenon~\cite{zhang2023h2o,tang2406quest} (only a few tokens significantly contribute to the attention result).



\paragraph{Mean-Shift Interpretation} 

Importantly, the Deja Vu study observed that \textbf{self-attention implements a one-step mean-shift clustering update}. Formally, fixing within one head, for the attention output $\by {=} \sum_{i} A_i (\bx_i \bW_V)$, if we ignore the $\bW_V$ projection for a moment, this computes the \emph{weighted centroid} of the token embeddings $\bx$ around the query $\bq$ (in the projected space). That is, define an \emph{attention mean} $m(\bq)$ as the weighted mean of the embeddings:
\begin{equation}
    m(q) \;=\; \frac{\sum_i \bK(\bx_i, \bq)\, \bx_i}{\sum_i \bK(\bx_i, \bq)}\,,
\end{equation}
so $\by {=} m(\bq) \bW_V$.  This is exactly the update step of the \textbf{mean-shift clustering algorithm}: $m(\bq)$ is pulling the query’s embedding toward a nearby cluster of points ${\bx_i}$ that are similar to $\bq$. In mean-shift, one would iteratively set $\bq \leftarrow m(\bq)$ to converge to the cluster center (mode of the density). In a Transformer, we usually apply just one such step per layer, but the analogy holds: \emph{each attention head nudges the query vector toward the centroid of a cluster of key vectors}.

\subsubsection{Among-Query Property}

\paragraph{Problem Setting} Consider two queries $\bq_t$ and $\bq_{t+1}$ from adjacent time steps in a sequence. Thanks to residual connections and the smooth nature of transformer dynamics, adjacent queries in a sequence tend to be close in embedding space~\cite{sun2024shadowkv,wuhshare}. We then assume as a general principle that \emph{$\bq_{t+1}$ is a small perturbation of $\bq_t$} (i.e. $\bq_{t+1} \approx \bq_t$ in $\ell_2$ norm, or in cosine similarity together with stable query norms), as the among-query property only concerns the behavior of attention when queries are similar, we revise the \textbf{Among-query Property}~\ref{prop:among-query} that:

\begin{property}[Among-query Property]
\label{prop:among-query_new}
If two query vectors $\bq_t$ and $\bq_{t+1}$ are aligned enough in vector space, then the distributions of their attention weights are also very similar.
\end{property}

In other words, adjacent queries attend to largely the same cluster of key tokens. As $\bq_t$ shifts to $\bq_{t+1}$, the attention weight mass simply shifts smoothly among those keys–obeying a continuous mean-shift-like change–rather than jumping to a completely different set of keys.

\paragraph{Intuition via Mean-Shift} Think of the keys ${\bk_i}$ as data points in the embedding space. The query $\bq_t$ will put most weight on keys in some neighborhood (a “cluster”) around $\bq_t$. Now if $\bq_{t+1}$ moves slightly (say $\bq_{t+1} = \bq_t + \Delta$ for a small $\Delta$), it will still lie in roughly the same neighborhood of points. The attention kernel $\bK(x_i,\bq) = \exp(\bq\cdot \bk_i)$ is a smooth, continuous function of $q$. A slight change in $\bq$ causes only a slight change in all the similarity scores $\bq\cdot \bk_i$. Thus the weights $A_i$ shift slightly: keys that are highly ranked for $\bq_t$ will remain highly ranked for $\bq_{t+1}$, as long as $\bq_{t+1}$ hasn’t moved so far that it finds a totally different cluster.

\paragraph{Formal Derivation of the Among-Query Property}
\label{sec:among-query-def}
Let $\bK{=}\{\bk_1, \bk_2, \dots, \bk_t\}$ be the key vectors available. Consider two queries $\bq$ and $\bq'$ (think of $\bq = \bq_t$ and $\bq' = \bq_{t+1}$) such that $\bq'$ is a small perturbation of $\bq$. We will show that the attention distributions of $A(\bq)$ and $A(\bq')$ are \emph{close}, and in particular that the high-weight components of $A(\bq)$ correspond to high-weight components of $A(\bq')$. We also look at the weighted average outputs $\by(\bq)$ and $y(\bq')$. We first consider the perturbation changes in queries (e.g., $\bq$ to $\bq'$) to that of logits. For each key $\bk_i$, the score (logit) for $q$ as $a_i = \frac{1}{\sqrt{d}} \, \bq\cdot \bk_i$ and for $\bq'$ as $a'_i = \frac{1}{\sqrt{d}} \, \bq'\cdot \bk_i$. Because $\bq' = \bq + \Delta$ for a small $\Delta$, we have
$a'_i {-} a_i {=} \frac{1}{\sqrt{d}}(\bq' {-} \bq){\cdot} \bk_i {=} \frac{1}{\sqrt{d}} \Delta {\cdot} \bk_i$.

Let $\Delta_i := a'_i - a_i$. Since $| \Delta |$ is small and (typically) the keys ${\bk_i}$ have bounded norm, each $\Delta_i$ is a small real number. In other words, the entire vector of logits $a' = (a'_1,\dots,a'_t)$ is a small perturbation of $a = (a_1,\dots,a_t)$, and there exists some $\epsilon \ll 1$ such that $|\Delta_i| \le \epsilon$ for all $i$. ($\Delta_i$ might vary across $i$, but we can bound the maximum change by $\epsilon {=} \frac{|\Delta| \max_i|k_i|}{\sqrt{d}}$.) Then, combining the softmax part in attention weights, by definition, $A_i(\bq) = \frac{e^{a_i}}{\sum_j e^{a_j}}$ and $A_i(\bq') = \frac{e^{a'_i}}{\sum_{j} e^{a'_j}}$. It will be convenient to express $A_i(\bq')$ in terms of $A(\bq)$ and the perturbations $\Delta_i$. We can write:
\begin{equation}
   A_i(\bq') = \frac{e^{a_i + \Delta_i}}{\sum_j e^{a_j + \Delta_j}} = \frac{A_i(\bq)\,e^{\Delta_i}}{\sum_j A_j(\bq)\,e^{\Delta_j}}\,. 
\end{equation}
Here we use $e^{a_i}/\sum_j e^{a_j} = A_i(\bq)$ and factored $e^{\Delta_j}$ out of the denominator. Let $A^{n} = \sum_j A_j(\bq)\, e^{\Delta_j}$ denote the normalization factor that adjusts the weights due to the query change. Since each $\Delta_j$ is small, we can expand $A^{n}$ in a Taylor series:
\begin{equation}
\label{eq:d_taylor_exp}
\begin{aligned}
      A^{n} {=} \sum_j A_j(\bq)(1 {+} \Delta_j {+} O(\Delta_j^2)){=} 1 {+} \sum_j A_j(\bq)\Delta_j {+} O(\epsilon^2).  
\end{aligned}
\end{equation}
Here $\sum_j A_j(\bq)\Delta_j$ is the expectation of $\Delta$ under the original distribution $A(\bq)$. Let $\mu_\Delta {=} \sum_j A_j(\bq)\Delta_j$ for brevity. Combining Eq.~\eqref{eq:d_taylor_exp}, we can write $A^{n} {=} 1 {+} \mu_\Delta {+} O(\epsilon^2)$. Using this, we can approximate the new weights $A_i(\bq')$ based on the First-Order Taylor Approximation in small $\Delta$,
\begin{equation}
\begin{aligned}
\label{eq:d_first_order_taylor}
    A_i(\bq') \approx \frac{A_i(\bq)(1 {+} \Delta_i)}{1 + \mu_\Delta} \approx A_i(\bq)(1 {+} \Delta_i {-} \mu_\Delta),
\end{aligned}
\end{equation}
where second approximation used $\frac{1}{1+\mu_\Delta} \approx 1 - \mu_\Delta$ for small $\mu_\Delta$. Thus, for small query shifts, the new attention weight is approximately $A_i(\bq') {\approx} A_i(\bq) {+} A_i(\bq)(\Delta_i {-} \mu_\Delta)$. This formula reveals how the attention distribution shifts in response to $\Delta$:
\begin{itemize}[leftmargin=3mm]
    \item If $\Delta_i$ is larger than the average $\mu_\Delta$, then $A_i$ will increase ($\Delta_i - \mu_\Delta > 0$).
    \item If $\Delta_i$ is smaller than the average, $A_i$ will decrease slightly.
    \item The total change is zero-sum to first order as shown in Eq.~\eqref{eq:zero_sum}, as expected (the weights still sum to 1).
\end{itemize}
\begin{align}
\label{eq:zero_sum}
    &\sum_i [A_i(\bq') - A_i(\bq)] \approx \sum_i A_i(\bq)(\Delta_i - \mu_\Delta) \nonumber \\ 
    = &(\sum_i A_i(\bq)\Delta_i) - \mu_\Delta\sum_i A_i(\bq) =
    \mu_\Delta - \mu_\Delta(1) = 0
\end{align}

\paragraph{Transmission to Token-sparse Attention}

After proving the minimal attention score shifts in self-attention, we integrate them to the token-sparse attention defined in Sec.~\ref{sec:prob_def}. Suppose the original query $\bq$ had a set of critical keys (a cluster) $C = \{i : A_i(\bq)\}$ is among the top weights and above some importance threshold. We want to show that for $\bq'$, the high-weight keys $C' = \{i : A_i(\bq')\}$ largely overlap with $C$. From the approximation above, if $i {\in} C$ had a large original weight $A_i(\bq)$, then $A_i(\bq')$ will remain large unless something drastic happens to $\Delta_i$, as $\bq'{\cdot} \bk_i$ will remain relatively larger for those same $i{\in}C$ than for others.

For any $i$ in the original top cluster $C$, $a_i$ was high relative to most other $a_j$. The change $\Delta_i$ would have to be exceptionally negative (compared to others) to knock $i$ out of the top ranks. But $\Delta_i = \frac{1}{\sqrt{d}}\Delta {\cdot} \bk_i$ is small for all $i$, so the attention distribution changes only slightly. Exact top-$k$ membership is stable when the relevant logit margins exceed the perturbation; without such a margin, we only rely on the distributional Lipschitz bound derived below.

For any $i$ outside the cluster ($i {\notin} C$ with tiny $A_i(\bq)$), a small $\Delta$ cannot create a large probability change unless the competing logits are nearly tied. Thus, when the relevant margins are non-degenerate, the high-mass sets tend to overlap. \textbf{Overall, $\bq$ and $\bq'$ induce close attention distributions.} The rigorous statement used below is the $\ell_1$ Lipschitz bound on $A(\bq)$ and $A(\bq')$, rather than a deterministic guarantee on discrete top-$k$ membership.

\paragraph{Smooth Shift of the Attention Centroid} 

Because the weight distribution shifts only slightly, the centroid (weighted average) of the attended keys also shifts smoothly. The output for query $q'$ is thus
\begin{equation}
\begin{aligned}
    \by(\,q') &= \sum_i [A_i(\bq) + A_i(\bq)(\Delta_i - \mu_\Delta)]\, \bk_i \\
    &\approx \sum_i A_i(\bq) \bk_i + \sum_i A_i(\bq)(\Delta_i - \mu_\Delta)\,\bk_i.
\end{aligned}
\end{equation}
The second term is a small adjustment. In fact, note that $\sum_i A_i(\bq)\Delta_i \bk_i {=} (\Delta {\cdot} \bk_i)$ weighted by $A_i$, this is roughly $\Delta$ projected onto the direction of the original output. Without going too deep into the algebra, the main result indicates that $\by(\bq') {-} \by(\bq)$ is of order $O(|\Delta|)$, a small change. Geometrically, if $\by(\bq)$ is (proportional to) the mean of cluster $C$, then $\by(\bq')$ will be the mean of almost the same cluster, just adjusted towards $\bq'$. \textbf{This is precisely what we expect from a mean-shift update: as the query moves, the cluster centroid moves along with it, continuously.}



\subsubsection{Centroid Shift}
\label{sec:centroid-shift}

We formalize the Lipschitz continuity of the attention centroid with respect to the query. Follow the definition in Sec.~\ref{sec:among-query-def}, let $\{p_i\}_{i=1}^t\subset\mathbb{R}$ be scalar token positions (sequence coordinates), and define the sequence-axis centroid $c(\bq){=} \sum_{i=1}^t A_i(\bq)p_i$. No normalization of the attention queries is assumed. We write $K_{\max}\!\stackrel{\mathrm{def}}{=}\!\max_i\|\bk_i\|$ and $\mathrm{diam}\,\mathcal{P}{=}\max_i p_i-\min_i p_i$; when cosine similarity is used below, query-norm variation is controlled separately.

\begin{lemma}[Softmax $\ell_\infty\!\to\!\ell_1$ Lipschitzness]
\label{lem:softmax-lip}
For any $a,a'\in\mathbb{R}^t$,
\[
\big \|\mathrm{softmax}(a')-\mathrm{softmax}(a)\big\|_1
\;\le\; 2\,\|a'-a\|_\infty.
\]
\end{lemma}

\begin{proof}
Along the line segment between $a$ and $a'$,
\[
\mathrm{softmax}(a')-\mathrm{softmax}(a)
=
\int_0^1 J\!\big(a+s(a'-a)\big)(a'-a)\,ds,
\]
where $J(z){=}\mathrm{diag}(A){-}AA^\top$ for $A{=}\mathrm{softmax}(z)$.
For any $v$, $\|J(z)v\|_1\le 2\|v\|_\infty$; integrating along the
segment gives
$\|\mathrm{softmax}(a') {-} \mathrm{softmax}(a)\|_1 {\le} 2 \|a'-a\|_\infty$.
\end{proof}

\begin{lemma}[Logit change under query perturbation]
\label{lem:logit-perturb}
Let $\bq'=\bq+\Delta$. Then
\[
\|a(\bq')-a(\bq)\|_\infty
\;=\;
\max_{i}\frac{1}{\sqrt d}\,|\,\Delta^\top \bk_i\,|
\;\le\; \frac{K_{\max}}{\sqrt d}\,\|\Delta\|.
\]
\end{lemma}

\begin{proof}
Immediate from Cauchy–Schwarz: $|\Delta^\top \bk_i| \!\le\! \|\Delta\|\|\bk_i\|$.
\end{proof}

\begin{theorem}[Query-Lipschitz Centroid Drift]
\label{thm:centroid-lip}
For $\bq'{=}\bq{+}\Delta$,
\begin{equation}
\begin{aligned}
\big|c(\bq'){-}c(\bq)\big|
&\;\le\;
(\mathrm{diam}\,\mathcal{P})\cdot
\big\|A(\bq')-A(\bq)\big\|_1 \\
&\;\le\;
2\,(\mathrm{diam}\,\mathcal{P})\cdot \frac{K_{\max}}{\sqrt d}\,\|\Delta\|.
\end{aligned}
\end{equation}
In particular, $c(\bq') {=} c(\bq) {+} O(\|\Delta\|).$
\end{theorem}

\begin{proof}
By definition,
$|c(\bq')\!-\!c(\bq)|\le \|A(\bq')\!-\!A(\bq)\|_1 \cdot \max_i |p_i\!-\!\bar p|$ for any $\bar p$.
Choosing $\bar p\in[\min_i p_i,\max_i p_i]$ gives $\max_i|p_i-\bar p|\le \mathrm{diam}\,\mathcal{P}$, hence the first inequality. The second inequality follows by Lemmas~\ref{lem:softmax-lip} and \ref{lem:logit-perturb}.
\end{proof}

\subsubsection{CIS Routing-MI Certificate Parameterization}
\label{sec:cis-mi-parameterization}

We characterize the CIS certificate for a reused query using parameters
$(\tau,m,r)$, explicit query-norm stability, and the already-scored anchor query.
The resulting $\beta_{\mathrm{th}}$ is an anchor-conditioned pre--hoc retained-mass
gap: it is available before full scoring of the reused query, although its
reference term is not a pure design-time constant. After averaging over
query--context states, it controls the routing-MI certificate in
Eq.~\eqref{eq:mi_bound}.

\paragraph{Notation}
Fix a head and a decoding step $t$. Let $S_t^*(\bq)$ be the middle-region
top-$k$ set defined in Sec.~\ref{sec:cis}, and let
$S_{t,m}(\bq)\subset S_t^*(\bq)$ be its top-$m$ winners. For $r\ge0$, define
$\mathcal{N}_r(S)=\bigcup_{p\in S}\{p+\delta:\delta\in\{-r,\dots,r\}\}$
with clipping to $[1,t]$, so
$\hat S_t(\bq)=S_t^*(\bq)\cup\mathcal{N}_r(S_{t,m}(\bq))$ as in
Eq.~\eqref{eq:neighboring}. Including the always-retained sink/local indices,
let $S_t^{\mathrm{CIS}}(\bq)$ denote the actual CIS set and
$b_t:=|S_t^{\mathrm{CIS}}(\bq)|$. Its budget-matched oracle is
\[
S_{b_t}^*(\bq)=\operatorname{Top}_{b_t}\!\big(A(\bq)\big),
\qquad
\tau_{b_t}^*(\bq):=\tau_{S_{b_t}^*}(\bq).
\]
For a query $\bq'$ that reuses this set, define
\[
\tau_{\mathrm{pre}}(\bq')
:=
\tau_{S_t^{\mathrm{CIS}}(\bq)}(\bq'),
\;
\beta_{\mathrm{ref}}(\bq)
:=
\tau_{b_t}^*(\bq)-\tau_{S_t^{\mathrm{CIS}}(\bq)}(\bq)
\ge0.
\]
Because $\bq$ is the already-scored anchor, $\beta_{\mathrm{ref}}(\bq)$ is
known before full scoring of the reused query $\bq'$. Let
\begin{equation}
\begin{aligned}
\Delta_{\mathrm{att}}(\bq,\bq')
&:=
\|A(\bq')-A(\bq)\|_1, \\
\Delta_{\mathrm{att}}(\tau,\varepsilon_{\mathrm{norm}},Q_{\max})
&:=
\frac{2K_{\max}}{\sqrt d}
\left(
\varepsilon_{\mathrm{norm}}
+Q_{\max}\sqrt{2-2\tau}
\right). 
\end{aligned}
\end{equation}

\begin{lemma}[Similarity and norm stability ${\Rightarrow}$ attention variation]
\label{lem:sim-to-var}
Let $\bq,\bq'\neq0$, $r:=\|\bq\|$, $r':=\|\bq'\|$,
$\mathbf{u}:=\bq/r$, and $\mathbf{u}':=\bq'/r'$. If
\[
|r-r'|\le\varepsilon_{\mathrm{norm}},
\quad
r,r'\le Q_{\max},
\quad
\mathsf{sim}(\bq,\bq')=\mathbf{u}^{\!\top}\mathbf{u}'\ge\tau,
\]
then
\[
\Delta_{\mathrm{att}}(\bq,\bq')
\le
\frac{2K_{\max}}{\sqrt d}\,\|\bq'-\bq\|
\le
\Delta_{\mathrm{att}}(\tau,\varepsilon_{\mathrm{norm}},Q_{\max}).
\]
\end{lemma}
\begin{proof}
By the triangle inequality, $\|\bq'-\bq\|
=\|r'\mathbf{u}'-r\mathbf{u}\|\le |r'-r|+r\|\mathbf{u}'-\mathbf{u}\|
\le \varepsilon_{\mathrm{norm}} +Q_{\max}\sqrt{2-2\tau}$.

Combining this bound with Lemmas~\ref{lem:softmax-lip}--\ref{lem:logit-perturb}
gives the result.
\end{proof}

\begin{lemma}[Budget-matched oracle continuity]
\label{lem:oracle-mass}
For the fixed budget $b_t$,
\[
\big|\tau_{b_t}^*(\bq')-\tau_{b_t}^*(\bq)\big|
\le \Delta_{\mathrm{att}}(\bq,\bq').
\]
\end{lemma}
\begin{proof}
For any fixed $S$,
$\big|\sum_{i\in S}(A_i(\bq')-A_i(\bq))\big|
\le\|A(\bq')-A(\bq)\|_1$.
Taking $S=S_{b_t}^*(\bq')$ gives
$\tau_{b_t}^*(\bq')\le\tau_{b_t}^*(\bq)+\Delta_{\mathrm{att}}(\bq,\bq')$;
the reverse direction follows symmetrically.
\end{proof}

\begin{lemma}[Reuse lower bound]
\label{lem:reuse-lb}
\[
\tau_{\mathrm{pre}}(\bq')
\ge
\tau_{S_t^{\mathrm{CIS}}(\bq)}(\bq)
-\Delta_{\mathrm{att}}(\bq,\bq').
\]
\end{lemma}
\begin{proof}
For the fixed reused set $S_t^{\mathrm{CIS}}(\bq)$, $\tau_{\mathrm{pre}}(\bq')
-\tau_{S_t^{\mathrm{CIS}}(\bq)}(\bq)
=\sum_{i\in S_t^{\mathrm{CIS}}(\bq)}
\big(A_i(\bq')-A_i(\bq)\big)
\ge-\Delta_{\mathrm{att}}(\bq,\bq')$.

\end{proof}

\begin{figure}[t]
    \centering
    \includegraphics[width=\columnwidth]{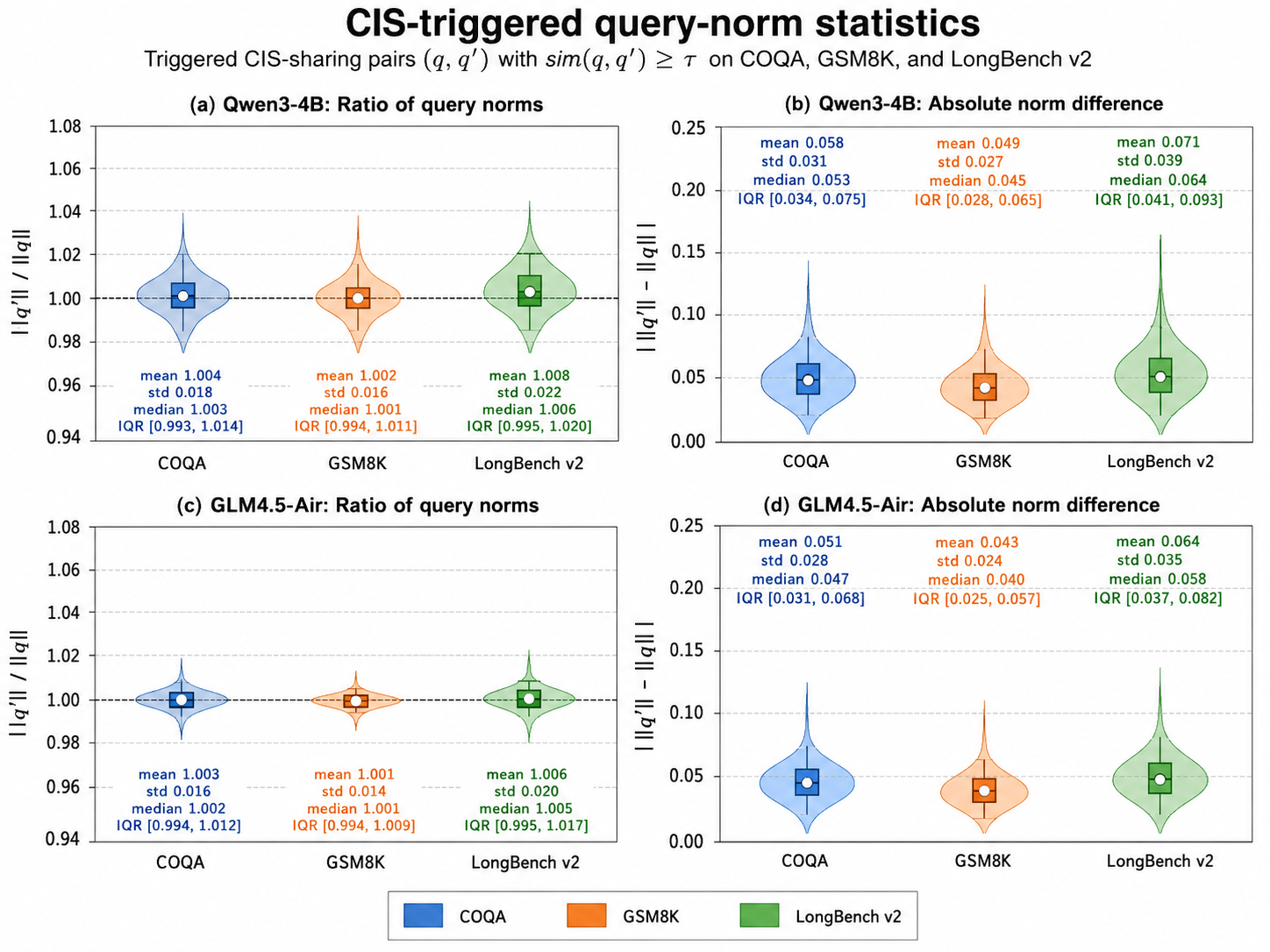}
    \vspace{-1.5em}
    \caption{
    Query-norm statistics of CIS-triggered sharing pairs on Qwen3-4B and GLM4.5-Air. Statistics are computed only over query pairs $(\bq,\bq')$ for which
    CIS sharing is actually triggered, i.e., $\mathsf{sim}(\bq,\bq')\ge\tau$.
    Left: the query-norm ratio $\|\bq'\|/\|\bq\|$, where the dashed line indicates the norm-preserving reference value $1$. Right: the absolute query-norm difference
    $|\|\bq'\|-\|\bq\||$. Results are reported on COQA, GSM8K, and LongBench v2.
    }
    \label{fig:query-norm-stability}
    \vspace{-2mm}
\end{figure}

\paragraph{Empirical Query-Norm Stability}

The CIS attention-variation bound requires the norms of a reference query $\bq$ and its reused query $\bq'$ to remain stable. We therefore examine this condition directly over the query pairs for which
CIS sharing is actually triggered. As shown in Fig.~\ref{fig:query-norm-stability}, the norm ratio $\|\bq'\|/\|\bq\|$ is tightly concentrated around $1$ across all three datasets and both model families. For Qwen3-4B, the mean ratios range from $1.002$ to $1.008$, with standard deviations of only $0.016$--$0.022$; for GLM4.5-Air, the corresponding means are $1.001$--$1.006$ with standard deviations of $0.014$--$0.020$. Consistently, the absolute norm drift $|\|\bq'\|-\|\bq\||$ remains small: its mean is $0.049$--$0.071$ for Qwen3-4B and $0.043$--$0.064$ for GLM4.5-Air, while all medians remain below $0.065$. Even on LongBench v2, which exhibits the broadest distribution, the upper quartiles are only $0.093$ and $0.082$ for Qwen3-4B and GLM4.5-Air, respectively. These observations indicate that, among the query pairs admitted by the CIS similarity gate, high directional similarity is accompanied in practice by limited norm variation. Hence, the additional $\varepsilon_{\mathrm{norm}}$ term in $\Delta_{\mathrm{att}}(\tau,\varepsilon_{\mathrm{norm}},Q_{\max})$ accounts for a small correction beyond the cosine-induced directional term, empirically supporting the norm-stability condition used in Theorem~\ref{thm:cis-main}. We emphasize that these distributional statistics support the assumption over the realized CIS-sharing pairs rather than establishing a uniform pointwise bound for every possible query pair.

\paragraph{Role of centroid drift and dilation}
Theorem~\ref{thm:centroid-lip} bounds the \emph{global} attention-centroid
shift, but does not by itself guarantee that every critical cluster remains
inside an $r$-neighborhood. We therefore do not require a cluster-coverage
assumption for the certificate below. Instead, $(m,r)$ determine the actual
set $S_t^{\mathrm{CIS}}$, while the already-scored anchor attention determines
its budget-matched reference gap $\beta_{\mathrm{ref}}$. Thus
$\beta_{\mathrm{ref}}$ is historical information available before scoring the
reused query, rather than a pure design-time constant; it is the same
retained-mass quantity measured empirically in
Eq.~\eqref{eq:empirical_mass_gap}.

\begin{proposition}[Anchor-conditioned pre--hoc retained-mass gap under CIS]
\label{prop:cis-beta}
For any reused query $\bq'$ satisfying the similarity and query-norm stability
conditions of Lemma~\ref{lem:sim-to-var}, define the budget-matched pointwise error
\[
\beta_{\mathrm{th}}^{\mathrm{CIS}}(\bq')
:=
\tau_{b_t}^*(\bq')-\tau_{\mathrm{pre}}(\bq')
\ge0.
\]
Then
\begin{equation}
\begin{aligned}
\beta_{\mathrm{th}}^{\mathrm{CIS}}(\bq')
&\le
\beta_{\mathrm{ref}}(\bq)
+
2\,\Delta_{\mathrm{att}}(\bq,\bq') \\
&\le
\beta_{\mathrm{ref}}(\bq)
+
2\,\Delta_{\mathrm{att}}(\tau,\varepsilon_{\mathrm{norm}},Q_{\max}).
\end{aligned}
\label{eq:cis-beta-pointwise}
\end{equation}
Averaging over query--context states for which CIS sharing is applied, let
$\bar\beta_{\mathrm{ref}}
:=\mathbb E_{\bX}[\beta_{\mathrm{ref}}(\bq)]$. Then
\begin{equation}
\bar\beta_{\mathrm{th}}^{\mathrm{CIS}}
\le
\bar\beta_{\mathrm{ref}}
+
2\,\Delta_{\mathrm{att}}(\tau,\varepsilon_{\mathrm{norm}},Q_{\max}).
\label{eq:cis-beta-avg}
\end{equation}
Consequently, in the increasing regime of $g$,
\begin{equation}
\begin{aligned}
\Delta I_{\mathrm{CIS}}
&\le
g\!\left(\bar\delta_{\mathrm{CIS}}\right)\\
&\le
g\!\left(
\bar\delta^*
+
\bar\beta_{\mathrm{ref}}
+
2\,\Delta_{\mathrm{att}}(\tau,\varepsilon_{\mathrm{norm}},Q_{\max})
\right),
\end{aligned}
\label{eq:cis-routing-mi-bound}
\end{equation}
where $\bar\delta^*$ is the oracle dropped mass matched pointwise to the
actual budget $b_t$, and
$\Delta I_{\mathrm{CIS}}
=|I_{\mathrm{full}}-I_{\mathrm{CIS}}|$.
\end{proposition}

\begin{proof}
By Lemma~\ref{lem:reuse-lb} and the definition of
$\beta_{\mathrm{ref}}$,
\[
\tau_{\mathrm{pre}}(\bq')
\ge
\tau_{b_t}^*(\bq)
-\beta_{\mathrm{ref}}(\bq)
-\Delta_{\mathrm{att}}(\bq,\bq').
\]
Lemma~\ref{lem:oracle-mass} gives
$\tau_{b_t}^*(\bq)\ge
\tau_{b_t}^*(\bq')-\Delta_{\mathrm{att}}(\bq,\bq')$, hence
\[
\tau_{\mathrm{pre}}(\bq')
\ge
\tau_{b_t}^*(\bq')
-\beta_{\mathrm{ref}}(\bq)
-2\,\Delta_{\mathrm{att}}(\bq,\bq'),
\]
which proves Eq.~\eqref{eq:cis-beta-pointwise}. Taking expectations gives
Eq.~\eqref{eq:cis-beta-avg}. Since
$\bar\delta_{\mathrm{CIS}}
=
\bar\delta^*+\bar\beta_{\mathrm{th}}^{\mathrm{CIS}}$,
the universal certificate in Eq.~\eqref{eq:mi_bound} and monotonicity of
$g$ yield Eq.~\eqref{eq:cis-routing-mi-bound}.
\end{proof}

\paragraph{Dilation parameter}
No separate "under-dilation" error term is required: changing $(m,r)$
changes $S_t^{\mathrm{CIS}}$, its actual budget $b_t$, and therefore
$\beta_{\mathrm{ref}}$. This keeps the theoretical quantity aligned with the
budget-matched empirical gap used in Sec.~\ref{sec:cis_matched_control}.

\subsection{Exponential Recency Beyond Sink}
\label{sec:exp-decay}
We adopt RoPE positional embeddings~\cite{su2024roformer} and analyze their interaction with self-attention~\cite{vaswani2017attention}. Excluding sink tokens, the multi-frequency rotations in RoPE cause phase dispersion with distance: representations of far-apart tokens become increasingly orthogonal in the rotated space, sharply reducing similarity (a \emph{long-term decay} effect). Softmax attention amplifies this decay and translates it into exponentially smaller attention weights. Empirically, attention scores decrease roughly linearly or sublinearly with distance to the current token, while attention probabilities decay exponentially or faster, consistent with prior measurements~\cite{zhang2023h2o,xiao2023streamingllm}. Thus, we assume there exist layer-dependent constants $\lambda_\ell>0$ and a factor
$\kappa_\ell\in(0,1]$ such that for all $i>C_{\text{sink}}$,
\begin{equation}
\label{eq:exp-decay}
A^{(\ell)}_i(\bq^{(\ell)}_t)
\ \le\
\kappa_\ell\,(1-\rho_\ell)\,\rho_\ell^{\,t-i},
\quad \rho_\ell \stackrel{\mathrm{def}}{=} e^{-\lambda_\ell}\in(0,1).
\end{equation}
Equivalently, the non-sink portion of $A^{(\ell)}$ is dominated pointwise by a geometric tail with ratio $\rho_\ell$ and envelope mass $\kappa_\ell$. 

In addition, the similarity of cross-layer hidden states (including keys and values) can exceed 95\% beyond certain layer depth (generally with $\ell{>}N/2$)~\cite{liu2405minicache}. We then assume that the one-layer key update decays with depth. Specifically, for $\Delta\bk^{(\ell)}_i\stackrel{\mathrm{def}}{=}\bk^{(\ell)}_i-\bk^{(\ell-1)}_i$, there exist $B>0$ and $\mu>0$ such that, absorbing any fixed depth offset into $B$,
\begin{equation}
\label{eq:cross-layer-update}
\max_{C_{\text{sink}}<i\le E_\ell(t)} \big\|\Delta\bk^{(\ell)}_i\big\|
\ \le\
B\,e^{-\mu\ell},\qquad \ell\ge\ell_s,
\end{equation}
whenever $E_\ell(t)>C_{\text{sink}}$; otherwise ETF is inactive and its local perturbation is zero.

\begin{lemma}[Logit change under key perturbations]
\label{lem:logitDelta}
Let $a_i=\tfrac{1}{\sqrt d}(\bq^\top \bk_i)$ and
$a'_i=\tfrac{1}{\sqrt d}(\bq^\top \bk'_i)$. Then
\[
\|a'-a\|_\infty
\ \le\ \frac{\|\bq\|}{\sqrt d}\,\max_i\|\bk'_i-\bk_i\|.
\]
\emph{Proof.} Immediate from Cauchy--Schwarz. \qedhere
\end{lemma}

\subsection{PSAW Bounds on Dropped Mass}
\label{sec:psaw-bound}

Following the PSAW selector defined in Sec.~\ref{sec:psaw}, and dropped set $\mathcal{D}_\ell=\{C_{\text{sink}}+1,\dots,P_\ell(t)-1\}$.
The dropped mass at layer $\ell$ is $\delta^{\mathrm{PSAW}}_\ell
\stackrel{\mathrm{def}}{=}\sum_{i\in\mathcal{D}_\ell} A^{(\ell)}_i(\bq^{(\ell)}_t)$. Let $u_\ell\stackrel{\mathrm{def}}{=}\phi^{\alpha(\ell-\ell_s)/(N-\ell_s)}$ and the window-start distance $D_\ell\stackrel{\mathrm{def}}{=} t-P_\ell(t)+1$. From Eq.~\eqref{eq:psaw}, $D_\ell\ge u_\ell t$.

\begin{theorem}[PSAW worst-case bound]
\label{thm:PSAW}
Under~\eqref{eq:exp-decay}, for every $\ell\ge\ell_s$,
\[
\delta^{\mathrm{PSAW}}_\ell
\ \le\
\kappa_\ell\,e^{-\lambda_\ell\,D_\ell}
\ \le\
e^{-\lambda_\ell\,D_\ell}.
\]
\emph{Proof.}
Reindex non-sink positions by distance $d=t-i\in\{0,1,2,\dots\}$. Then $\mathcal{D}_\ell$ corresponds to $d\ge D_\ell$ and
\begin{equation}
\begin{aligned}
\delta^{\mathrm{PSAW}}_\ell
&=\sum_{i\in\mathcal{D}_\ell}A^{(\ell)}_i
\ \le\
\kappa_\ell\sum_{d\ge D_\ell}(1-\rho_\ell)\rho_\ell^{\,d} \\
&=\kappa_\ell\,\rho_\ell^{\,D_\ell}
=\kappa_\ell\,e^{-\lambda_\ell D_\ell}.
\end{aligned}
\end{equation}
Since $\kappa_\ell\le 1$, the claim follows.
\qedhere
\end{theorem}

For the top layer $\ell=N$, $u_N=\phi^\alpha$ and $D_N\ge \phi^\alpha t$, we then have
\[
\delta^{\mathrm{PSAW}}_N
\ \le\
\kappa_N\,e^{-\lambda_N\,\phi^\alpha t}
\ \le\
e^{-\lambda_N\,\phi^\alpha t}.
\]
Given a target budget $\beta^{\mathrm{PSAW}}\in(0,1)$, any $(\phi,\alpha)$ satisfying
\[
\phi^\alpha
\ \ge\
\frac{1}{\lambda_N t}\,\log\!\frac{\kappa_N}{\beta^{\mathrm{PSAW}}}
\]
ensures $\delta^{\mathrm{PSAW}}_N\le \beta^{\mathrm{PSAW}}$. Otherwise, if $\lambda_\ell \ge \underline{\lambda}>0$ almost surely, then
\[
\mathbb{E}\big[\delta^{\mathrm{PSAW}}_\ell\big]
\ \le\
\kappa_\ell\,e^{-\underline{\lambda}\,D_\ell}
\ \le\
e^{-\underline{\lambda}\,D_\ell}.
\]

\subsection{ETF Bounds on Additional Mass Gap}
\label{sec:etf-bound}
For ETF defined in Sec.~\ref{sec:etf}, let
$\mathcal F_\ell=\{C_{\text{sink}}+1,\dots,E_\ell(t)\}$ denote the frozen non-sink prefix, with $\mathcal F_\ell=\varnothing$ when $E_\ell(t)\le C_{\text{sink}}$. To isolate the layerwise perturbation, let $A^{(\ell)}$ denote the counterfactual unfrozen attention distribution under the same incoming state. We separate the implemented ETF operator into two steps: $\widetilde A^{(\ell)}$ is the intermediate distribution obtained by reusing previous-layer keys for $i\in\mathcal F_\ell$ while retaining the full routing support, and $\widehat A^{(\ell)}$ is the actual ETF distribution after removing $\mathcal F_\ell$ from the routing support and renormalizing over the remaining visible keys. Define
\[
\delta^{\mathrm{frozen}}_\ell
\ \stackrel{\mathrm{def}}{=}\
\widetilde A^{(\ell)}(\mathcal F_\ell)
\]
and the ETF-induced mass gap (a total-variation budget)
\[
\beta^{\mathrm{ETF}}_\ell(\bq^{(\ell)}_t)
\ \stackrel{\mathrm{def}}{=}\ \frac12\big\|\,\widehat A^{(\ell)}(\bq^{(\ell)}_t)-A^{(\ell)}(\bq^{(\ell)}_t)\big\|_1.
\]

\begin{theorem}[ETF: per-layer worst-case bound]
\label{thm:ETF}
Under~\eqref{eq:cross-layer-update} and $\|\bq^{(\ell)}_t\|\le Q_{\max}$,
\[
\beta^{\mathrm{ETF}}_\ell(\bq^{(\ell)}_t)
\ \le\
\frac{Q_{\max}}{\sqrt d}\,B\,e^{-\mu\ell}
+\delta^{\mathrm{frozen}}_\ell.
\]
Moreover, under~\eqref{eq:exp-decay}, define
$D_\ell^{\mathrm{ETF}}\stackrel{\mathrm{def}}{=}t-E_\ell(t)$. Then
\[
\delta^{\mathrm{frozen}}_\ell
\ \le\
\kappa_\ell e^{-\lambda_\ell D_\ell^{\mathrm{ETF}}}
+\frac{Q_{\max}}{\sqrt d}\,B\,e^{-\mu\ell},
\]
and therefore
\[
\beta^{\mathrm{ETF}}_\ell(\bq^{(\ell)}_t)
\ \le\
\kappa_\ell e^{-\lambda_\ell D_\ell^{\mathrm{ETF}}}
+2\frac{Q_{\max}}{\sqrt d}\,B\,e^{-\mu\ell}.
\]
\emph{Proof.}
Let $a_i=\frac{1}{\sqrt d}(\bq^{(\ell)}_t)^\top \bk^{(\ell)}_i$
and $\tilde a_i=\frac{1}{\sqrt d}(\bq^{(\ell)}_t)^\top \tilde \bk^{(\ell)}_i$, where
$\tilde \bk^{(\ell)}_i=\bk^{(\ell-1)}_i$ if $i\in\mathcal F_\ell$ and
$\tilde \bk^{(\ell)}_i=\bk^{(\ell)}_i$ otherwise. Then $\Delta a_i=\tilde a_i-a_i=-\frac{1}{\sqrt d}(\bq^{(\ell)}_t)^\top \Delta\bk^{(\ell)}_i$
for $i\in\mathcal F_\ell$ and $0$ otherwise.
By Lemma~\ref{lem:logitDelta} and~\eqref{eq:cross-layer-update},
\begin{equation}
\begin{aligned}
\|\tilde a-a\|_\infty
{\le}
\frac{\|\bq^{(\ell)}_t\|}{\sqrt d}
\max_{i\in\mathcal F_\ell} \|\Delta\bk^{(\ell)}_i\|
{\le}
\frac{Q_{\max}}{\sqrt d}\,B\,e^{-\mu\ell}.
\end{aligned}
\end{equation}
By Lemma~\ref{lem:softmaxLip},
\[
\frac12\|\widetilde A^{(\ell)}-A^{(\ell)}\|_1
\le
\frac{Q_{\max}}{\sqrt d}\,B\,e^{-\mu\ell}.
\]
Since $\widehat A^{(\ell)}$ is obtained by truncating $\mathcal F_\ell$ from $\widetilde A^{(\ell)}$ and renormalizing,
\[
\frac12\|\widehat A^{(\ell)}-\widetilde A^{(\ell)}\|_1
=\widetilde A^{(\ell)}(\mathcal F_\ell)
=\delta^{\mathrm{frozen}}_\ell.
\]
The triangle inequality gives the first claim.

For the second claim, $\mathcal F_\ell$ corresponds to distances $d=t-i\ge D_\ell^{\mathrm{ETF}}$, so Eq.~\eqref{eq:exp-decay} gives
\[
A^{(\ell)}(\mathcal F_\ell)
\le
\kappa_\ell e^{-\lambda_\ell D_\ell^{\mathrm{ETF}}}.
\]
Furthermore,
\[
\big|\widetilde A^{(\ell)}(\mathcal F_\ell)-A^{(\ell)}(\mathcal F_\ell)\big|
\le
\frac12\|\widetilde A^{(\ell)}-A^{(\ell)}\|_1,
\]
which yields the stated bound on $\delta^{\mathrm{frozen}}_\ell$ and hence the final inequality.
\qedhere
\end{theorem}

This is a layerwise certificate conditioned on the same incoming state; propagation through subsequent hidden states is evaluated empirically. The first term explicitly accounts for removing the frozen prefix from the routing support, while the second controls the cross-layer stale-key perturbation. If $\mathbb{E}\|\bq^{(\ell)}_t\|\le \bar Q$ and $\lambda_\ell\ge\underline{\lambda}>0$ almost surely, then
\[
\mathbb{E}\big[\beta^{\mathrm{ETF}}_\ell\big]
\ \le\
e^{-\underline{\lambda}D_\ell^{\mathrm{ETF}}}
+2\frac{\bar Q}{\sqrt d}\,B\,e^{-\mu\ell}.
\]
At the top layer, $D_N^{\mathrm{ETF}}\ge\psi^\gamma t$, and therefore
\[
\beta^{\mathrm{ETF}}_N
\le
\kappa_N e^{-\lambda_N\psi^\gamma t}
+2\frac{Q_{\max}}{\sqrt d}\,B\,e^{-\mu N}.
\]

\subsection{Joint pre-hoc certificate and tuning}

At layer $\ell$, let $\delta^{\mathrm{(PSAW+ETF)}}_\ell$ denote the routing-TV perturbation between $A^{(\ell)}$ and the final ETF+PSAW distribution. If $\widehat\delta^{\mathrm{PSAW}}_\ell$ is the PSAW-dropped mass under $\widehat A^{(\ell)}$, then
\begin{equation}
\begin{aligned}
\delta^{\mathrm{(PSAW+ETF)}}_\ell
&\le \beta^{\mathrm{ETF}}_\ell+\widehat\delta^{\mathrm{PSAW}}_\ell
\le \delta^{\mathrm{PSAW}}_\ell+2\beta^{\mathrm{ETF}}_\ell \\
&\le
\kappa_\ell e^{-\lambda_\ell D_\ell}
+2\kappa_\ell e^{-\lambda_\ell D_\ell^{\mathrm{ETF}}}
+4\frac{Q_{\max}}{\sqrt d}\,B\,e^{-\mu\ell},
\end{aligned}
\end{equation}
The displayed bound is worst-case; after taking expectations, $Q_{\max}$ can be replaced by $\bar Q$. The second inequality uses $|\widehat A^{(\ell)}(\mathcal D_\ell)-A^{(\ell)}(\mathcal D_\ell)|\le\beta^{\mathrm{ETF}}_\ell$.
At the top layer $\ell=N$,
$\delta^{\mathrm{PSAW+ETF}}_N
\le
\kappa_N e^{-\lambda_N \phi^\alpha t}
+2\kappa_N e^{-\lambda_N\psi^\gamma t}
+4\frac{Q_{\max}}{\sqrt d}\,B\,e^{-\mu N}$,
with the analogous expected bound using $\bar Q$.

In practice, given targets $\beta^{\mathrm{PSAW}},\beta^{\mathrm{ETF}}\in(0,1)$, it suffices under the average case to choose $(\phi,\alpha)$, $(\psi,\gamma)$, and $\ell_s$ such that
\[
\kappa_N e^{-\underline{\lambda}\phi^\alpha t}
\le \beta^{\mathrm{PSAW}},
\qquad
e^{-\underline{\lambda}\psi^\gamma t}
+2\frac{\bar Q}{\sqrt d}\,B\,e^{-\mu\ell_s}
\le \beta^{\mathrm{ETF}}.
\]
The ETF condition uniformly controls the frozen-prefix truncation term and the cross-layer perturbation over active layers $\ell\ge\ell_s$; $(\psi,\gamma)$ also determine the frozen workload through $E_\ell(t)$. Equivalently, provided $\beta^{\mathrm{ETF}}>e^{-\underline{\lambda}\psi^\gamma t}$,
\[
\phi^\alpha
\ge
\frac{1}{\underline{\lambda} t}\,\log\!\frac{\kappa_N}{\beta^{\mathrm{PSAW}}},
\quad
\ell_s
\ge
\frac{1}{\mu}\,\log\!\frac{2\bar Q B}{\big(\beta^{\mathrm{ETF}}-e^{-\underline{\lambda}\psi^\gamma t}\big)\sqrt d}.
\]

\subsection{Cuda Acceleration}
\label{sec:app_cuda_acceleration}

\begin{algorithm}[!t]
\caption{Combined CPE Inference at Layer $\ell$}
\label{alg:cpe_inference_long}
\scriptsize

\algrenewcommand\algorithmicindent{0.8em}

\begin{algorithmic}[1]
\Require $Q,K,V$, layer $\ell$, depth $N$, layer state $\mathcal{C}_\ell$
\Statex \hspace{\algorithmicindent}
CIS $(s,\tau,k,C_s,C_l,m,r)$;
PSAW $(\ell_s,\phi,\alpha,W_{\min})$;
ETF $(\ell_s,\psi,\gamma)$
\State $L\gets |K|$, \quad
$u_\ell\gets\operatorname{clip}((\ell-\ell_s)/(N-\ell_s),0,1)$
\If{$|Q|>1$} \Comment{prefill}
    \If{$\ell<\ell_s$}
        \State \Return $\Call{FlashAttention2}{Q,K,V}$
    \EndIf
    \State $\phi_\ell\gets\phi^{\alpha u_\ell}$ if PSAW else $1$
    \State $F_\ell\gets L-\lceil\psi^{\gamma u_\ell}L\rceil+1$
           if ETF else $1$
    \State $\mathcal{A}\gets[1,C_s]\cup[F_\ell,L]$
    \ForAll{$i=1,\ldots,L$ \textbf{ in parallel}}
        \State $w_i\gets\min\{i,\max(\lceil\phi_\ell i\rceil,W_{\min})\}$
        \State $\mathcal{V}_i\gets[1,C_s]\cup[\max(i-w_i+1,F_\ell),i]$
    \EndFor
    \State $O\gets\Call{FA2-CPE}{Q,K,V,\mathcal{A},\{\mathcal{V}_i\}}$
          \Comment{skip inactive Q/KV tiles}
    \State \Return $O$ \Comment{$O_i=0$ for ETF-frozen rows}
\Else
    \State $t\gets L$; \quad
    $p\gets t-\max(\lceil\phi^{\alpha u_\ell}t\rceil,W_{\min})+1$
           if PSAW else $1$
    \State $j\gets t\bmod s$; \quad
    $a_{bh}\gets\Call{CISMatch}{q_{bh},\mathcal{C}_\ell[0{:}j],\tau}$
          \Comment{batched over $(b,h)$}
    \State $\mathcal{R}\gets\{(b,h):a_{bh}=\bot\}$
    \ForAll{$(b,h)$ \textbf{ in parallel}}
        \If{$(b,h)\in\mathcal{R}$}
            \State $\mathcal{B}_{bh}\gets
            \Call{TopK}{q_{bh}K_{[\max(C_s+1,p),\,t-C_l]}^\top,k}$
            \State $\mathcal{C}_\ell[j]\gets(q_{bh},\mathcal{B}_{bh})$;
                   $\mathcal{D}_{bh}\gets\emptyset$
        \Else
            \State $\mathcal{B}_{bh}\gets\mathcal{C}_\ell[a_{bh}].\mathrm{idx}$
            \State $\mathcal{D}_{bh}\gets
            \Call{Dilate}{\operatorname{Top}_m(\mathcal{B}_{bh}),r}$
        \EndIf
        \State $\mathcal{S}_{bh}\gets\Call{SortUnique}{
        ([1,C_s]\cup[t-C_l+1,t]\cup\mathcal{B}_{bh}\cup\mathcal{D}_{bh})
        \cap([1,C_s]\cup[p,t])}$
        \State $O_{bh}\gets
        \Call{FusedSparseAttn}{q_{bh},K_{\mathcal{S}_{bh}},
        V_{\mathcal{S}_{bh}}}$
    \EndFor
    \State \Return $O$
\EndIf
\end{algorithmic}
\end{algorithm}

\paragraph{Backend organization.}
We implement CPE as a forward-only CUDA backend with separate but
complementary prefill and decoding paths. During prefill, PSAW and
attention-level ETF are fused into a modified FlashAttention-2
forward kernel~\cite{dao2024flashattention}. During autoregressive
decoding, CIS matching, retrieval scoring, index composition, and
token-sparse attention are implemented using dedicated CUDA
operators. Each Transformer layer maintains an independent CIS state,
including normalized anchor queries, anchor-valid flags, and their
associated top-$k$ indices. The fast paths consume FP16 or BF16
tensors in the $[B,H,S,d]$ layout, support GQA/MQA by mapping each
query head to its corresponding KV head, and currently support
$d\leq256$ on Ampere-or-newer GPUs. Unsupported training, backward,
or complex external-mask cases use a correctness-oriented fallback. Algorithm~\ref{alg:cpe_inference_long} details the execution pipeline.

\paragraph{Joint PSAW--ETF prefill.}
The PSAW and ETF schedules are resolved before launching the
FlashAttention kernel. Let
\begin{equation}
u_\ell =
\operatorname{clip}\left(
    \frac{\ell-\ell_s}{N-\ell_s},0,1
\right),\qquad
\phi_\ell=\phi^{\alpha u_\ell},\qquad
\psi_\ell=\psi^{\gamma u_\ell}.
\end{equation}
For query position $i$, the implementation uses the PSAW window
length
\begin{equation}
W_{\ell,i}
=
\min\left\{
i,\,
\max\left(\left\lceil\phi_\ell i\right\rceil,W_{\min}\right)
\right\},
\end{equation}
where $W_{\min}=C_{\rm sink}+k+C_{\rm local}$ by default prevents the
visible range from becoming smaller than the configured sparse
budget. Let
\begin{equation}
F_\ell=L-\left\lceil\psi_\ell L\right\rceil+1
\end{equation}
denote the first ETF-active position. When a module is disabled, we
set $\phi_\ell=1$ for PSAW or $F_\ell=1$ for ETF. The active query rows
and visible keys are therefore
\begin{align}
\mathcal{A}_\ell
&=
\{1,\ldots,C_{\rm sink}\}
\cup
\{F_\ell,\ldots,L\},\\
\mathcal{V}_{\ell,i\in\mathcal{A}_\ell}
&=
\{1,\ldots,C_{\rm sink}\}
\cup
\left\{
\max(i-W_{\ell,i}+1,F_\ell),\ldots,i
\right\}
\end{align}
Thus, PSAW and ETF are composed by intersection rather than by
executing two independent attention operators.

We modify the FlashAttention-2 key/value tile iterator to traverse
only the two disjoint visible intervals: the sink prefix and the
recent active tail. For each query tile, the kernel computes the last
sink tile and the first tail tile once, processes the tail tiles, and
then jumps directly to the sink tiles. K/V tiles in the excluded
middle region do not issue global-memory loads and do not participate
in the QK or PV matrix multiplications. A query tile containing no
ETF-active rows exits before the main attention pipeline and writes
zero attention outputs and log-sum-exp values. Tiles intersecting a
PSAW or ETF boundary are retained, but an exact element-wise
causal--PSAW--ETF predicate is applied to the score fragments before
softmax. Consequently, no dense $L\times L$ mask is materialized.

The sparse traversal preserves the original FlashAttention-2
CUTLASS/CUTE execution pipeline, including tensor-core MMA,
asynchronous global-to-shared-memory staging, shared-memory tiling,
and online softmax. The $[B,H,S,d]\rightarrow[B,S,H,d]$ conversion is
passed to the extension as a strided view, so the input Q/K/V tensors
are not explicitly transposed or recopied. At the operator level,
frozen ETF query rows return zero attention updates; when ETF also
freezes the MLP and normalization paths, the same active-row mask is
applied in the model-specific decoder layer to reuse the previous
hidden states.

\paragraph{CIS and sparse decoding.}
At decoding step $t$, CIS first performs cosine-similarity matching
between every query head and the valid anchors in its current
size-$s$ block. One 256-thread CTA is assigned to each batch--query
head. The query norm is computed using warp-shuffle and block-level
reductions, while the eight warps in the CTA evaluate different
anchors concurrently. Only heads for which no anchor exceeds
$\tau$ are compacted into a retrieval-row array.

For each retrieval row, the dense-QK kernel scores only the middle
range
\begin{equation}
\left[
\max(C_{\rm sink}+1,P_\ell(t)),\,
t-C_{\rm local}
\right],
\end{equation}
because sink and local tokens are retained deterministically. A warp
computes one query--key dot product: adjacent lanes access adjacent
head-dimension elements, and the partial products are reduced with
warp shuffles. Hence, shared heads neither compute a dense QK vector
nor allocate a corresponding score row. Top-$k$ selection is
performed only on the compact $R\times L_{\rm mid}$ score tensor,
where $R$ is the number of retrieval heads. We retain the optimized
ATen CUDA \texttt{topk} primitive as a stable library boundary rather
than materializing scores for all $BH$ heads.

The index-composition kernel then merges sink tokens, local tokens,
the retrieved or shared top-$k$ set, and the $\pm r$ neighbors of the
top-$m$ shared indices. It simultaneously applies the PSAW lower
boundary. For the common index capacities, one CTA generates the
candidates in registers and uses CUB block radix sort, followed by
in-place duplicate removal. Larger index capacities use a batched
sorting fallback. Sorting removes repeated sink/local/dilation
indices and provides monotonically ordered gathers, improving
spatial locality in the subsequent K/V accesses.

Finally, one CTA per batch--query head performs fused sparse
attention over the resulting index set. Selected keys are processed
in tiles of 128 entries. Each warp computes the QK product for one
selected key, while block reductions compute the tile maximum and
normalization sum. FP32 running maxima and sums implement online
softmax across tiles, and each output-dimension thread accumulates
the corresponding value elements. Therefore, gather, QK, softmax,
and PV are executed without constructing either a full attention
matrix or a full-length probability vector.

\paragraph{Parallel schedule and optimized baselines.}
The shared and retrieval paths reconverge after top-$k$ selection:
all heads participate in the index-composition and sparse-attention
kernels. The dependent stages are issued on the current CUDA stream;
parallel acceleration is realized through independent CTAs over
batch--head rows and by completely removing dense retrieval work for
shared heads, rather than by overlapping two dependent attention
branches on separate streams. The critical path contains dense
middle-region scoring and top-$k$ only for the non-shared retrieval
rows.

\subsection{Supplementation of Experiment Results}

\begin{table*}[t]
\centering
\caption{
\textbf{Controlled attribution and hyperparameter-transfer analysis.}
Panel A compares CIS with budget-, retrieval-, and implementation-matched controls.
Panel B compares CIS variants under matched settings.
$\overline{\beta}_{\mathrm{emp}}$ denotes the mean empirical retained-mass gap.
Accuracy retention is averaged over Qwen2.5-7B-Instruct and Qwen3-4B.
}
\label{tab:cis_transfer_per_model}
\scriptsize
\setlength{\tabcolsep}{3.5pt}
\renewcommand{\arraystretch}{1.08}

\textbf{Panel A: CIS matched-control attribution for Qwen2.5-7B-Instruct}

\resizebox{\textwidth}{!}{
\begin{tabular}{lccccccc}
\hline
Variant
& Parameter setting
& $\hat{\rho}\downarrow$
& Avg.\ KV
& $\overline{\beta}_{\mathrm{emp}}\downarrow$
& GSM8K Ret.$\uparrow$
& COQA Ret.$\uparrow$
& LongBench v2 Ret.$\uparrow$
\\
\hline

HShare-matched
& Matched
& 10.0 & 1060 & 0.069 & 0.6217 & 0.5092 & \cotcell{20.1}{25.5}
\\

DirectReuse
& Global-frozen
& 10.2 & 1060 & 0.021 & 0.7655 & 0.5384 & \cotcell{29.7}{29.3}
\\

RandomExpansion
& Global-frozen
& 10.3 & 1060 & 0.036 & 0.6836 & 0.5084 & \cotcell{28.8}{26.5}
\\

CIS
& Global-frozen
& 9.8 & 1064 & 0.018 & 0.7787 & 0.5336 & \cotcell{30.2}{29.6}
\\

CIS
& Per-model tuned
& 10.2 & 1062 & 0.017 & 0.7876 & 0.5389 & \cotcell{30.6}{30.2}
\\

\hline
\end{tabular}
}

\vspace{1mm}

\textbf{Panel B: CIS matched-control attribution for Qwen3-4B}

\resizebox{\textwidth}{!}{
\begin{tabular}{lccccccc}
\hline
Variant
& Parameter setting
& $\hat{\rho}\downarrow$
& Avg.\ KV
& $\overline{\beta}_{\mathrm{emp}}\downarrow$
& GSM8K Ret.$\uparrow$
& COQA Ret.$\uparrow$
& LongBench v2 Ret.$\uparrow$
\\
\hline

HShare-matched
& Matched
& 10.0 & 1060 & 0.049 & 0.7786 & 0.6007 & \cotcell{28.4}{29.2}
\\

DirectReuse
& Global-frozen
& 10.3 & 1060 & 0.024 & 0.8159 & 0.6751 & \cotcell{32.9}{39.5}
\\

RandomExpansion
& Global-frozen
& 9.9 & 1060 & 0.043 & 0.7538 & 0.5824 & \cotcell{29.5}{36.3}
\\

CIS
& Global-frozen
& 10.4 & 1065 & 0.020 & 0.8226 & 0.6763 & \cotcell{33.5}{41.2}
\\

CIS
& Per-model tuned
& 10.2 & 1068 & 0.018 & 0.8226 & 0.6774 & \cotcell{33.3}{41.5}
\\

\hline
\end{tabular}
}

\end{table*}

\subsubsection{CIS Transfer Results}

We report detailed CIS transfer results for two evaluated models: Qwen2.5-7B-Instruct, Qwen3-4B. The results are presented in Tab.~\ref{tab:cis_transfer_per_model}. For HShare, we set the sharing ratios to 1/2, 1/2, and 2/5, to match the corresponding CIS retrieval ratio $\hat{\rho}$. DirectReuse preserves the CIS retrieval budget but removes the dilation operation, such that the retrieved tokens are directly shared without expansion. RandomExpansion uses the retrieval budget of CIS$^\ast$ and supplements the retrieved tokens with additional randomly selected tokens. For the model-specific tuned CIS configurations, we retain the CIS$^\ast$ setup and adjust only the similarity threshold $\tau$ and the dilation factor $m$.

\begin{figure*}[!t]
  \centering
  \captionsetup[subfloat]{font=scriptsize}
  \subfloat[Dilation sensitivity.]{
  \footnotesize
  \label{fig:cis_m_div_sweep}
    \includegraphics[width=0.23\textwidth]{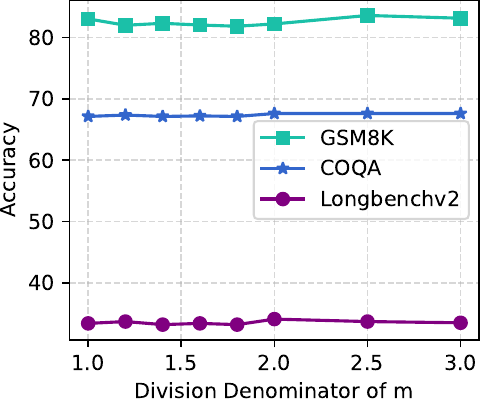}
  }\hfill
  \subfloat[Processed-KV composition on LongBench v2.]{
  \footnotesize
  \label{fig:cis_dilation_budget}
    \includegraphics[width=0.23\textwidth]{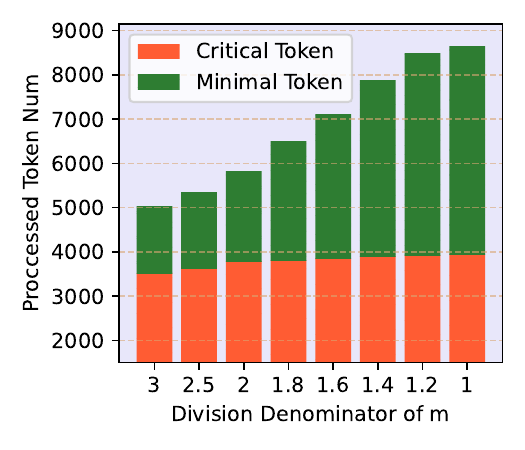}
  }\hfill
  \subfloat[Similarity-threshold sensitivity.]{
  \footnotesize
  \label{fig:cis_sim_sweep}
    \includegraphics[width=0.23\textwidth]{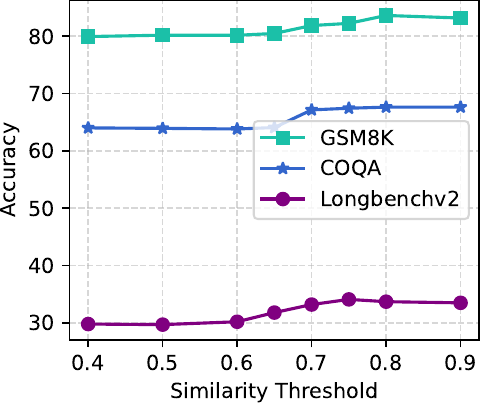}
  }\hfill
  \subfloat[Top-50\% oracle-index overlap.]{
  \footnotesize
  \label{fig:cis_sim_top50}
    \includegraphics[width=0.23\textwidth]{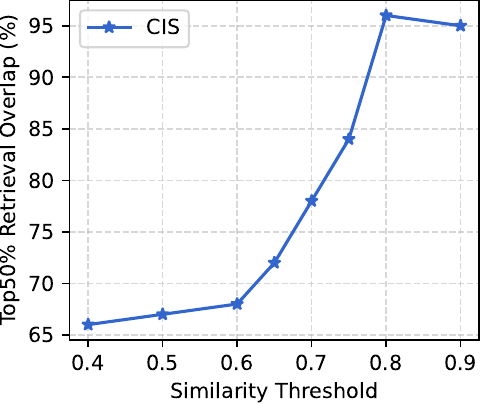}
  }
  \vspace{-0.05in}
  \caption{\textbf{Failure-mode and sensitivity analysis of CIS on Qwen3-4B.}
(a) Downstream accuracy under different dilation strengths, where
$m=\lfloor k/n_{\mathrm{den}}\rfloor$ and a smaller $n_{\mathrm{den}}$
corresponds to more aggressive dilation.
(b) Average processed-KV composition on LongBench v2 when varying the
dilation strength. The \textcolor{orange}{orange segment} denotes selected
indices that also belong to the reference top-$k$ oracle
($k=4096$, following Sec.~\ref{sec:longbench}), whereas the
\textcolor{darkgreen}{green segment} denotes additional non-oracle
("minimal") indices introduced by dilation.
(c) Downstream accuracy as a function of the CIS similarity threshold
$\tau$.
(d) Overlap with the top-50\% highest-attention-mass indices of the
top-$k$ oracle.
Dilation causes only limited accuracy fluctuations even when it
substantially enlarges the processed KV set, whereas an overly permissive
similarity threshold sharply reduces the preservation of high-mass
critical indices and leads to considerably larger downstream degradation.
These results identify mismatched cross-query sharing, rather than
dilation itself, as the primary failure mode of CIS.}
  \label{fig:cis-failure}
  \vspace{-0.1in}
\end{figure*}

\subsubsection{CIS Failure Modes}
\label{sec:cis-failure}

We further characterize the failure modes of CIS along its two main
control dimensions: neighboring dilation and similarity-gated sharing.
We first vary the dilation strength using the CIS configurations in
Sec.~\ref{sec:gsm8k&coqa} and Sec.~\ref{sec:longbench}. Specifically,
we set
$m=\lfloor k/n_{\mathrm{den}}\rfloor$
and vary the denominator $n_{\mathrm{den}}$, where a smaller
$n_{\mathrm{den}}$ corresponds to dilating a larger fraction of the
retrieved critical indices.

\textbf{Dilation is not the dominant source of accuracy degradation.}
As shown in Fig.~\ref{fig:cis_m_div_sweep}, downstream performance is
remarkably stable over a wide range of dilation strengths on GSM8K,
COQA, and LongBench v2. This stability holds even though the actual
processed-KV workload changes substantially. As shown in
Fig.~\ref{fig:cis_dilation_budget}, increasing $m$ continuously enlarges
the selected set, while most of the additional entries are minimal
tokens that do not belong to the reference top-$k$ oracle. In contrast,
the number of oracle-overlapping critical tokens increases much more
slowly and gradually saturates.

This behavior reflects two competing effects of dilation. Initially,
expanding the neighborhood around high-attention anchor indices recovers
critical entries that may have shifted between adjacent queries,
thereby reducing false negatives in the shared set. Once these
high-mass neighborhoods are sufficiently covered, however, further
dilation mainly introduces low-attention minimal tokens. Under an
aggressive KV sparsity regime, e.g., approximately $1/16$ on LongBench
v2, these additional entries can slightly perturb the renormalized
sparse-attention distribution and therefore produce small,
non-monotonic accuracy fluctuations. Nevertheless, because their
attention mass is typically small, such perturbations remain limited.
Increasing the dilation coverage therefore tends to recover the
important missing indices and brings the sparse model back toward the
dense-model operating point, while excessive dilation primarily incurs
additional computation rather than severe accuracy loss.

This observation is also consistent with the CIS certificate in Eq.~\eqref{eq:cis-routing-mi-bound}. Recall that for a reused query $\bq'$, the retained-mass error is controlled by
\[
    \beta_{\mathrm{th}}^{\mathrm{CIS}}(\bq')
    \le
    \beta_{\mathrm{ref}}(\bq)
    + 2\Delta_{\mathrm{att}}(\bq,\bq').
\]
Changing $(m,r)$ modifies the actual dilated set and therefore primarily
affects the budget-matched reference gap $\beta_{\mathrm{ref}}$.
Increasing dilation can reduce this gap by recovering displaced
high-mass indices, whereas adding further low-mass indices after the
coverage has saturated has only a limited effect on the retained
attention mass. The empirical robustness in
Fig.~\ref{fig:cis_m_div_sweep} is consistent with this interpretation.

\begin{table*}[t]
\centering
\caption{\textbf{Hyperparameter tuning on LLaMA2-Chat-7B for CPE.} The COQA setting for CIS follows the CIS$^\ast$ configuration, and HyperKV follows the HyperKV configuration, as described in Sec.~\ref{sec:gsm8k&coqa}. WikiText perplexity (PPL) is measured only during the prefilling stage. The Avg.Token measures the average processed tokens per head only in the decoding stage.}
\label{tab:hyper-tune}
\begin{tabular}{l|ccccccc|cccccc}
\toprule
Methods & $s$ & $\tau$ & $r$ & $\phi$ & $\psi$ & $\alpha$ & $\gamma$ & $\hat{\rho}$ & Avg. Token & Wiki PPL & GSM8K(flexible) & COQA(EM/F1)\\
\midrule
\multirow{4}{*}{CIS}
  & 4 & 0.8 & 1 & - & - & - & - & 0.301 & 124 & - & 0.1834 & 0.6023 / 0.7632\\
  & 8 & 0.7 & 1 & - & - & - & - & 0.152 & 131 & - & 0.1801 & 0.6090 / 0.7648\\
  & 8 & 0.8 & 2 & - & - & - & - & 0.184 & 167 & - & 0.1653 & 0.6083 / 0.7650\\
  & 32 & 0.8 & 1 & - & - & - & - & 0.0046 & 136 & - & 0.1693 & 0.6080 / 0.7641\\

\midrule

\multirow{2}{*}{PSAW}
  & - & - & - & 0.5 & - & 1 & - & - & - & 6.120 & 0.1896 & 0.5925 / 0.7568\\
  & - & - & - & 0.7 & - & 1.5 & - & - & - & 6.113 & 0.1803 & 0.5938 / 0.7573\\

\midrule

\multirow{2}{*}{ETF}
  & - & - & - & - & 0.5 & - & 1.5 & - & - & 6.162 & 0.1864 & 0.5962 / 0.7563\\
  & - & - & - & - & 0.4 & - & 1 & - & - & 6.156 & 0.1883 & 0.5962 / 0.7537\\

\midrule

\multirow{2}{*}{CPE}
  & 8 & 0.8 & 2 & 0.7 & 0.5 & 1.2 & 1.2 & 0.176 & 166 & 6.148 & 0.1348 & 0.5998 / 0.7591 \\
  & 32 & 0.8 & 1 & 0.7 & 0.5 & 1 & 1 & 0.0044 & 133 & 6.136 & 0.1634 & 0.5965 / 0.7572\\

\bottomrule
\end{tabular}
\end{table*}

\textbf{Overly permissive similarity gating is the main CIS failure
mode.}
A substantially different behavior is observed when varying the
similarity threshold $\tau$. As shown in
Fig.~\ref{fig:cis_sim_sweep}, relaxing the threshold from its default
operating region leads to markedly larger downstream degradation than
changing the dilation strength. More importantly,
Fig.~\ref{fig:cis_sim_top50} shows that this degradation is accompanied
by a sharp loss of high-importance critical indices: the overlap with
the top-50\% highest-mass oracle indices is about $96\%$ at
$\tau=0.8$, but decreases to approximately $78\%$ at $\tau=0.7$ and
$68\%$ at $\tau=0.6$, and approaches roughly $66\%$ under very
permissive sharing. Thus, the degradation is associated not merely
with a change in the number of processed tokens, but with reusing
critical sets across queries whose attention distributions are no
longer sufficiently aligned.

This failure mode directly agrees with the continuity term in the CIS
analysis. Under the similarity condition, the cross-query contribution
to the certificate is controlled by
$\Delta_{\mathrm{att}}(\tau)$, which increases as $\tau$ decreases.
Hence, lowering the threshold admits query pairs with larger possible
attention-distribution variation and correspondingly weakens the
retained-mass guarantee. Empirically, the simultaneous decrease in
top-50\% oracle overlap and downstream accuracy provides direct evidence
for this mechanism.

The degradation does not, however, grow indefinitely as $\tau$ is
further reduced. CIS uses the similarity threshold as an admission
criterion for sharing and selects the most similar eligible anchor
within the current local block. For a query that already has a
high-similarity anchor, lowering $\tau$ only admits additional,
less-similar candidates and therefore does not change the selected
anchor. Further relaxation mainly affects the remaining queries that
previously had no eligible reference. Once most shareable queries have
already been activated, the number of newly changed sharing decisions
becomes small, causing the marginal effect of further reducing $\tau$
to gradually saturate. Overall, these results show that CIS is
relatively insensitive to its dilation strength, whereas its principal
failure case arises when the similarity gate becomes too permissive
and allows reuse across insufficiently aligned queries.

\subsubsection{Hyperparameter Tuning}
\label{sec:hyper-tune}
\paragraph{Setting} We tune the key knobs of our compression pipeline in Table~\ref{tab:hyper-tune}. For \emph{CIS}, we adopt the CIS$^\ast$ configuration from Sec.~\ref{sec:gsm8k&coqa} and vary the block size $s$, similarity threshold $\tau$, and neighbor radius $r$. For \emph{PSAW} and \emph{ETF} in isolation, we apply them on dense attention, fix $\ell_s=\lfloor 3N/4\rfloor$, and control pruning/freezing aggressiveness via $(\phi,\alpha)$ for PSAW and $(\psi,\gamma)$ for ETF; PSAW is active in both prefill and decoding. For \emph{CPE}, we combine all three components, adding PSAW and ETF to CIS under TSA. We report: the averaged computation ratio $\hat{\rho}$, the average processed tokens per head in decoding (Avg.Token), \emph{prefill}-only WikiText perplexity (PPL), and task accuracy on GSM8K (flex) and COQA (EM/F1).

\paragraph{CIS} Varying $(s,\tau,r)$ reveals the accuracy–efficiency trade-off. Smaller $s$ increases retrieval frequency (higher $\hat{\rho}$), whereas enlarging to $s{=}32$ drives the retrieval ratio down to $\hat{\rho}{=}0.0046$ with only a modest GSM8K accuracy drop; COQA remains notably stable across all $s$. At fixed $s{=}8$, increasing the dilation radius to $r{=}2$ substantially enlarges the active set but yields negligible accuracy change on both GSM8K and COQA; inspection shows that the added neighbors often include low-importance tokens, perturbing the attention-mass distribution without benefits. Slightly loosening $\tau$ has little effect on accuracy. Overall, $s$ is the dominant lever for efficiency, while $\tau$ and $r$ offer fine-grained control of recall versus budget.

\paragraph{PSAW} Prefill pruning is gentle on language modeling quality and downstream accuracy. Increasing aggressiveness from $(\phi{=}0.5,\alpha{=}1)$ to $(\phi{=}0.7,\alpha{=}1.5)$ barely changes WikiText PPL, and COQA is essentially unchanged, while GSM8K (flexible match) decreases only slightly.

\paragraph{ETF} Freezing during prefill shows similarly small effects. Adjusting from $(\psi{=}0.5,\gamma{=}1.5)$ to $(\psi{=}0.4,\gamma{=}1)$ yields virtually identical PPL and a tiny trade between GSM8K and COQA. ETF is therefore a stable knob for further prefill savings and slightly milder freezing is a safe default setting.

\paragraph{CPE} Combining CIS with PSAW and ETF preserves efficiency while keeping quality competitive. With aggressive dilation ($r{=}2$), GSM8K dips to $0.1348$. Increasing block size and reducing dilation to $r{=}1$ lowers selection frequency and trims Avg. Token by ${\sim}20\%$, recovering GSM8K to $0.1634$ with COQA nearly unchanged. Overall, CPE performs best with moderate–large blocks, restricted dilation ($r{=}1$), and mild PSAW/ETF schedules.

\begin{table}[tb]
\centering
\footnotesize
\caption{Variants of CIS on the similarity designs. The model is LlaMA2-7b-chat.}
\label{tab:cis_variant}
{
\begin{tabular}{l|c|ccc}
\toprule
Method & $s$ & GSM8K (flexible/strict) & COQA\\ 
\midrule



\multirow{2}{*}{Key} & 8  & 0.1547 / 0.1482 & 0.5613 \\ 
 & 16  &  0.1405 / 0.1334 & 0.5608  \\ 

\midrule

\multirow{2}{*}{Hidden} &  8 & 0.1303 / 0.1289 & 0.5416  \\ 
 &  16 &  0.1121 / 0.1108 & 0.5214 \\ 

\bottomrule
\end{tabular}
}
\end{table}

\subsubsection{Similarity Ablations} 

We explore alternative similarity metrics for CIS derived from different representations. Concretely, we replace the default cosine similarity between \emph{query} embeddings (Eq.~\ref{eq:cos_sim}) with cosine similarity computed on (i) hidden states $\bx$ and (ii) key embeddings $\bk$, and evaluate under the CIS$^\ast$ configuration in Sec.~\ref{sec:gsm8k&coqa}. As shown in Table~\ref{tab:cis_variant}, both variants reduce accuracy, with the hidden-state metric degrading most. This indicates that the query, key, and hidden-state spaces emphasize distinct semantics; temporal semantic proximity for CIS is best captured in the \emph{query} space, whereas measuring similarity in $\bx$ or $\bk$ misaligns clusters and harms selection quality.

\begin{figure}[!t]
    \centering
    \subfloat[Qwen3-4B-Instruct-2507.]{
        \includegraphics[width=0.44\textwidth]{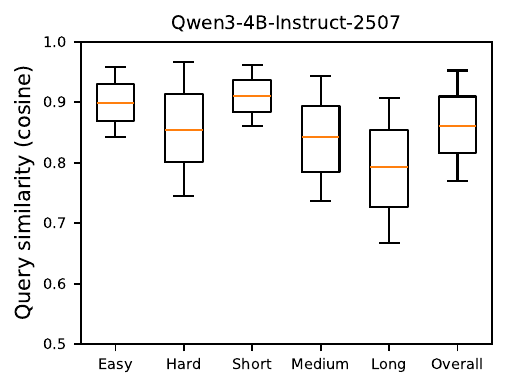}
    }\\
    \subfloat[GLM-4.5-Air.]{
        \includegraphics[width=0.44\textwidth]{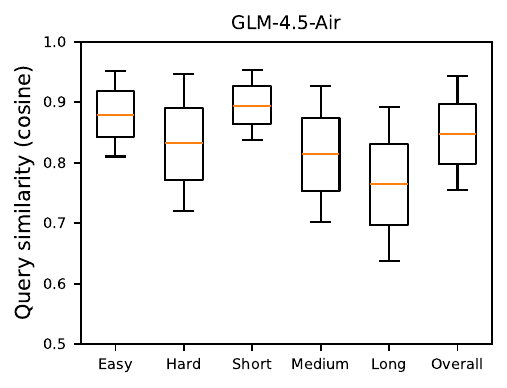}
    }
    \caption{\textbf{Query-similarity distributions across LongBench v2 workloads.}
    Cosine similarities are measured between consecutive per-head decoding queries.
    Boxes denote the 25th--75th percentiles with the median, and whiskers denote
    the 10th--90th percentiles. Similarity remains concentrated at high values,
    while harder and longer workloads exhibit lower and broader distributions.}
    \label{fig:query_similarity_workload}
    \vspace{-1em}
\end{figure}

\subsubsection{Query-Similarity Distribution Across Workloads}
\label{app:query_similarity}

To characterize when the query-similarity assumption underlying CIS holds, we measure the cosine similarity between consecutive query vectors on LongBench v2. Specifically, for each decoding step $t$, layer $\ell$, and query head $h$, we compute Eq.~\eqref{eq:cos_sim} using the same per-head query representations employed by CIS. We aggregate the scores across decoding steps, layers, heads, and samples, and report their distributions by LongBench v2 difficulty and context-length groups. The boxes in Fig.~\ref{fig:query_similarity_workload} show the interquartile range with the median, while the whiskers denote the 10th--90th percentiles.

Figure~\ref{fig:query_similarity_workload} shows that consecutive queries remain highly correlated across both models, but the distribution is workload dependent. For Qwen3-4B, the median similarity is approximately $0.90$--$0.91$ on Easy/Short workloads and decreases to $0.86$, $0.84$, and $0.79$ for Hard, Medium, and Long subsets, respectively. GLM-4.5-Air exhibits the same trend, with corresponding medians of approximately $0.88$--$0.90$ on Easy/Short inputs and $0.83$, $0.82$, and $0.77$ on Hard, Medium, and Long inputs. The harder and longer subsets also exhibit wider tails, indicating more frequent semantic changes between adjacent decoding queries.

This behavior clarifies the applicability and failure mode of CIS. Query similarity is not assumed to remain uniformly high: sharing is activated only when the similarity exceeds $\tau$. When no valid anchor satisfies the threshold, CIS falls back to normal retrieval rather than reusing a stale critical set. Thus, abrupt semantic changes primarily reduce the achievable sharing ratio and speedup instead of forcing erroneous KV reuse. In particular, the Long distributions straddle the default $\tau=0.8$ threshold, naturally causing more retrieval operations than Short-context workloads. The similarity gate therefore provides an adaptive safeguard for workloads with rapid topic shifts or heterogeneous context.

\begin{figure*}[!t]
    \centering
    \subfloat[Qwen3-4B, GSM8K, Layer 10, Head 2.]{
        \includegraphics[width=0.235\textwidth]{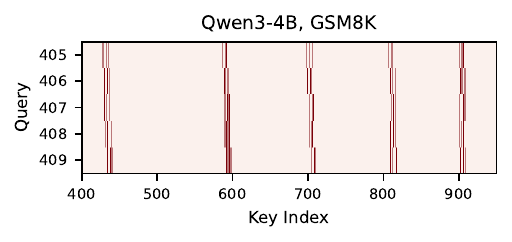}
    }
    \hfill
    \subfloat[Qwen3-4B, COQA, Layer 12, Head 5.]{
        \includegraphics[width=0.235\textwidth]{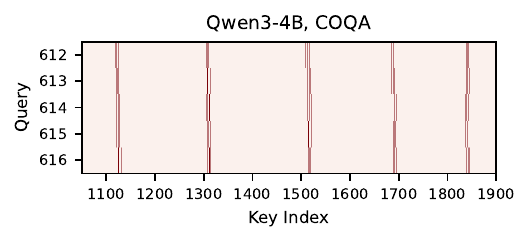}
    }
    \hfill
    \subfloat[GLM-4.5-Air, GSM8K, Layer 9, Head 3.]{
        \includegraphics[width=0.235\textwidth]{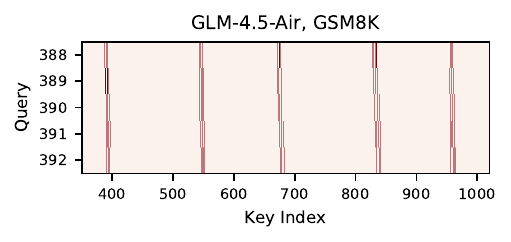}
    }
    \hfill
    \subfloat[GLM-4.5-Air, COQA, Layer 11, Head 4.]{
        \includegraphics[width=0.235\textwidth]{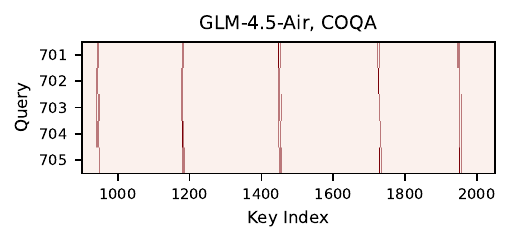}
    }
    \vspace{-0.5em}
    \caption{\textbf{Distribution of critical indices on modern models.}
    Each panel shows the top-$32$ oracle critical indices of five consecutive
    decoding queries after excluding deterministic sink and local regions.
    Critical indices form spatially localized groups whose positions exhibit
    only small shifts across adjacent queries, consistently on Qwen3-4B and
    GLM-4.5-Air.}
    \label{fig:modern_cluster}
    \vspace{-1em}
\end{figure*}

\subsubsection{Critical-Index Clustering on Modern Models}
\label{app:modern_clustering}

We further repeat the clustering analysis of Fig.~\ref{fig:attn-cluster} on Qwen3-4B and GLM-4.5-Air to examine whether the property observed on LLaMA2 generalizes to more recent architectures. We use representative GSM8K and COQA samples and inspect five consecutive decoding queries at selected middle-layer attention heads. For each query, we compute the original dense attention distribution and retain the top-$16$ critical indices from the eligible middle KV region; deterministic sink and local tokens are omitted from the visualization. No CIS sharing or neighboring dilation is applied, such that the plotted structure is induced solely by the original dense attention.

As shown in Fig.~\ref{fig:modern_cluster}, both model families exhibit the same spatial organization observed in LLaMA2. Rather than being uniformly scattered across the KV cache, high-attention indices concentrate into a small number of narrow groups. More importantly, these groups remain spatially stable across adjacent queries: individual indices may shift by a few positions, while the corresponding cluster locations change only gradually. The behavior appears consistently on both GSM8K and the longer-context COQA workload, although COQA exposes a larger eligible KV range and consequently more separated clusters.

These observations indicate that critical-index clustering and local centroid drift are not specific to LLaMA2 or WikiText2. They also explain why direct index reuse can miss important tokens despite high query similarity: the cluster itself persists, whereas its individual members change locally. CIS therefore dilates the most important shared indices by a small neighborhood, allowing the reused set to follow this local drift without requiring full per-query retrieval. Together with Fig.~\ref{fig:query_similarity_workload}, the results support the two components of the CIS design: similarity gating determines \emph{when} sharing is reliable, while neighboring dilation compensates for the residual movement of critical indices \emph{within} similar-query regions.

\begin{figure}[!t]
    \centering
    \subfloat[COQA exact-match accuracy.]{
        \includegraphics[width=0.47\textwidth]{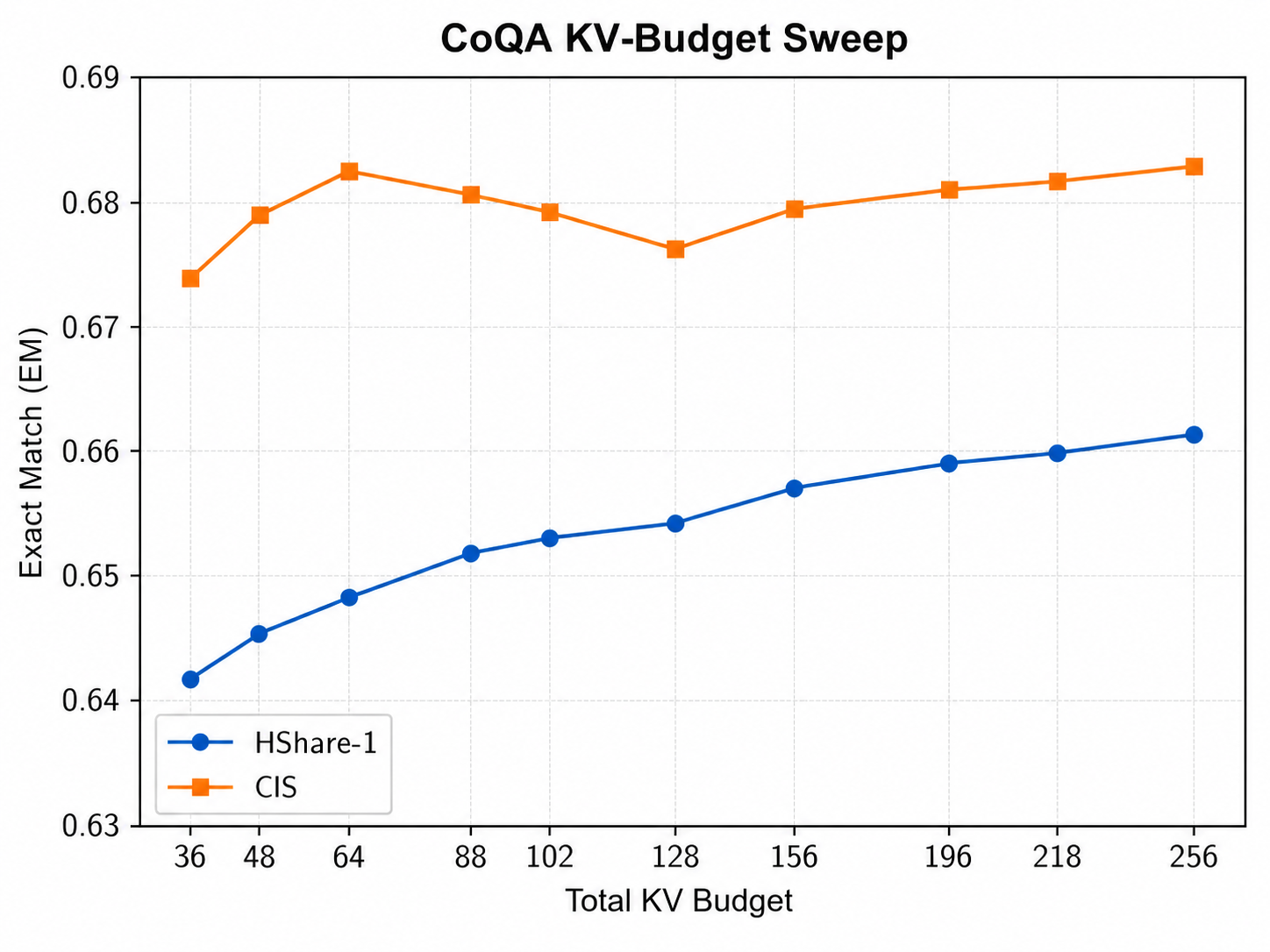}
        \label{fig:kv_budget_coqa}
    }
    \hfill
    \subfloat[GSM8K strict exact-match accuracy.]{
        \includegraphics[width=0.47\textwidth]{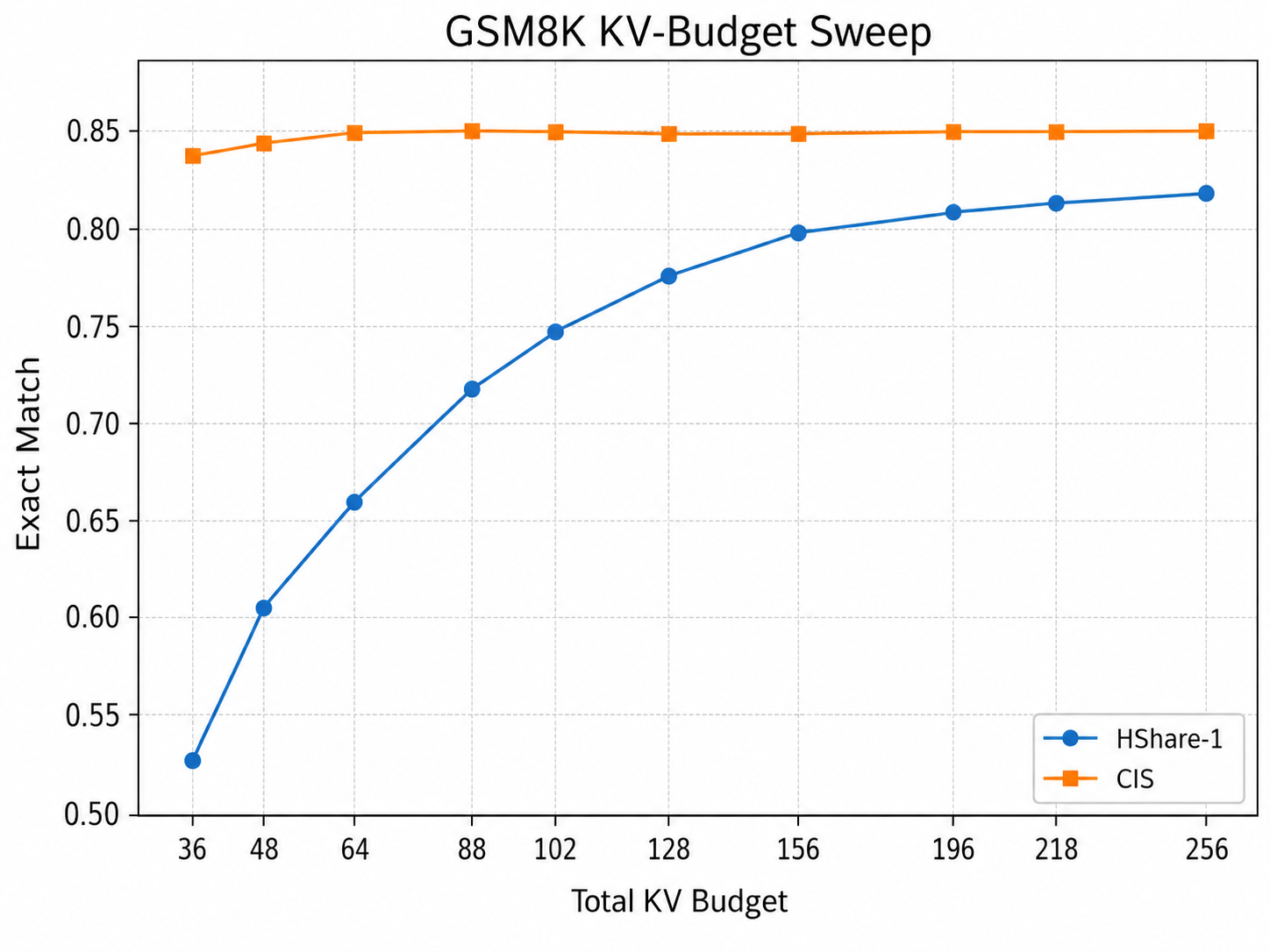}
        \label{fig:kv_budget_gsm8k}
    }
    \vspace{-0.5em}
    \caption{\textbf{Performance across total KV budgets on Qwen3-4B.}
    We sweep the retained KV budget from $36$ to $256$ for HShare-1 and CIS.
    CIS remains stable under aggressive KV compression, whereas HShare
    degrades substantially as the budget decreases, particularly on GSM8K.}
    \label{fig:kv_budget_sweep}
    \vspace{-1em}
\end{figure}

\subsubsection{Sensitivity to the Total KV Budget}
\label{app:kv_budget}

We further examine whether the advantage of CIS depends on a particular KV budget by sweeping the total retained KV capacity from $36$ to $256$ on Qwen3-4B. We compare CIS with HShare-1 under the same evaluation protocol and vary only the total KV budget. Figure~\ref{fig:kv_budget_sweep} reports exact-match (EM) accuracy on COQA and strict exact-match accuracy on GSM8K.

As shown in Fig.~\ref{fig:kv_budget_sweep}(a), CIS consistently outperforms HShare across the entire COQA budget range. At the measured budgets of $36$, $64$, and $128$, CIS obtains $0.6738$, $0.6825$, and $0.6763$ EM, compared with $0.6415$, $0.6483$, and $0.6543$ for HShare, respectively. Increasing the budget gradually improves HShare, whereas CIS remains comparatively insensitive to the available capacity, maintaining approximately $0.67$--$0.68$ EM even under aggressive compression.

The difference is more pronounced on GSM8K (Fig.~\ref{fig:kv_budget_sweep}(b)). HShare is strongly budget dependent: its strict exact-match score increases from $0.5254$ at a budget of $36$ to $0.6573$ at $64$ and $0.7756$ at $128$, and gradually approaches saturation only at larger budgets. In contrast, CIS already achieves $0.8385$ with only $36$ KV entries and remains stable around $0.84$--$0.85$ throughout the sweep. This indicates that CIS preserves high-importance KV entries even when the available capacity is severely constrained, whereas direct sharing in HShare increasingly loses salient tokens as the budget decreases.

Overall, the sweep shows that the gain of CIS is not confined to the default budget used in the main experiments. Its advantage becomes particularly large in the low-budget regime and gradually narrows as the KV capacity increases, where both methods have sufficient capacity to retain more relevant tokens. These results are consistent with the retained-mass analysis: neighboring dilation improves coverage of high-importance critical-token clusters, making CIS substantially more robust to aggressive KV compression.


\vfill

\end{document}